\RequirePackage[l2tabu,orthodox]{nag}
\documentclass
[letterpaper,12pt,]
{article}

\usepackage{etex}
\usepackage{verbatim}
\usepackage{xspace,enumerate}
\usepackage[dvipsnames]{xcolor}
\usepackage[T1]{fontenc}
\usepackage[full]{textcomp}
\usepackage[american]{babel}
\usepackage{mathtools}
\usepackage{amsthm}
\usepackage[
letterpaper,
top=1in,
bottom=1in,
left=1in,
right=1in]{geometry}
\usepackage{newpxtext} 
\usepackage{textcomp} 
\usepackage[varg,bigdelims]{newpxmath}
\usepackage[scr=rsfso]{mathalfa}
\usepackage{bm} 
\usepackage{cancel} 
\let\mathbb\varmathbb
\usepackage{microtype}
\usepackage[pagebackref,colorlinks=true,urlcolor=blue,linkcolor=blue,citecolor=OliveGreen]{hyperref}
\usepackage[capitalise,nameinlink]{cleveref}
\crefname{lemma}{Lemma}{Lemmas}
\crefname{fact}{Fact}{Facts}
\crefname{theorem}{Theorem}{Theorems}
\crefname{corollary}{Corollary}{Corollaries}
\crefname{claim}{Claim}{Claims}
\crefname{example}{Example}{Examples}
\crefname{algorithm}{Algorithm}{Algorithms}
\crefname{problem}{Problem}{Problems}
\crefname{definition}{Definition}{Definitions}
\crefname{exercise}{Exercise}{Exercises}
\crefname{model}{Model}{Models}
\usepackage{amsthm}

\newtheorem{theorem}{Theorem}[section]
\newtheorem*{theorem*}{Theorem}
\newtheorem{lemma}[theorem]{Lemma}
\newtheorem*{lemma*}{Lemma}
\newtheorem{fact}[theorem]{Fact}
\newtheorem*{fact*}{Fact}

\newtheorem*{proposition*}{Proposition}
\newtheorem{corollary}[theorem]{Corollary}
\newtheorem*{corollary*}{Corollary}

\newtheorem*{hypothesis*}{Hypothesis}

\newtheorem*{conjecture*}{Conjecture}
\theoremstyle{definition}
\newtheorem{definition}[theorem]{Definition}
\newtheorem*{definition*}{Definition}

\newtheorem*{construction*}{Construction}

\newtheorem*{example*}{Example}
\newtheorem{question}[theorem]{Question}
\newtheorem*{question*}{Question}
\newtheorem{algorithm}[theorem]{Algorithm}
\newtheorem*{algorithm*}{Algorithm}

\newtheorem*{assumption*}{Assumption}

\newtheorem*{problem*}{Problem}

\newtheorem*{openquestion*}{Open Question}
\theoremstyle{remark}

\newtheorem*{claim*}{Claim}

\newtheorem*{remark*}{Remark}

\newtheorem*{observation*}{Observation}
\theoremstyle{model}

\newtheorem*{model*}{Model}
\usepackage{paralist}
\let\originalleft\left
\let\originalright\right
\renewcommand{\left}{\mathopen{}\mathclose\bgroup\originalleft}
\renewcommand{\right}{\aftergroup\egroup\originalright}
\usepackage{turnstile}
\usepackage{mdframed}
\usepackage{tikz}
\usetikzlibrary{shapes,arrows}
\usetikzlibrary{positioning}
\usetikzlibrary{decorations.pathreplacing}
\usepackage{caption}
\DeclareCaptionType{Algorithm}
\usepackage{newfloat}
\usepackage{array}
\usepackage{subfig}
\usepackage{xparse}
\usepackage{amsthm} 
\makeatletter
\let\latexparagraph\paragraph
\RenewDocumentCommand{\paragraph}{som}{%
  \IfBooleanTF{#1}
    {\latexparagraph*{#3}}
    {\IfNoValueTF{#2}
       {\latexparagraph{\maybe@addperiod{#3}}}
       {\latexparagraph[#2]{\maybe@addperiod{#3}}}%
  }%
}
\newcommand{\maybe@addperiod}[1]{%
  #1\@addpunct{.}%
}
\makeatother

\newcommand{\Authornote}[2]{{\sffamily\small\color{red}{}}}

\newcommand{\Authorcomment}[2]{{\sffamily\small\color{gray}{}}}
\newcommand{\Authorfnote}[2]{\footnote{\color{red}{}}}

\usepackage{boxedminipage}
\newcommand{\paren}[1]{(#1)}
\newcommand{\Paren}[1]{\left(#1\right)}
\newcommand{\bigparen}[1]{\big(#1\big)}
\newcommand{\Bigparen}[1]{\Big(#1\Big)}
\newcommand{\brac}[1]{[#1]}
\newcommand{\Brac}[1]{\left[#1\right]}

\newcommand{\abs}[1]{\lvert#1\rvert}
\newcommand{\Abs}[1]{\left\lvert#1\right\rvert}

\newcommand{\card}[1]{\lvert#1\rvert}

\newcommand{\set}[1]{\{#1\}}
\newcommand{\Set}[1]{\left\{#1\right\}}

\newcommand{\Bigset}[1]{\Big\{#1\Big\}}
\newcommand{\norm}[1]{\lVert#1\rVert}
\newcommand{\Norm}[1]{\left\lVert#1\right\rVert}

\newcommand{\normo}[1]{\norm{#1}_1}

\newcommand{\iprod}[1]{\langle#1\rangle}
\newcommand{\Iprod}[1]{\left\langle#1\right\rangle}

\newcommand{\Esymb}{\mathbb{E}}
\newcommand{\Psymb}{\mathbb{P}}

\DeclareMathOperator*{\E}{\Esymb}

\DeclareMathOperator*{\ProbOp}{\Psymb}
\renewcommand{\Pr}{\ProbOp}

\newcommand{\sbits}{\{\pm1\}}

\newcommand{\seteq}{\mathrel{\mathop:}=}
\newcommand{\from}{\colon}

\newcommand{\Mid}{\nonscript\;\middle\vert\nonscript\;}

\newcommand\bdot\bullet

\DeclareMathOperator{\Ind}{\mathbf 1}

\DeclareMathOperator{\Tr}{Tr}

\DeclareMathOperator{\poly}{poly}

\DeclareMathOperator{\argmax}{argmax}
\DeclareMathOperator{\polylog}{polylog}
\DeclareMathOperator{\supp}{supp}

\DeclareMathOperator{\sign}{sign}

\newcommand{\Erdos}{Erd\H{o}s\xspace}
\newcommand{\Renyi}{R\'enyi\xspace}

\newcommand{\Z}{\mathbb Z}

\newcommand{\R}{\mathbb R}

\newcommand{\cA}{\mathcal A}

\newcommand{\cD}{\mathcal D}

\renewcommand{\leq}{\leqslant}
\renewcommand{\le}{\leqslant}
\renewcommand{\geq}{\geqslant}
\renewcommand{\ge}{\geqslant}
\let\epsilon=\varepsilon
\numberwithin{equation}{section}

\newcommand\MYcurrentlabel{xxx}

\let\MYoldlabel\label

\newcommand{\MYstore}[2]{%
  \global\expandafter \def \csname MYMEMORY #1 \endcsname{#2}%
}
\newcommand{\MYload}[1]{%
  \csname MYMEMORY #1 \endcsname%
}
\newcommand{\MYnewlabel}[1]{%
  \renewcommand\MYcurrentlabel{#1}%
  \MYoldlabel{#1}%
}

\newcommand{\MYdummylabel}[1]{}
\newcommand{\torestate}[1]{%
  \let\MYoldlabel\label%
  \let\label\MYnewlabel%
  #1%
  \MYstore{\MYcurrentlabel}{#1}%
  \let\label\MYoldlabel%
}
\newcommand{\restatetheorem}[1]{%
  \let\MYoldlabel\label
  \let\label\MYdummylabel
  \begin{theorem*}[Restatement of \cref{#1}]
    \MYload{#1}
  \end{theorem*}
  \let\label\MYoldlabel
}
\newcommand{\restatelemma}[1]{%
  \let\MYoldlabel\label
  \let\label\MYdummylabel
  \begin{lemma*}[Restatement of \cref{#1}]
    \MYload{#1}
  \end{lemma*}
  \let\label\MYoldlabel
}
\newcommand{\restateprop}[1]{%
  \let\MYoldlabel\label
  \let\label\MYdummylabel
  \begin{proposition*}[Restatement of \cref{#1}]
    \MYload{#1}
  \end{proposition*}
  \let\label\MYoldlabel
}

\newcommand{\restatecor}[1]{%
  \let\MYoldlabel\label
  \let\label\MYdummylabel
  \begin{corollary*}[Restatement of \cref{#1}]
    \MYload{#1}
  \end{corollary*}
  \let\label\MYoldlabel
}

\newcommand{\restatefact}[1]{%
  \let\MYoldlabel\label
  \let\label\MYdummylabel
  \begin{fact*}[Restatement of \cref{#1}]
    \MYload{#1}
  \end{fact*}
  \let\label\MYoldlabel
}
\newcommand{\restate}[1]{%
  \let\MYoldlabel\label
  \let\label\MYdummylabel
  \MYload{#1}
  \let\label\MYoldlabel
}

\newcommand{\e}{\epsilon}
\newcommand{\eps}{\epsilon}

\allowdisplaybreaks
\newcommand*{\Id}{\mathrm{Id}}

\newcommand*{\Normop}[1]{\Norm{#1}_{\mathrm{op}}}

\newcommand*{\normf}[1]{\norm{#1}_{\mathrm{F}}}
\newcommand*{\Normf}[1]{\Norm{#1}_{\mathrm{F}}}

\newcommand{\normm}[1]{\norm{#1}_\textnormal{max}}

\newenvironment{algorithmbox}{\begin{mdframed}[nobreak=true]
\begin{algorithm}}{\end{algorithm}\end{mdframed}}

\newcommand{\SBM}{\mathsf{SBM}}

\renewcommand{\normo}[1]{\norm{#1}_{\text{sum}}}

\newcommand*{\transpose}[1]{{#1}{}^{\mkern-1.5mu\mathsf{T}}}
\newcommand*{\dyad}[1]{#1#1{}^{\mkern-1.5mu\mathsf{T}}}

\newcommand{\tildeE}{\tilde{\E}}
\newcommand{\Gnull}{G^{\circ}}
\newcommand{\Gbarnull}{\bar{G}^{\circ}}
\newcommand{\Xnull}{X^{\circ}}
\newcommand{\Znull}{Z^{\circ}}

\newcommand{\error}{\text{error}}

\newcommand{\tE}{\tilde{\E}}

\newcommand{\one}{\mathbf{1}}
\newcommand{\zeros}{\mathbf{0}}
\newcommand{\Ber}{\text{Ber}}

\providecommand{\Mom}{\mathrm{Mom}}

\DeclareMathOperator{\D}{\mathrm{D}}
\providecommand{\cAclose}{\cA_{\mathrm{close}}}
\providecommand{\cAsymbreak}{\cA_{\mathrm{sym\text{-}break}}}
\providecommand{\cAinit}{\cA_{\mathrm{init}}}
\providecommand{\cAmix}{\cA_{\mathrm{mix}}}

\providecommand{\cAmaj}{\cA_{\mathrm{maj}}}
\providecommand{\cAset}{\cA_{\mathrm{set}}}
\providecommand{\cAboost}{\cA_{\mathrm{boost}}}
\providecommand{\cAlabel}{\cA_{\mathrm{label}}}
\providecommand{\cAbasic}{\cA_{\mathrm{basic}}}
\providecommand{\cArand}{\cA_{\mathrm{rand}}}
\providecommand{\cArobust}{\cA_{\mathrm{robust}}}
\providecommand{\cArobustboost}{\cA_{\mathrm{robust-boost}}}
\providecommand{\cOpt}{\cA_{\mathrm{opt}}}
\providecommand{\Gcorrupted}{G_{\mathrm{corrupted}}}

\title{Rate-optimal community detection near the KS threshold via node-robust algorithms{\thanks{This work is supported by funding from the European Research Council (ERC) under the European Union's Horizon 2020 research and innovation programme (grant agreement No 815464).}}}

\author{
  Jingqiu Ding \footnotemark[2]
  \and Yiding Hua \thanks{ETH Z\"urich.} 
  \and Kasper Lindberg \footnotemark[2]
  \and David Steurer\footnotemark[2]
  \and Aleksandr Storozhenko \thanks{Princeton University. 
  Part of this work was done while the author was supported by the 
  ETH Z\"urich Student Summer Research Fellowship (SSRF).}
}

\begin{document}

\pagestyle{empty}


\maketitle



 \begin{abstract}
  We study community detection in the \emph{symmetric $k$-stochastic block model},
  where $n$ nodes are evenly partitioned into $k$ clusters with intra- and inter-cluster connection probabilities $p$ and $q$, respectively. 
  Our main result is a polynomial-time algorithm that achieves the minimax-optimal misclassification rate
  \begin{equation*}
    \exp \Bigl(-\bigl(1 \pm o(1)\bigr) \tfrac{C}{k}\Bigr),
    \quad \text{where } C = (\sqrt{pn}-\sqrt{qn})^2,
  \end{equation*}
  whenever $C \ge K\,k^2\,\log k$ for some universal constant $K$, matching the Kesten--Stigum (KS) threshold up to a $\log k$ factor.
  Notably, this rate holds even when an adversary corrupts an $\eta \le \exp\bigl(- (1 \pm o(1)) \tfrac{C}{k}\bigr)$ fraction of the nodes.

  To the best of our knowledge, the minimax rate was previously only attainable either via computationally inefficient procedures~\cite{zhang2016minimax} or via polynomial-time algorithms that require strictly stronger assumptions such as $C \ge K k^3$~\cite{Gao2018MinimaxNR}.
  In the node-robust setting, the best known algorithm requires the substantially stronger condition $C \ge K k^{102}$~\cite{Liu2022minimax}.
  Our results close this gap by providing the first polynomial-time algorithm that achieves the minimax rate near the KS threshold in both settings.

  Our work has two key technical contributions:
  (1) we robustify majority voting via the Sum-of-Squares framework,
  (2) we develop a novel graph bisectioning algorithm via robust majority voting, which allows us to significantly improve the misclassification rate to $1/\poly(k)$ for the initial estimation near the KS threshold.
\end{abstract}




\pagebreak
\microtypesetup{protrusion=false}
\tableofcontents{}
\microtypesetup{protrusion=true}
\pagebreak


\pagestyle{plain}
\setcounter{page}{1}


\section{Introduction}
\label{sec:intro}

The stochastic block model (SBM), introduced by Holland et al.~\cite{holland1983stochastic}, is a fundamental statistical model for community detection in network analysis and has been extensively studied over the years~\cite{mossel2012stochastic,abbe2015exact,krzakala2013spectral,Abbe18Review}. 

In this work, we focus on the symmetric balanced $k$-stochastic block model ($k$-SBM).
Given $k, n \in \Z^+$, let \((\bm{G}, \bm{Z}) \sim \SBM_n(d, \varepsilon, k)\) denote the $k$-SBM with expected degree \(d>0\) and bias \(\varepsilon \in [0,1]\).
The label \(\bm{Z} \in \{0, 1\}^{n \times k}\) is generated by partitioning the $n$ vertices of the graph uniformly at random into $k$ communities of size $n/k$ each;
the $i$th row $\bm{Z}(i)$ is the indicator vector of the community containing vertex $i$. 
Conditioned on the latent labels $\bm{Z}$, the edges between distinct vertices \(i, j \in [n]\) are sampled independently with probability
\begin{equation}
\Pr\Set{(i,j) \in E(\bm{G})}
\coloneqq
\begin{cases}
p_1
= \Paren{1 + \bigl(1 - \tfrac{1}{k}\bigr)\varepsilon} \cdot \tfrac{d}{n},
& \text{if } \bm{Z}(i) = \bm{Z}(j), \\[0.6ex]
p_2
= \Paren{1 - \tfrac{\varepsilon}{k}} \cdot \tfrac{d}{n},
& \text{otherwise}.
\end{cases}
\end{equation}

Given an observation of the graph \(\bm{G} \sim \SBM_n(d, \varepsilon, k)\), 
a central statistical problem is to recover its latent community structure $\bm{Z}$. 
This perspective naturally leads to the question of how accurately the hidden labels $\bm{Z}$ can be recovered as a function of the parameters $\varepsilon$ and $d$; 
see the survey~\cite{Abbe18Review} for a detailed overview.

The strongest statistical goal is to recover all labels without error, a task known as \emph{exact recovery}.
This is not always achievable: for instance, when $d = O(1)$, the graph $\bm{G}$ typically contains isolated vertices whose labels cannot be inferred. 
More generally, Hajek et al.~\cite{Hajek2014AchievingEC} and Abbe et al.~\cite{abbe2015exact,abbe2015community} showed that there exist efficient algorithms achieving exact recovery whenever \( C_{d,\e} \ge k \log n,\)
the so-called \emph{Chernoff--Hellinger threshold}, 
and that exact recovery is information-theoretically impossible below this threshold.

When exact recovery is out of reach, a natural weaker requirement is to output a labeling that performs better than random guessing, or equivalently, has nontrivial correlation with the ground truth. 
This guarantee is called \emph{weak recovery} and has been a central notion in the study of the SBM. 
The key threshold governing weak recovery is the \emph{Kesten--Stigum (KS) threshold}, given by $\varepsilon^2 d > k^2$.
In the symmetric, balanced $k$-SBM, a sequence of works~\cite{Decelle_2011,massoulie2014community,bordenave2015non,AbbeSandonLABP2016,Mossel2018proof} established that weak recovery can be achieved in polynomial time whenever $\varepsilon^2 d > k^2$. 
Conversely, for $\varepsilon^2 d < k^2$, 
a rich line of work~\cite{Hopkins17,Hopkins18,bandeira2021spectral,ding2025low,sohn2025sharp} provides strong evidence that no polynomial-time algorithm can achieve weak recovery in this regime.

These results delineate the regimes where exact recovery is possible and those where only weak recovery can be hoped for.
Between these two extremes, it is natural to ask: for fixed $\varepsilon$ and $d$, what is the best achievable recovery rate, either information-theoretically or by efficient algorithms? 
This question leads to the notion of the \emph{minimax rate} in the SBM.

\paragraph{Minimax error rate}
For a family $\Theta$ of $k$-SBMs (specified by the degree $d$, bias $\varepsilon$, and allowing imbalance), we define the misclassification error as
\begin{equation*}
\error_k(\hat{Z}, \bm{Z}) \,\coloneqq\, 
\min_{\pi:[k]\to[k]} \frac{1}{2n}\sum_{j=1}^k 
\Bigl\|\hat{Z}(\cdot,j) - \bm{Z}(\cdot,\pi(j))\Bigr\|_1 \,.
\end{equation*}
The minimax rate is then
\begin{equation*}
\inf_{\hat{Z}}\sup_{\theta \, \in \, \Theta} 
\mathbb{E}_{\theta}\bigl[\error_k(\hat{Z}, \bm{Z})\bigr] \,,
\end{equation*}
where the infimum is taken over all estimators $\hat{Z}$, the supremum over parameters $\theta \in \Theta$, and $\mathbb{E}_{\theta}$ denotes expectation with respect to the $k$-SBM with parameter $\theta$.

Zhang et al.~\cite{zhang2016minimax} showed that whenever \(\varepsilon^2 d \geq \Omega(k\log k)\), the minimax error rate is given by \(\exp(-C_{d,\varepsilon}/k)\), where \(C_{d,\varepsilon}\)\footnote{For interpretation, note that $C_{d,\varepsilon}$ differs from $\varepsilon^2 d$ only by a constant factor.}
is defined as
\begin{equation}
    C_{d,\varepsilon} \coloneqq \Paren{ \sqrt{np_1} - \sqrt{np_2} }^2
    = \Paren{
        \sqrt{ \Paren{1 + \Paren{1 - \dfrac{1}{k}} \varepsilon} d }
        -
        \sqrt{ \Paren{1 - \dfrac{\varepsilon}{k}} d }
    }^2 \enspace.
\end{equation}
In the same work, they constructed a statistically optimal but computationally inefficient estimator that attains this rate for the $k$-SBM under the condition \(\varepsilon^2 d \ge \Omega(k\log k)\).

Efficient procedures are known in two regimes. 
For the case $k=2$, Mossel et al.~\cite{pmlr-v35-mossel14} developed a polynomial-time algorithm based on belief propagation that achieves the optimal rate when \(C_{d,\varepsilon}\) is sufficiently large.
For general $k$, Gao et al.~\cite{Gao2018MinimaxNR} proposed a polynomial-time algorithm based on spectral initialization followed by majority voting, 
showing that it attains the minimax rate under the stronger condition \(\varepsilon^2 d \ge \Omega(k^3)\).
Thus, there remains a gap between the KS threshold \(\varepsilon^2 d > k^2\) and the best-known polynomial-time algorithm \(\varepsilon^2 d \ge \Omega(k^3)\), 
which motivates the following question on achieving the minimax rate in polynomial time for the $k$-SBM.

\begin{question}
    Can the minimax error rate \(\exp(-C_{d,\varepsilon}/k)\) be achieved by efficient algorithms under the condition \(\varepsilon^2 d \ge Kk^2\) for some universal constant \(K\)?
\end{question}

\paragraph{Node robustness} 
The SBM is an idealized statistical model and may fail to capture certain features of real-world networks. 
In social networks, for example, fake or automated accounts can distort connectivity patterns, 
while in biological networks, contaminated samples may induce spurious interactions.

A natural way to model such effects is via \emph{node corruptions}, formally introduced by Acharya et al.~\cite{acharya2022robust} in the context of edge density estimation in \Erdos-\Renyi random graphs. 
In the node corruption model, an adversary may arbitrarily modify all edges incident to up to $\eta n$ vertices, where \(\eta \in [0,1)\) is the corruption fraction.
This allows up to \(O(\eta n^2)\) adversarial edge changes and poses significant challenges for algorithm design. 
In particular, both belief propagation~\cite{pmlr-v35-mossel14} and spectral algorithms~\cite{Gao2018MinimaxNR} rely on distributional assumptions that are brittle under such perturbations.

The study of the node corruption model for community recovery in the stochastic block model dates back to~\cite{graphPowering,abbe2020alon}, where the authors give a fast algorithm that is robust to a $\polylog(n)$ number of corrupted nodes.
For the symmetric $k$-SBM, Liu et al.~\cite{Liu2022minimax} recently gave the first polynomial-time algorithm that achieves the minimax rate under node corruption when \(\eta \le \exp(-C_{d,\varepsilon}/k)\).\footnote{Their algorithm also tolerates semi-random corruption that can arbitrarily add intra-cluster edges and delete inter-cluster edges.}
However, their result requires \(\varepsilon^2 d \ge \Omega(k^{102})\), which is very far from the KS threshold.
This leads to the following natural question in the node-robustness setting.

\begin{question}
    Given an $\eta$-node-corrupted $k$-SBM and assuming \(\varepsilon^2 d \ge K k^2\), can we achieve the minimax rate \(\exp(-C_{d,\varepsilon}/k)\) in polynomial time when \(\eta \le \exp(-C_{d,\varepsilon}/k)\)?
\end{question}

\subsection{Results}
In this work, we answer the above questions affirmatively, up to a multiplicative $\log k$ factor.
More specifically, we give the first polynomial-time algorithm that achieves the minimax rate in the presence of node corruption, 
even when the number of communities $k$ diverges with $n$.

\begin{theorem}[Informal main theorem, see \cref{thm:formal_main_result} for the formal version]\torestate{\label{thm:main_result}
Let $k\leq n^{0.001}$, $\eta \leq \frac{1}{\poly(k)}$, and $d = o(n)$.
Suppose $\e^2 d\geq Kk^2\log k$ for some sufficiently large universal constant $K$.
Then there exists a polynomial-time algorithm that, given an $\eta$-node corrupted graph \(\bm{G} \sim \SBM_n(d,\e,k)\), produces a labeling with expected error at most
\[
 \exp\left( -\frac{C_{d,\varepsilon}}{k} + O\left( \sqrt{d}\varepsilon \right) \right) + \poly(k)\eta \,.
\]
}
\end{theorem}

Under the conditions \(\eta \le \exp \bigl(- \tfrac{C_{d,\varepsilon}}{k}(1+o(1))\bigr)\) and \(\varepsilon^2 d \ge K k^2 \log k\), this bound shows that our algorithm robustly attains the minimax error rate. 
The requirement \(\eta \le \exp(-C_{d,\varepsilon}/k)\) is information-theoretically necessary, and the condition \(\varepsilon^2 d \ge K k^2 \log k\) matches the KS threshold up to a \(\log k\) factor. 
Moreover, as an immediate corollary of \cref{thm:main_result}, in the absence of node corruptions (i.e., when $\eta=0$), our algorithm achieves exact recovery under the information-theoretically optimal condition $C_{d,\e} \geq k\log n$.

For comparison, the best previously known node-robust algorithm~\cite{Liu2022minimax} achieves an error rate of
\[
 \exp \left(-\frac{C_{d,\varepsilon}}{k} + O\bigl(k^{50}\sqrt{d}\,\varepsilon\bigr)\right) + O(k \eta) \,,
\]
which exhibits a much worse dependence on $k$.
In addition, even to obtain weak recovery, their algorithm requires \(\varepsilon^2 d \ge k^{102}\), which is far above the KS threshold.

\paragraph{Open problem} We leave as an open problem whether there exist polynomial-time algorithms that achieve the minimax error rate $\exp(-C_{d,\e}/k)$ under the weaker condition $\e^2 d \ge \Omega(k^2)$. 
This problem remains open even in the non-robust setting.
Note that, since \(C_{d,\varepsilon}/k \gtrsim k\), when $\e^2 d \ge \Omega(k^2)$ the optimal error \(\exp(-C_{d,\varepsilon}/k)\) is at most \(\exp(-k)\).
Thus, a positive resolution would yield an error rate of $\exp(-(1\pm o(1))k)$ throughout the regime above the KS threshold.

Moreover, we suspect that our results extend to inhomogeneous stochastic block models in which the intra-cluster connection probabilities are at least $a/n$ and the inter-cluster connection probabilities are at most $b/n$.

\subsection{Previous work}

\paragraph{Semidefinite programming for community detection}
Semidefinite programming (SDP) has been extensively studied for community detection.
The most natural SDP is based on an SDP relaxation of the \emph{minimum bisection} problem (which we refer to as the basic SDP), first studied by Feige and Kilian~\cite{FeigeKilian01}.
Hajek et al. initiated the study of the power of the basic SDP for exact recovery in the $2$-SBM~\cite{Hajek2014AchievingEC}.
Guédon and Vershynin developed the first polynomial-time algorithm that achieves an error rate $o(1)$ for the balanced $k$-SBM when \(\e^2 d\gg k^2\)~\cite{Gudon2014CommunityDI}.
More recently, Fei and Chen~\cite{pmlr-v99-fei19a} applied a novel primal-dual analysis to the basic SDP for the balanced $2$-SBM, proving that it achieves the minimax error rate near the KS threshold.

\paragraph{Majority voting based algorithms for minimax error rate}
Majority voting has been the de facto method to achieve the minimax rate.
Gao et al. proposed a two-stage algorithm for $k$-SBM that combines spectral methods for initial estimation with majority voting for accuracy boosting~\cite{Gao2018MinimaxNR}.
However, their algorithm requires \( \varepsilon^2 d \geq \Omega(k^3) \), which exceeds the KS threshold by a factor of $k$.

Liu and Moitra later developed the first polynomial-time algorithm that achieves the minimax error rate under node corruption by robustly simulating majority voting via SDP~\cite{Liu2022minimax}. 
However, their algorithm requires \(\e^2 d\geq \Omega(k^{102})\), which is substantially above the KS threshold.

\paragraph{Sum-of-Squares algorithms for robust estimation}
The Sum-of-Squares (SoS) framework provides a general \emph{proof-to-algorithm} paradigm: 
low-degree SoS certificates of identifiability can be converted by semidefinite programming into efficient algorithms whose guarantees follow mechanically from the SoS proofs~\cite{barak2014sum,raghavendra2018high}. 
This framework is particularly powerful for robust estimation problems in random graphs, 
where pseudorandomness can be certified by constant-degree spectral constraints. 
In such settings, one can encode exponentially many correlation inequalities using only a polynomial number of constraints~\cite{pmlr-v195-hua23a, chen2024graphon, chen2024private, chen2025improved}. 
In the context of SBMs, \cite{pmlr-v195-hua23a} used SoS to achieve weak recovery at the KS threshold, even when a constant fraction of nodes is adversarially corrupted.

\subsection{Organization}
The rest of the paper is organized as follows.
In \cref{sec:techniques}, we give a high-level overview of our techniques with a focus on robust majority voting via Sum-of-Squares.
In \cref{sec:preliminary}, we cover notations and basics of the Sum-of-Squares framework.
In \cref{sec:stat-prop}, we establish the statistical properties of the \( k \)-SBM model with a focus on the majority voting concentration result. 
In \cref{sec:Main-theorem-algorithm}, we outline our algorithmic framework and present our formal main results.
The proofs of these results are subsequently detailed in \cref{sec:Initialization} (initial rough clustering),
\cref{sec:IdentifyingRecoveredBlocks} (identifying clusters with a \( 0.99 \)-fraction of the vertices recovered),
and \cref{sec:robust-bisection-boosting} (robustly boosting the rough bisection accuracy to bisection minimax rate), respectively.
These ingredients are then combined in \cref{sec:robust-bisection-algo} to obtain our robust bisectioning algorithm. 
In \cref{sec:sym-break-robust-clustering}, we show how to recursively apply the robust bisectioning algorithm to robustly find a $k$-clustering of the graph with error rate \( 1/\poly(k) \) for symmetry breaking. 
Finally, we conclude in \cref{sec:Optimal-robust-clustering} by boosting the error rate from \( 1/\poly(k) \) to the minimax optimal rate using robust pairwise majority voting.

\section{Techniques}

\label{sec:techniques}

\subsection{Two communities}

For notational simplicity, we first illustrate some of the ideas underlying our algorithm for the case of two communities.
Let \((\bm G, \bm x)\sim \SBM_{n}(d,\e)\) denote the stochastic block model with degree parameter \(d>0\) and bias parameter \(\e\in [0,1]\).
Here, we choose \(\bm x \in \{\pm 1\}^{n}\) uniformly at random among all unbiased vectors so that \(\sum_{i=1}^{n} \bm x(i)=0\), where we take \(n\ge 2\) to be even.
We can view \(\bm x\) as a balanced bipartition of the vertex set \([n]\).
We refer to its parts \(\bm x^{-1}(1)\) and \(\bm x^{-1}(-1)\) as blocks or communities.
Then, for every pair of distinct vertices \(i,j\in[n]\), we decide independently at random whether to join them by an edge with probability,
\begin{equation}
  \label{eq:2-sbm-edge-probability}
  \Pr\Set{ij \in E(\bm G) \mid \bm x}
  = \Paren{1+\e\cdot \bm x(i) \cdot \bm x(j)}\cdot \frac dn\,.
\end{equation}

\paragraph{Basic recovery approach}

In the stochastic block model without adversarial corruptions, the most basic approach for approximately recovering the underlying blocks uses spectral techniques \cite{boppana1987eigenvalues,guedon2016community,coja2010graph,massoulie2014community,chin2015stochastic}.
As a slight abuse of notation, let \(G\) denote the centered adjacency matrix of an \(n\)-vertex graph with degree parameter \(d\),
\begin{equation}
  \label{eq:centered-adjacency-matrix}
  G(i,j) = \begin{cases}
    1-\tfrac d n & \text{ if \(ij\in E(G)\),}\\
    -\tfrac dn & \text{ otherwise.}
  \end{cases}
\end{equation}
Then, for the stochastic block model \((\bm G,\bm x)\sim \SBM_n(d,\e)\),
\begin{equation}
  \label{eq:centered-adjacency-matrix-expectation}
  \E \Brac{\bm G \Mid \bm x} = \tfrac {\e d} n \cdot  \dyad{\bm x}\,.
\end{equation}
At the same time, conditioned on $\bm x$, the matrix \(\bm E\seteq\tfrac{n}{\e d}  \bm G- \dyad{\bm x}\) has (up to symmetry) independent mean-zero entries and enjoys a spectral norm bound \(\norm{\bm E}\lesssim n/\paren{\e  \sqrt d}\) for \(d\gg \log n\).
In this way, the best rank-one approximation of \(G\) allows us to approximately recover \(\bm x\) up to sign as long as \(\e\cdot \sqrt d\gg 1\) and \(d\gg \log n\).
After carefully truncating high-degree nodes, we can even drop the logarithmic-degree condition \(d\gg \log n\).
It turns out that for approximately recovering the underlying blocks, the only remaining condition, \(d\gg 1/\e^{2}\), is also information-theoretically necessary.

\paragraph{Sum-of-squares lens on the basic approach.}

It will be instructive to formulate this basic approach in terms of the sum-of-squares method.
While this sum-of-squares formulation is certainly overkill for approximate recovery in the vanilla version of the model, it will give us node-robustness essentially for free (whereas basic spectral techniques or even semidefinite programming techniques cannot handle node corruption).
The basic spectral algorithms and majority voting algorithms are known to be brittle under node corruptions, while the basic semidefinite programming algorithms are only known to be robust against $O(n d\epsilon)$ number of corrupted edges. 

As a starting point for our sum-of-squares formulation, we note that a spectral norm bound \(\norm{E}\le \lambda\cdot n\) certifies the following universally quantified inequality over the hypercube,
\begin{equation}
  \label{eq:quadratic-hypercube-bound}
  \forall u,v\in \sbits^{n}.\quad \iprod{u, E v}\le \lambda \cdot n^{2}\,.
\end{equation}
Moreover, this spectral norm bound constitutes a degree-\(2\) sum-of-squares proof of the above inequality.

For our purposes, it will be useful to introduce sum-of-squares proofs as mathematical objects in their own right.
For \(n\)-variate polynomials \(p,q_{1},\ldots,q_{m}\in \R[x]\), we say that a real-valued square matrix \(R\) is a \emph{degree-\(\ell\) sum-of-squares proof} of the inequality \(p\ge 0\) subject to the constraints  \(q_{1}\ge 0,\ldots,q_{m}\ge 0\) if
\begin{equation}
  \label{eq:funny-sos-definition}
  p=\Tr \dyad R \cdot \Mom_{\ell}(q_{1},\ldots,q_{m})\,,
\end{equation}
where \(\Mom_{\ell}(q_{1},\ldots,q_{m})\) is a block-diagonal matrix with entries in \(\R[x]\) and (possibly empty) blocks \(M_{S}\) indexed by multisubsets \(S\subseteq [m]\) such that
for all pairs of monomials \(x^{\alpha},x^{\beta}\) with \(\deg x^\alpha,\deg x^{\beta}\le \tfrac 1 2 \cdot (\ell-\sum_{i\in S} \deg q_{i})\),
\begin{equation}
  \label{eq:sos-quadratic-equations}
  M_{S}(\alpha,\beta)\seteq
  \Paren{ \prod\nolimits_{i\in S} q_{i}}\cdot x^{\alpha}\cdot x^{\beta}\,.
\end{equation}

We remark that sum-of-squares proofs are more commonly defined in terms of sum-of-squares polynomial multipliers for (products of) the constraint polynomials \(q_{1},\ldots,q_{m}\).
However, it will be beneficial for our purposes to represent them as matrices (corresponding to the monomial-coefficients of the polynomial multipliers).
Crucially, our definition makes it apparent that sum-of-squares proofs are solutions to a system of quadratic equations of \cref{eq:sos-quadratic-equations}.

Note that each block \(M_{S}\) has size \(O(n^{\ell/2})\) by \(O(n^{\ell/2})\) and the number of non-empty blocks is \(O(m^{\ell})\) (as we may assume the polynomials \(q_{1},\ldots,q_{m}\) to have degree at least \(1\)).
Therefore, \(R\) has size \(O(nm)^{\ell}\) by \(O(nm)^{\ell}\) and the number of quadratic equations in \cref{eq:funny-sos-definition} is polynomial in \(n\) and \(m\) for every constant \(\ell\).

We denote that \(R\) is a \emph{degree-\(\ell\) sum-of-squares proof} of the inequality \(p\ge 0\) subject to the constraints  \(q_{1}\ge 0,\ldots,q_{m}\ge 0\) by
\begin{equation}
  \label{eq:turnstile-notation}
  R\colon \Set{q_{1}(x)\ge 0,\ldots,q_{m}(x)\ge 0}  \sststile{\ell}{x} p(x) \ge 0\,.
\end{equation}
If the variables \(x\) quantified by the proof or the degree bound \(\ell\) are clear from the context, we may drop them in the above notation.

A fundamental property of sum-of-squares proofs is that they are automatizable. If a matrix \(R\) as above exists, we can find it in time polynomial in its bit complexity \cite{barak2014sum,barak2016proofs,raghavendra2018high}.
A simple but very important consequence for estimation problems is that given as input a system of polynomial constraints \(\cA(x)\), we can compute in polynomial-time a vector \(\hat x\) that satisfies all convex polynomial inequalities that can be derived from \(\cA(x)\) by a sum-of-squares proof with polynomial bit complexity \cite{barak2014sum,barak2016proofs,raghavendra2018high}.
In this way, referred to as the \emph{proof-to-algorithm paradigm}, we obtain efficient estimation algorithms from the mere existence of small identifiability proofs.

In order to apply this terminology for approximate recovery of the stochastic block model, we consider the following polynomial constraints for the set of possible label assignments (for the case of two communities),
\begin{equation}
  \label{eq:label-constraints}
  \cA_{\mathrm{label}}(x) \seteq
  \Set{x_{1}^{2}=\ldots=x_{n}^{2}=1, \sum\nolimits_{i=1}^{n} x_{i}=0}
  \,.
\end{equation}
Next, we consider a system of quadratic equations for matrices \(E\) that are certifiably pseudo-random, in the sense that they are uncorrelated with all possible label assignments,
\begin{equation}
  \label{eq:pseudorandom}
  \cA_{\mathrm{rand}}(E, R; \lambda) \seteq \Set { R\colon \cA_{\mathrm{label}}(u),\cA_{\mathrm{label}}(v) \sststile{2}{x}\iprod{u,E v}\le \lambda\cdot n^{2}}
  \,.
\end{equation}
Here, the auxiliary variables \(R\) represent the sum-of-squares certificate for the correlation bound.
When we combine systems of polynomial constraints, we usually take these auxiliary variables to be unrelated.
In this case, we do not keep track of these variables explicitly and suppress them in the above notation.
We emphasize that the system \cref{eq:pseudorandom} has only \(E\) and \(R\) as variables but not \(u\) or \(v\).
Instead \(u,v\) determine the quadratic equations we impose for the variables \(E\) and \(R\).
We also remark that \cref{eq:pseudorandom} is a proxy for requiring that \(E\) satisfies the universally quantified inequality \cref{eq:quadratic-hypercube-bound}.
In this sense, \cref{eq:pseudorandom} allows us to encode \(2^{n}\) constraints as a polynomial number of constraints.

With the above two polynomial systems at hand, we can mimic the basic recovery approach by the following combined system,\footnote{
  In the constraint system \(\cAbasic\) we choose to treat the graph \(G\) as a variable.
  As a consequence, our sum-of-squares proof will be for statements universally quantified over \(G\).
  While this property would not be necessary for recovery in the non-robust setting, it will turn out to be crucial in the robust setting.
}
\begin{equation}
  \label{eq:basic-recovery-sos}
  \cA_{\mathrm{basic}}(G,x; \lambda)
  \seteq \cA_{\mathrm{label}}(x) \cup   \cA_{\mathrm{rand}}(\tfrac {n}{\e d}G- \dyad x; \lambda)
  \,,
\end{equation}
where we choose \(\lambda \lesssim 1/(\e \sqrt d)\).
These constraints are useful for recovery because for the same graph \(G\), essentially only one label assignment \(x\) can satisfy the constraints.
Concretely, we have
\begin{equation}
  \label{eq:basic-identifiability}
  \cA_{\mathrm{basic}}(G,x; \lambda),
  \cA_{\mathrm{basic}}(G,x'; \lambda)
  \sststile{16}{G,x,x'}
  \tfrac 1 {n^{2}}\iprod{x,x'}^{2} \ge 1-2\lambda
  \,.
\end{equation}
We emphasize that the above fact is deterministic and thus independent of any statistical model.
Ignoring some details, its proof boils down to the following direct calculation,
\begin{align}
  \label{eq:proof-sketch-identifiability}
  n^{2}-\iprod{x,x'}^{2} = \Iprod{\dyad x -\dyad{(x')}, \dyad x}
  & = \Iprod{E'-E, \dyad x}\\
  & = \iprod{ x,E'x} + \iprod{-x,E x}
  \,,
\end{align}
where \(E=\tfrac {n}{\e d}G- \dyad x\) and \(E'=\tfrac{n}{\e d}G- \dyad{(x')}\).
At the same time, the pseudorandomness constraints \cref{eq:pseudorandom} allow us to bound both terms on the right in the same way by \(\lambda\cdot n^{2}\),
\begin{align}
  \cA_{\mathrm{rand}}(E',R; \lambda), \cA_{\mathrm{label}}(x)
  \sststile{L}{E',R,x}
  \iprod{x,E' x}
  & = \lambda n^{2} - \Tr R\transpose R \cdot \Mom_{2}\Paren{\cA_{\mathrm{label}}(x)}\\
  & \le \lambda n^{2}
    \,.
\end{align}
In the second step, we use that \(\cA_{\mathrm{rand}}(E',R)\) contains the constraint \(\lambda \cdot n^{2}-\iprod{u,E'v}^{2} = \Tr R \transpose{R} \cdot \Mom_{2}(\cA_{\mathrm{label}}(u),\cA_{\mathrm{label}}(v))\), where \(u,v\) are formal variables, and that we can substitute \(x\) for both \(u\) and \(v\).
In the third step, we use that \(\Tr R \transpose{R} \cdot \Mom_{2}(\cA_{\mathrm{label}}(x))\) is a sum-of-squares polynomial in variables \(x\) and \(R\) modulo the constraints we have on \(x\)\footnote{ See \cref{thm:sos-as-constraints} for details}.
Combining the previous two calculations, we derive a bound of \(2\lambda\) on the distance \(1-\iprod{x,x'}^{2}/n^{2}\) from the constraints \(\cA_{\mathrm{basic}}(G,x),\cA_{\mathrm{basic}}(G,x')\).

The stochastic block model \((\bm G,\bm x) \sim \SBM_{n}(d,\e)\) satisfies the constraints \(\cA_{\mathrm{basic}}(G,x;\lambda)\) for \(\lambda\lesssim (\e\sqrt d)^{-1}\) with high probability if \(\e \sqrt d\gg 1\) (even for constant degree \(d\) \cite{guedon2016community}).
Thus, the now standard proof-to-algorithm paradigm for sum-of-squares \cite{ding2022robust,pmlr-v195-hua23a,mohanty2024robust} yields a polynomial-time algorithm to recover the underlying blocks up to an error of \(O\bigparen{\e \sqrt d}^{-1}\).

\paragraph{Handling node corruptions using sum-of-squares}

We will now see that the sum-of-squares formulation \cref{eq:basic-identifiability} of the basic recovery approach directly implies that we can handle a constant fraction of node corruptions.

To this end, we formulate the node distance between graphs in terms of polynomial equations,
\begin{equation}
  \label{eq:node-distance}
  \cA_{\mathrm{close}}(G_{1},G_{2},z;\eta) \seteq
  \Set{
    \D(z)^{2}=\D(z),
    D(z)\cdot (G_{1}-G_{2})\cdot \D(z)=0,
    \Tr \D(z)=(1-\eta)\cdot n
  }
  \,,
\end{equation}
where \(D(z)\) denotes the diagonal matrix with \(z\) on the diagonal.
A pair of graphs \(G_{1}\) and \(G_{2}\) satisfy the above constraints for some \(z\) if and only if the two graphs differ in at most \(\eta \cdot n\) nodes.
Here, the binary vector \(z\) indicates the set of nodes, where the two graphs agree.
As before, we will omit the auxiliary variables \(z\) in the above notation whenever there are no relations to other variables.

To handle node corruptions, we use that \(\cA_{\mathrm{basic}}\) is Lipschitz with respect to this node distance and restrictions,
\begin{equation}
  \label{eq:lipschitz-basic}
  \cAclose(G,G',z;\eta),\cA_{\mathrm{basic}}(G,x;\lambda)
  \sststile{16}{G,G',x,z}
  \cA_{\mathrm{basic}}\bigparen{\D(z) \cdot G' \cdot \D(z),x;\lambda+2\eta}
  \,.
\end{equation}
The above statement is not surprising:
If we have a good label assignment \(x\) for \(G\), then it continues to be one if we were to remove a small fraction of nodes (the complement of the set indicated by \(z\)).
After removing this set of nodes, the two graphs \(G\) and \(G'\) agree.
(Note that, in general, we cannot hope \(G'\) to satisfy the pseudorandomness constraints \(\cArand\) with respect to \(x\) and that the restriction to the vertex set indicated by \(z\) is necessary.)
Letting \(E=\tfrac {n}{\e d}G -\dyad x\) and \(E'=\tfrac {n}{\e d}\D(z)\cdot G'\cdot \D(z) - \dyad x\) and omitting some technical details, the above Lipschitz property follows from the fact,
\begin{equation}
  \label{eq:lipschitz-basic-proof}
  \cAclose(G,G',z;\eta),\cA_{\mathrm{label}}(u), \cA_{\mathrm{label}}(v)
  \sststile{16}{}
  \iprod{u,E'v}\le \iprod{u, \D(z) \cdot E \cdot \D(z) v} + 2\eta n^{2}
  \,,
\end{equation}
where \(2\eta n^{2}\) accounts for the entries where \(E'\) and \(\D(z)\cdot E \cdot \D(z)\) differ.
Each of these entries is either \(1\) or \(-1\) and is located in the union of \(\eta n\) rows and \(\eta n\) columns.
Hence, the sum of the absolute values of these entries is at most \(2\eta n^{2}\).

Given a corrupted graph \(\Gcorrupted\) as input, we can recover the underlying blocks using the following polynomial constraints in variables \(G\) and \(x\),
\begin{equation}
  \label{eq:corrupted-constraints}
  \cArobust(G,x; \lambda,\eta)
  \seteq \cAclose(G,\Gcorrupted; \eta)
  \cup \cAbasic(G,x;\lambda+4\eta)
  \,.
\end{equation}
The previous Lipschitz property of \(\cAbasic\) together with the vanilla analysis of \(\cAbasic\) for recovery implies that any two solutions to the above system for the same input graph \(\Gcorrupted\) are close (and this fact has low-degree sum-of-squares proofs),
\begin{equation}
  \label{eq:robust-identifiability}
  \cArobust(G,x; \lambda,\eta) ,   \cArobust(G',x'; \lambda,\eta)
  \sststile{16}{}
  1-\tfrac 1{n^{2}}\iprod{x,x'}^{2}\le 2\lambda + 8\eta\,.
\end{equation}
Indeed, since we impose both \(G\) and \(G'\) to be \(\eta\)-close to \(\Gcorrupted\), the graphs \(G\) and \(G'\) have to be \(2\eta\)-close and we can derive \(\cAclose(G,G',z;2\eta)\) for an appropriately chosen \(z\) by a low-degree sum-of-squares proof.
By the Lipschitz property of \(\cAbasic\), we get that \(x\) and \(x'\) satisfy \(\cAbasic\) for the same graph \(H\seteq \D(z)\cdot G\cdot\D(z)\).
Concretely, we can derive \(\cAbasic(H,x; \lambda +2\cdot 2\eta)\) and \(\cAbasic(H,x'; \lambda +2\cdot 2\eta)\) by \cref{eq:lipschitz-basic} and conclude that \(x\) and \(x'\) are close by \cref{eq:basic-identifiability}.
As before, \cref{eq:robust-identifiability} means that the sum-of-squares meta algorithm recovers the underlying blocks with error \(O(\e \sqrt d)^{-1} + O(\eta)\) for the stochastic block model with degree parameter \(d\) and bias \(\e\) in the presence of an \(\eta\) fraction of adversarial node corruptions, recovering one of the results of \cite{Liu2022minimax}.

The algorithm obtained in this way is essentially the same as in \cite{Liu2022minimax} with only superficial differences.
However, we believe the analyses differ substantially.
While the analysis in \cite{Liu2022minimax} requires significant ingenuity tailored to the particular algorithm,
our analysis in contrast is purely mechanical.

\paragraph{Boosting accuracy robustly}

While the techniques discussed so far capture up to constant factors the information-theoretically necessary parameters to achieve small error, say error at most \(0.01\),
they do not capture the optimal error rate.

In the non-robust setting, we can boost the accuracy of a solution with small error by several rounds of majority voting, where we assign each node the most common label among its neighbours.
However, this strategy fails dramatically in the presence of node corruptions.
Concretely, even without changing the degrees of nodes, an adversary can fool majority voting by corrupting only an \(O(\eta)\) fraction of nodes \cite{Liu2022minimax}.

For simplicity, suppose that we start with the ground-truth label assignment \(\bm x\) and an adversary can select random vertex subsets \(S\) and \(T\) with \(\card S=O(\eta n)\) and \(\card T=n/2\).
Then, the adversary can use \(S\) to poison the votes for \(T\) by setting all edges from \(S\) to random nodes in \(T\) with the opposite label.
This corruption flips the majority vote outcome for most of the nodes in \(T\) with high probability.
Hence, majority voting applied to \(\bm x\) results in an assignment that agrees with \(\bm x\) in a random subset of roughly half of the nodes and thus has large error.

In this example, a tiny fraction of nodes is responsible for roughly half of the majority vote outcomes.
However, in a random-like graph, the expander mixing lemma certifies that small fractions of nodes cannot change many majority vote outcomes.
Concretely, in the previous example, the expander mixing lemma implies for the uncorrupted graph that the number of edges between \(S\) and \(T\) is \(\tfrac d n \card S \cdot \card T \pm O( d \cdot \card S \cdot \card T)^{1/2}=\tfrac 12 \e d n\cdot (1 \pm O(\e d)^{-1/2})\).
On the other hand, in the corrupted graph we have about \(\e d n\) edges between these sets, which is about two times larger than before the corruption as \(\e d \gg \e^{2} d\gg 1\) when the blocks are recoverable.

More formally, we can limit the effect of small sets on majority vote outcomes if we have a label assignment \(x\) consistent with the majority votes for most of the nodes such that the matrix \(E=\tfrac n {\e d} G -\dyad x\) satisfies the following pseudorandomness constraints, strengthening \(\cArand(E, R; \lambda)\),
\begin{equation}
  \label{eq:mixing-constraints}
  \cAmix(E,R;\lambda) \seteq
  \Set{R\colon
    \sststile{2}{u,v}
    \iprod{u, E v} \le \lambda n \cdot \tfrac 12(\norm{u}^{2}+\norm{v}^{2})
  }
  \,.
\end{equation}
Since the previous pseudorandomness constraints \(\cArand\) defined in \cref{eq:pseudorandom} consider vectors \(u,v\) with \(\norm{u}=\norm{v}=\sqrt n\), the above normalization for \(\lambda\) is the same as in \cref{eq:pseudorandom}.
We remark that due to the characterization of the above inequality in terms of matrix factorizations an equivalent formulation of the above system would be \(\Set{\dyad R=\lambda n\cdot I_{n}- E}\).
We use the following polynomial constraints to encode that \(x\) is consistent with the majority votes for most of the nodes,
\begin{equation}
  \label{eq:majority-vote-constraints}
  \cAmaj(G,x,R; \beta, \gamma) \seteq
  \Set{R \colon \cAset(z)\sststile{2}{z} \iprod{D(z) x, Gx} \ge (1-\gamma)\e d \Paren{\card z - \beta n} }
  \,,
\end{equation}
where \(\card z=z_{1} + \cdots + z_{n}\) and \(\cAset(z) \seteq \set{\cD(z)^{2}=\D(z)}\).
Note that for the stochastic block model \((\bm G,\bm x)\sim \SBM_{n}(d,\e)\),
we expect that \((G x)_{i}\approx \e d \cdot x_{i}\) for most nodes \(i\in [n]\).
Since the inequality in \cref{eq:majority-vote-constraints} is linear subject to simple constraints,
sum-of-squares proofs capture all valid inequalities of this type up to a small error.
If the inequality in \cref{eq:majority-vote-constraints} is valid, then we can have at most \(\beta n/\gamma\) nodes \(i\) with \((G x)_{i} \cdot x_{i}\le (1-2 \gamma) \e d\).
At the same time, if the entries \((G x)_{i}\) with \((G x)_{i}\cdot x_{i}\le (1-\gamma) \e d\) sum\footnote{
  Note that for random-like graphs, the expander mixing lemma implies a bound on the sum of these entries as soon as we have a bound on the number of entries with \((G x)_{i}\) with \((G x)_{i}\cdot x_{i}\le (1-\gamma) \e d\).
} in absolute value to at most \(\beta n\), then the above inequality is valid.

Liu and Moitra~\cite{Liu2022minimax} show that the stochastic block model satisfies the above constraints for \(\lambda\lesssim (\e\sqrt d)^{-1}\) and \(\ln 1/\beta \approx_{o(1)} \gamma  C_{d,\e}/2\) with high probability (after removing a negligible number of nodes).
They also provide an iterative algorithm that boosts the accuracy of a solution to a rate that is asymptotically optimal.

Simplifying their approach significantly, we show that the sum-of-squares meta-algorithm can boost to such a rate in one shot.
To this end, we combine the previous constraints as follows,
\begin{equation}
  \label{eq:boost-constraints}
  \cAboost(G,x) \seteq
  \cAmix(\tfrac n {\e d} G -\dyad x; \lambda)
  \cup \cAmaj(G,x; \beta_{1}, \gamma_{1})
  \cup \cAmaj(G,x; \beta_{2}, \gamma_{2})
  \,,
\end{equation}
where we choose \(\lambda\lesssim (\e \sqrt d)^{-1}\), \(\gamma_{1}=0.01\), \(\beta_{1}=\lambda\),   \(\gamma_{2}=1-10\lambda\), and \(\beta_{2}=\exp(-C_{d,\e}/2 + \chi \sqrt C_{d,\e})\) for some sufficiently large absolute constant $\chi$.
The following property of these constraints allows us to boost the accuracy of a solution,
\begin{equation}
  \label{eq:boost-sos-proof}
  \tfrac 1n \norm{x - x'}^{2}\le 0.01,
  \cAboost(G,x),\cAboost(G,x')
  \sststile{}{G,x,x'}
  \tfrac 1 n\norm{x-x'}^{2}\le \exp\Paren{-C_{d,\e}/2 + O(\sqrt C_{d,\e})}
  \,.
\end{equation}
For similar reasons as before, we can use Lipschitz and restriction properties of these constraints in order to obtain error bounds in the robust setting at the cost of an additional \(O(\eta)\) error for an \(O(\eta)\) fraction of node corruptions.

Toward establishing \cref{eq:boost-sos-proof}, we consider the vector \(z\) with \(z_{i}=\tfrac 1 4\bigparen{ x_{i}-x'_{i}}^{2}\).
Since \(x\) and \(x'\) have coordinates in \(\{\pm 1\}\), the vector \(z\) has coordinates in \(\{0,1\}\).
Indeed, \(\cAlabel(x),\cAlabel(x')\sststile{}{} \cAset(z)\).
Furthermore, the vector \(z\) also satisfies \(\D(z)x=-\D(z)x'\) and \(x-x' = \D(z)\cdot (x-x')\)
because \(z\) indicates the set of coordinates, where \(x\) and \(x'\) differ.

Now since the labelings \(x\) and \(x'\) agree in at least a \(0.99\) fraction of the nodes, the majority voting based on \(x\) or \(x'\) will agree in all but a $\beta$ fraction of the nodes (when there is no node corruption).
This gives us the intuition that the constraints from expander mixing lemmas (namely \(\cAmaj(G,x;\beta,\gamma)\) and \(\cAmaj(G,x';\beta,\gamma)\)) ensure that \(x\) and \(x'\) are close to each other in distance $\beta$.

Concretely, using these relations between \(z\), \(x\), and \(x'\) together with majority-voting consistency constraints \(\cAmaj(G,x;\beta,\gamma)\) and \(\cAmaj(G,x';\beta,\gamma)\), we get
\begin{align*}
  \cAmaj(G,x;\beta,\gamma), \cAmaj(G,x';\beta,\gamma) \sststile{}{G,x,x'}  \quad
  & (1-\gamma )\e d (\card z - \beta n)\\
  & \le \tfrac 12 \Bigparen{\iprod{\D(z)x,Gx} + \iprod{\D(z)x',Gx'}} \\
  & = \tfrac 1 2 \iprod{\D(z) x, G (x-x')}\\
  & = \iprod{\D(z) x , G \D(z) x}
    \,.
\end{align*}
We can upper bound the quantity on the right using the constraints \(\cAmix(E;\lambda)\) for the matrix \(E=\tfrac n {\e d} G - \dyad x\) as well as the label-assignment and set constraints for \(x\) and \(z\),
\begin{align*}
  \cAmix(E;\lambda),\cAset(z),\cAlabel(x)
  \sststile{}{E,z,x}\quad 
  & \iprod{\D(z) x , G \D(z) x}\\
  & = \tfrac {\e d}{n}\Paren{ \iprod{x ,\D(z) x}^{2} + \iprod{\D(z) x, E \D(z) x}}\\
  & \le \tfrac{\e d}{n}\Paren{ \norm{\D(z) x}^{4} + \lambda  n \cdot \norm{\D(z) x}^{2}} \\
    & = \tfrac{\e d}{n}\Paren{ \card{z}^{2} + \lambda n \cdot \card{z}}
\end{align*}
Combining the lower and upper bounds on \(\iprod{\D(z) x, G \D(z) x}\), we get the following upper bound on \(\card z\),
\begin{align}
  \label{eq:z-boosting-bound}
  \cAmaj(G,x;\beta,\gamma), \cAmaj(G,x';\beta,\gamma), \cAmix(E;\lambda),\cAset(z),\cAlabel(x),\cAlabel(x') & \\ \sststile{}{G,x,x',E,z,x} 
  \card z \le \beta n + \tfrac 1 {(1- \gamma - \lambda)n } \card{z}^{2}\,.
\end{align}
While this bound alone is not enough to conclude an absolute bound on \(\card z\), it does allow us to boost assumed bounds on \(\card z\).
Suppose \(\gamma + \lambda \le 1\) and \( \lambda \le \tfrac 1 2 \paren{1-\gamma}\).
Then, \(\tfrac 1 {1-\gamma-\lambda} \le \tfrac 2 {1-\gamma}\).
Moreover, if we add the constraint \(\card {z} \le \tfrac {1-\gamma}{4} n\) to the previous constraints, we can derive \(\tfrac 1 {(1-\gamma -\lambda)n}\card{z}^{2}\le \tfrac 12 \card z\) and thus, by subtracting \( \tfrac 12 \card z\) from both sides of the inequality \cref{eq:z-boosting-bound},
\begin{align}
  \label{eq:final-z-boosting-bound}
  \card {z} \le \tfrac {1-\gamma}{4} n, \card z \le \beta n + \tfrac 1 {(1- \gamma - \lambda)n } \card{z}^{2} \sststile{}{z}
  \card z \le 2 \beta n \,.
\end{align}
Our target derivation \cref{eq:boost-sos-proof} contains the inequality \(\card{z}=\norm{x-x'}^{2}\le 0.01n\) as an initial assumption.
Hence, we get to assume \(\card z \le \tfrac {1-\gamma_{1}} 4 n\) from the get-go as \(\gamma_{1}=0.01\).
By substituting  \(\gamma_{1}\) and \(\beta_{1}\) for \(\gamma\) and \(\beta\) in \cref{eq:final-z-boosting-bound}, we derive \(\card z\le 2\beta_{1}n\).
Since \(\gamma_{2} = 1- 2\lambda - 8\beta_{1}\), we can compose this derivation with \cref{eq:final-z-boosting-bound} with  \(\gamma_{2}\) and \(\beta_{2}\) substituted for \(\beta\) and \(\gamma\).
In this way, we derive the desired bound \(\card z\le \beta_{2} n\) for \(\ln \beta_{2} = -C_{d,\e}/2 +O(\sqrt C_{d,\e})\) on the distance between the label assignments \(x\) and \(x'\).

\paragraph{Concluding the two-community case}
Using the polynomial constraints and sum-of-squares proofs discussed so far, we obtain a polynomial-time algorithm based on the sum-of-squares method for robust recovery in stochastic block models with two communities up to asymptotically optimal error.
Concretely, supposing \(d\ge O(1/\e^{2})\), we use the constraints \(\cArobust\) described in \cref{eq:corrupted-constraints} in order to obtain with high probability a label assignment \(\tilde x\in \{\pm 1\}^{n}\) with error at most \(0.005n\).
Then, we solve the sum-of-squares program for the following combination of constraints discussed in this paragraph,
\begin{equation}
  \label{eq:boot-robust-constraints}
  \cArobustboost(G,x;\eta) \seteq
  \begin{aligned}[t]
    & \cAclose(G,\Gcorrupted;\eta)  \cup \Set{\norm{x-\tilde x}^{2}\le 0.005n}\\
    & \cup \cAboost(G,x)
      \,.
  \end{aligned}
\end{equation}
Here, \(\Gcorrupted\) is the \(\eta\)-node-corrupted input graph and \(\tilde x\) is the rough label assignment with error at most \(0.005\) obtained in the preprocessing phase of the algorithm.
As discussed before, the stochastic block model \((\bm G,\bm x)\sim \SBM_{n}(d,\e)\) satisfies the above constraints with high-probability uniformly for all \(\eta\)-corruptions on \(\bm G\) when we substitute \(\bm G\) and \(\bm x\) for \(G\) and \(x\) (after pruning a tiny fraction of nodes with high degree).
By \cref{eq:boost-sos-proof}, any two solutions \(x\) and \(x'\) for the same graph \(G\) are \(\exp(-C_{d,\e}/2+O(\sqrt C_{d,\e}))\)-close.
(While our constraints do not include the inequality \(\tfrac 1n\norm{x-x'}^{2}\le 0.01\), we can derive it directly from the two included inequalities \(\tfrac 1n\norm{x-\tilde x}^{2}\le 0.005\) and \(\tfrac1n \norm{x'-\tilde x}^{2}\le 0.005\) by \(\ell_{2}^{2}\) triangle inequality.)
By restriction arguments analogous to \cref{eq:lipschitz-basic}, we can allow for \(\eta\)-close graphs \(G\) and \(G'\) at the cost of an additional \(O(\eta)\) error.

We emphasize that this algorithm relies on the same structural properties of the stochastic block model as the algorithm by Liu and Moitra \cite{Liu2022minimax}.
Concretely, both algorithms try to rule out a small fraction of nodes swaying a disproportionate number of majority vote outcomes.
However, the final algorithms end up being quite different.
While our algorithm and its analysis exploit the structural properties directly and don't require additional concepts besides the generic SOS framework, Liu and Moitra introduce several non-trivial technical tools such as the notion of resolvable matrices.
Furthermore, they need to keep track of invariants satisfied throughout the iterations of their algorithm.
Our algorithm avoids this complication by boosting to ``full accuracy'' in one shot.

\subsection{More than two communities}


In order to obtain high-accuracy label assignments for more than two communities,
we can extend the basic recovery approach.
For \(k\) communities, we consider the stochastic block model \((\bm G,\bm x)\sim \SBM_n(d, \varepsilon, k)\).
Here, \(\bm x\) is a random balanced assignment of labels from \([k]\) to nodes \([n]\).
In particular, the blocks / communities are random subsets of nodes with \(\card{\bm x^{-1}(1)}=\cdots=\card{\bm x^{-1}(k)}=\tfrac n k\).
For every pair of distinct nodes \(i,j\in [n]\), we decide independently at random whether to join them by an edge with probability,
\begin{equation}
  \label{eq:k-sbm-edge-probability}
  \Pr\Set{ij \in E(\bm G) \Mid \bm x}
  =
  \begin{cases}
    \Paren{1+(1-\tfrac 1 k)\e}\cdot \frac dn
    & \text{ if \(\bm x(i)=\bm x(j)\),}\\
    \Paren{1-\tfrac 1 k \cdot \e}\cdot \frac dn
    & \text{ if \(\bm x(i)\neq \bm x(j)\).}
  \end{cases}
\end{equation}
For \(k=2\), the above model \(\SBM_{n}(d,\e,2)\) corresponds to the previously introduced model \(\SBM_{n}(d,\e/2)\).
The Kesten-Stigum threshold for this model, conjectured to coincide with the computational threshold for weak recovery, is \(d>k^{2}/\e^{2}\) \cite{abbe2015community,Hopkins18,bandeira2021spectral,sohn2025sharp}.
Indeed, for \(\tfrac {\e^{2} d}{k^{2}}\to \infty\) and sufficiently large \(n\), we know of polynomial-time algorithms that achieve error tending to \(0\) (but not necessarily at an optimal rate) \cite{guedon2016community}.

\paragraph{Rough initialization}
First, we briefly describe a sum-of-squares formulation that achieves vanishing error under the above condition for more than two communities.
We denote the constraints for balanced label assignments as follows,
\begin{equation}
  \label{eq:general-label-assignments}
  \cAlabel(x)
  \seteq \Set{ x(i,a)^{2}=x(i,a), ~ \sum_{a}x(i,a)=1, ~ \sum_{i} x(i,a)=\tfrac n k}
  \,.
\end{equation}
Here $x \in \R^{n \times k}$ and we denote by \(x_{a}\) the vector with entries \(x(1,a),\ldots,x(n,a)\) indicating the set of nodes assigned label \(a\).

We consider the following constraints as analogue for basic recovery,
\begin{align}
  \label{eq:basic-recovery-general}
  \cAbasic(G,x)
  \seteq \cAlabel(x) \cup \cAmix\Paren{E ; \lambda}
  \,,
\end{align}
where \(E\seteq \tfrac n {\e d} G + \tfrac 1k \dyad \Ind -\sum_{a=1}^{k} \dyad{(x_{a})}\) and \(\lambda \lesssim (\e \sqrt d)^{-1}\).
Suppose we have two solutions \(x^{(1)}\) and \(x^{(2)}\) of the above constraints for the same graph \(G\).
Let \(X^{(t)}\seteq \sum_{a=1}^{k}\dyad{\Paren{x^{(t)}_{a}}}\) for \(t\in \set{1,2}\).
Then, 
\begin{align}
  \label{eq:general-identifiability-proofs}
  \normf{X^{(1)}-X^{(2)}}^{2}
  & = \iprod{E^{(2)}-E^{(1)},X^{(1)}-X^{(2)}}\\
  & \le \sum_{a} \iprod{x^{(1)}_{a},(E^{(2)}-E^{(1)})x^{(1)}_{a}}
  +  \sum_{a} \iprod{-x_{a}^{(2)},(E^{(2)}-E^{(1)})x^{(2)}_{a}}  \\
  & \le 2\lambda n \cdot \sum_{a=1}^{k}\norm{x_{a}}^{2} = 2\lambda  n^{2}
  \,,
\end{align}
where \(E^{(t)}\seteq \tfrac n {\e d} G + \tfrac 1k \dyad \Ind -\sum_{a=1}^{k} \dyad{(x^{(t)}_{a})}\) for \(t\in\set{1,2}\).
Since the derivation \cref{eq:general-identifiability-proofs} uses only steps captured by the sum-of-squares proof system, we obtain an SOS proof for the following error bound,
\begin{equation}
  \label{eq:general-identifiability-proofs-sos}
  \cAbasic(G,x^{(1)}),  \cAbasic(G,x^{(2)})
  \sststile{2}{G,x^{(1)},x^{(2)}}
  \normf{X^{(1)}-X^{(2)}}^{2} \le 2\lambda n^{2}
  \,.
\end{equation}
Since \(\normf{X^{(1)}}^2=\normf{X^{(2)}}^2=n^{2}/k\), the above error bounds allow us to recover the \(k\) communities with small error whenever \(1 \ll (k \lambda)^{-2} = \e^{2} d/ k^{2}\).
Furthermore, there is also a robust version of this sum-of-squares proof (see \cref{lem:intialization-identifiability})
\begin{equation}
  \label{eq:general-identifiability-proofs-sos-robust}
  \cAclose(G_1,G_2;\eta),
  \cAbasic(G_1,x^{(1)}),  \cAbasic(G_2,x^{(2)}) 
  \sststile{4}{G_1, G_2, x^{(1)},x^{(2)}}
  \normf{X^{(1)}-X^{(2)}}^{2} \le 2\lambda n^{2} + O(\eta/k) \cdot n^{2}
  \,.
\end{equation}
The above sum-of-squares proof directly implies a polynomial-time algorithm to robustly recover communities for stochastic block models with more than two communities (under node corruptions).

\paragraph{Boosting accuracy for more than two communities}

While the rough initialization allow us to recover a label assignment with error at most \(0.001\) when \(d\gg k^{2}/\e^{2}\), we aim to boost the error close to the asymptotically optimal bound \(\exp(-C_{d,\e}/k)\).
A natural extension of the previous majority voting strategy would assign to every node the most common label among its neighbours.
For this strategy to succeed, even without node corruptions, we must start with a label assignment that has small error for every block.
However, a label assignment with error at most \(0.001\) could assign random labels to a \(0.001\) fraction of the blocks.
In this case, simple voting schemes cannot improve the accuracy for these blocks and the total error would stay at least \(0.001\).
On the other hand, if the initial label assignment had error \(\ll 1/k\), we can conclude that the assignment has small error for every block and, in the non-robust setting, we can boost by the above voting scheme to an error of roughly \(\exp(-C_{d,\e}/k)\).
However, previous algorithms, including the robust algorithm by Liu and Moitra~\cite{Liu2022minimax}, can guarantee this small an error only if \(d > k^{2+c}/\e^{2}\) for a constant \(c>0\).
(Indeed, the error bound \cref{eq:general-identifiability-proofs-sos} requires \(d\gg k^{4}/\e^{2}\) to guarantee error at most \(1/k\).)

To overcome this inherent technical barrier, we develop a recursive bisectioning algorithm for obtaining error rate $1/\poly(k)$ when $\e^2 d\geq Kk^2\log k$ for some sufficiently large constant $K$.

\paragraph{Robust verification} A crucial procedure in the robust bisectioning algorithm is to test whether a cluster output by the initialization algorithm well approximates one of the ground truth clusters.
In other words, given a set of $n/k$ vertices, the algorithm needs to decide whether at least $0.99$ fraction of the nodes belong to the same cluster. 
For this, we consider the following sum-of-squares program.
\begin{gather}
  \cA_{\text{verify}}(z; G,\eta) \seteq 
  \left. 
      \begin{cases}
      z \odot z = z,\, \iprod{z, \Ind} \geq \Paren{\frac{0.99}{k}-\eta}n \\
      \norm{\Paren{G-\frac{a}{n} J} \odot z z^{\top}} \leq \chi\sqrt{d/k} \\
      \Iprod{\Paren{G-\frac{d}{n} J} \odot z z^{\top}, \Ind \Ind^{\top}} \geq \frac{0.97 (k-1) \e d n}{k^3}
      \end{cases}
  \right\} \,, 
\end{gather}
where $a/n$ is the intra-cluster connection probability, and $\chi$ is a large constant. 
When a $0.99$-fraction of the nodes indeed belongs to the same cluster, it is easy to see that the program is satisfied by setting \(z\) as the indicator vector for the set of uncorrupted vertices from the ground truth cluster.
On the other hand, we give a constant-degree sum-of-squares refutation for the program when less than $0.98$ fraction of the nodes belong to the same cluster. 
The observation is that, having a large number of nodes from the other community leads to a large spectral norm of the matrix $\Paren{G-\frac{a}{n} J} \odot z z^{\top}$.
Let $\xi$ be the indicator vector for the set of vertices belonging to other clusters.
As the inter-cluster edge connection probability is given by $a/n-\e d/n$, we have 
\begin{equation*}
  \Abs{\iprod{\xi\odot z, \Paren{G-\frac{a}{n} J} z}}\geq \frac{\e d}{k}\norm{z} \norm{\xi\odot z} \,,
\end{equation*}  
which violates the spectral norm constraint.

\paragraph{Robust bisectioning algorithm}
Now, we show that, given the rough initialization with error rate $0.001$, how to find a bisection of the graph such that, for each community, at least $1-\exp\left(-\e^2 d/2k^2\right)+\poly(k) \eta$ fraction of the nodes is assigned to the same side of the bisection.

To do this, we use the robust verification algorithm on the recovered communities from the rough initialization.
Since the rough initialization has error $0.001$, among the recovered communities, there can be at most $k/2$ communities that has error more than $0.01n/k$.
Therefore, we can obtain $k/2$ clusters that each has error at most \(0.01n/k\) via the robust verification algorithm, and we can put them in the same side of the bisection and form a rough bisection with bisection error $0.01n$.
Now, we treat the bisection as a 2-SBM and use majority voting to boost the bisection accuracy to $\exp(-\e^2 d/2k^2)+\poly(k) \eta$. \footnote{The error rate here is of order $\exp(-\e^2 d/2k^2)$ because the voting signal is diluted after putting the communities in partitions of $\frac{k}{2}$ communities.}
This leads to an algorithm with error rate $1/\poly(k)+O(\eta)$ when \(\e^2 d\geq Kk^2\log k\).\footnote{This is the bottleneck for us to get to the KS threshold.}
 
\paragraph{Robust $k$-clustering algorithm} Given the robust bisectioning algorithm, we can now recursively bisect the graph and get a better initialization with accuracy $1/\poly(k)+\poly(k)\cdot\eta$ when $\e^2 d\geq Kk^2\log k$.
More precisely, the algorithm robustly partitions the given graph into a bisection with error rate $\exp(-\e^2 d/2k^2)+\poly(k) \eta$, then recursively applies the same procedure to each side of the bisection.
Since the algorithm is robust against node corruptions, it can handle the misclassified nodes from the previous rounds by treating them as corruptions.
As a result, we can obtain a clustering with error rate $1/\poly(k)+\poly(k)\eta$.

\paragraph{Achieving minimax rate.}
The last step is to boost the rough \( k \)-clustering with misclassification error \( \frac{1}{\mathrm{poly}(k)} + \mathrm{poly}(k) \cdot \eta \) to the minimax-optimal error rate.
\[
\exp\left(-(1 \pm o(1)) \frac{C_{d,\varepsilon}}{k}\right) + \mathrm{poly}(k) \cdot \eta.
\]
Notice that, when there is no node corruptions, simple majority voting can boost clustering error to \( \exp(-C_{d,\varepsilon}/k) \).
To deal with node corruptions, we use the robust majority voting algorithm for 2-SBM (as described in previous section) for each pair of the communities.
More precisely, we simulate majority voting within the SoS framework by imposing majority-vote and robust mixing constraints for every pair of communities.
These constraints mirror the 2-community scenario and establish that any feasible solution in the SoS program must differ from the ground truth by no more than a \( \exp(-C_{d,\varepsilon}/k) \) fraction of nodes.\footnote{For technical reasons, two rounds of majority voting are needed to achieve the optimal error rate. Correspondingly, in our sum-of-squares program, we need to add two sets of majority voting constraints for achieving the minimax error rate.}

To round the SoS solution to a community labelling, we apply the standard node-distance Lipschitzness arguments and obtain a final clustering with misclassification error at most
\[
  \exp\left(-(1-o(1))\frac{C_{d,\varepsilon}}{k}\right) + \mathrm{poly}(k) \cdot \eta \,.
\]

\subsection{Relation to previous works}

\paragraph{Comparison with \cite{Liu2022minimax}.}
One of the main reasons why \cite{Liu2022minimax} requires \(\e^2 d\geq \Omega(k^{102})\) is that their initialization algorithm is the vanilla SDP which can only achieve error $1/\mathrm{poly}(k)$ when it is far from the KS threshold.
One key algorithmic innovation of our work is the bisection boosting procedure.
Instead of attempting to get an initialization with error $1/\mathrm{poly}(k)$ in one shot, we recursively find bisections of the graph that have the optimal bisection error rate.
More concretely, in each recursion, we first get a rough initialization with error $0.001$, then boost the bisection accuracy to its optimal rate using bisection majority voting on the rough initialization.
This allows us to get a  $1/\mathrm{poly}(k)$ initial estimation when \(\e^2 d\geq \Omega(k^{2} \log k)\).

\paragraph{Comparison with \cite{pmlr-v195-hua23a}.}
In \cite{pmlr-v195-hua23a}, the authors gave an SoS-based algorithm that achieves weak recovery at the KS threshold. 
Our algorithm can be viewed as a strengthening of the SoS relaxation from \cite{pmlr-v195-hua23a} by adding \emph{majority voting constraints}. 
Informally, these constraints enforce that, for most vertices, the assigned label agrees with the label suggested by the majority of their neighbours in an appropriate pairwise sense. 
Since all known approaches that achieve the minimax error rate ultimately rely on belief propagation or majority voting, 
encoding such voting constraints into the SoS program appears essential for attaining optimal accuracy.

\section{Preliminary}
\label{sec:preliminary}

\subsection{Notation}

\paragraph{Random variables} We use boldface to denote random variables, e.g., \(\bm X, \bm Y, \bm Z\).

\paragraph{Asymptotics} We write \(f \lesssim g\) to denote the inequality \(f \le C \cdot g\) for some absolute constant \(C>0\).
We write \(O(f)\) and \(\Omega(f)\) to denote quantities \(f_-\) and \(f_+\) satisfying \(f_-\lesssim f\) and \(f \lesssim f_+\), respectively.

\paragraph{Vector} For any vector, we use the notation $\norm{\,\cdot\,}_1$ for the $\ell_1$-norm, $\norm{\,\cdot\,}_2$ for the $\ell_2$-norm, and $\norm{\,\cdot\,}_\infty$ for the $\ell_\infty$-norm.
Furthermore, for any two vectors $x,y\in \R^n$, we use $x\odot y$ for the Hadamard (entrywise) product of the two vectors.
We write $\one$ and $\zeros$ for the all-ones and all-zeros vectors (or matrices) of the appropriate dimension.
For a subset $S\subseteq[n]$, we denote its indicator vector by $\one_S\in\{0,1\}^n$, where $\one_S(i)=1$ if $i\in S$ and $\one_S(i)=0$ otherwise.

\paragraph{Matrix} For a matrix \(M\in \R^{n\times m}\), we denote its \((i,j)\)-th entry by \(M(i,j)\), its \(i\)-th row by \(M(i,\cdot)\), and its $j$-th column by \(M(\cdot, j)\).
For two matrices of the same dimension, we also use $M\odot N$ for their Hadamard (entrywise) product.
We use \(\norm{M}\) (and also \(\Normop{M}\)) for the spectral norm of \(M\) and \(\normf{M}\) for the Frobenius norm of \(M\).
We denote by \(\normo{M}\) and \(\normm{M}\) the sum and the maximum of the absolute values of the entries in \(M\), respectively.
For two matrices \(M,N\in \R^{n\times m}\), we denote their inner product by \(\iprod{M,N} = \Tr M \transpose{N}=\sum_{i,j} M(i,j)N(i,j)\).
We write $I_n$ for the $n\times n$ identity matrix, and $J = \one \one^{\top}$ for the all-ones matrix.

\paragraph{Graph} We write $G\in\{0,1\}^{n\times n}$ for the adjacency matrix of a (possibly corrupted) graph on vertex set $[n]$.
When an average degree parameter $d$ is specified, we define the centered adjacency matrix
\[
  \bar{G} \coloneqq G-\frac{d}{n}J.
\]
When we need to distinguish the uncorrupted graph sampled from the stochastic block model from a corrupted observation, we write $G^{\circ}$ for the uncorrupted adjacency matrix and $\bar{G}^{\circ} \coloneqq G^{\circ} - \frac{d}{n}J$ for its centered version.

\subsection{Sum-of-Squares hierarchy}

In this paper, we employ the sum-of-squares hierarchy \cite{barak2014sum,barak2016proofs,raghavendra2018high} for both algorithm design and analysis. 
As a broad category of semidefinite programming algorithms, sum-of-squares algorithms provide many optimal or state-of-the-art results in algorithmic statistics~\cite{hopkins2018mixture,KSS18,pmlr-v65-potechin17a,hopkins2020mean}.
We provide here a brief introduction to pseudo-distributions, sum-of-squares proofs, and sum-of-squares algorithms.

\paragraph{Pseudo-distribution} 

We can represent a finitely supported probability distribution over $\R^n$ by its probability mass function $\mu\from \R^n \to \R$ such that $\mu \geq 0$ and $\sum_{x\in\supp(\mu)} \mu(x) = 1$.
We define pseudo-distributions as generalizations of such probability mass distributions, by relaxing the constraint $\mu\ge 0$ and only requiring that $\mu$ passes certain low-degree non-negativity tests.

\begin{definition}[Pseudo-distribution]
  \label{def:pseudo-distribution}
  A \emph{level-$\ell$ pseudo-distribution} $\mu$ over $\R^n$ is a finitely supported function $\mu:\R^n \rightarrow \R$ such that $\sum_{x\in\supp(\mu)} \mu(x) = 1$ and $\sum_{x\in\supp(\mu)} \mu(x)f(x)^2 \geq 0$ for every polynomial $f$ of degree at most $\ell/2$.
\end{definition}

We can define the formal expectation of a pseudo-distribution in the same way as the expectation of a finitely supported probability distribution.

\begin{definition}[Pseudo-expectation]
  Given a pseudo-distribution $\mu$ over $\R^n$, we define the \emph{pseudo-expectation} of a function $f:\R^n\to\R$ by
  \begin{equation}
    \tE_\mu f \seteq \sum_{x\in\supp(\mu)} \mu(x) f(x) \,.
  \end{equation}
\end{definition}
In later sections, when the underlying pseudo-distribution $\mu$ is implicit, we write $\tilde{\E}$ for its pseudo-expectation operator and use the shorthand $\tilde{\E}[p]$ for $\tE_\mu p$.

The following definition formalizes what it means for a pseudo-distribution to satisfy a system of polynomial constraints.

\begin{definition}[Constrained pseudo-distributions]
  Let $\mu:\R^n\to\R$ be a level-$\ell$ pseudo-distribution over $\R^n$.
  Let $\cA = \{f_1\ge 0, \ldots, f_m\ge 0\}$ be a system of polynomial constraints.
  We say that \emph{$\mu$ satisfies $\cA$} at level $r$, denoted by $\mu \sdtstile{r}{} \cA$, if for every multiset $S\subseteq[m]$ and every sum-of-squares polynomial $h$ such that $\deg(h)+\sum_{i\in S}\max\set{\deg(f_i),r} \leq \ell$,
  \begin{equation}
    \label{eq:constrained-pseudo-distribution}
    \tE_{\mu} h \cdot \prod_{i\in S}f_i \ge 0 \,.
  \end{equation}
  We say $\mu$ satisfies $\cA$ and write $\mu \sdtstile{}{} \cA$ (without further specifying the degree) if $\mu \sdtstile{0}{} \cA$.
\end{definition}

We remark that if $\mu$ is an actual finitely supported probability distribution, then we have  $\mu\sdtstile{}{}\cA$ if and only if $\mu$ is supported on solutions to $\cA$.

\paragraph{Sum-of-squares algorithm}
The \emph{sum-of-squares algorithm} searches through the space of pseudo-distributions that satisfy a given system of polynomial constraints, by solving semideﬁnite programming.

\begin{theorem}[Sum-of-squares algorithm]
  \label{theorem:SOS_algorithm}
  There exists an $(n+ m)^{O(\ell)} $-time algorithm that, given any explicitly bounded\footnote{A system of polynomial constraints is \emph{explicitly bounded} if it contains a constraint of the form $\|x\|^2 \leq M$.} and satisfiable system\footnote{Here we assume that the bit complexity of the constraints in $\cA$ is $(n+m)^{O(1)}$.} $\cA$ of $m$ polynomial constraints in $n$ variables, outputs a level-$\ell$ pseudo-distribution that satisfies $\cA$ approximately.
\end{theorem}

\paragraph{Sum-of-squares proof} 
We introduce sum-of-squares proofs as the dual objects of pseudo-distributions, which can be used to reason about properties of pseudo-distributions.
We say a polynomial $p$ is a sum-of-squares polynomial if there exist polynomials $(q_i)$ such that $p = \sum_i q_i^2$.

\begin{definition}[Sum-of-squares proof]
  \label{def:sos-proof}
  A \emph{sum-of-squares} proof that a system of polynomial constraints $\cA = \{f_1\ge 0, \ldots, f_m\ge 0\}$ implies $q\ge0$ consists of sum-of-squares polynomials $(p_S)_{S\subseteq[m]}$ such that\footnote{Here we follow the convention that $\prod_{i\in S}f_i=1$ for $S=\emptyset$.}
  \[
    q = \sum_{\text{multiset } S\subseteq[m]} p_S \cdot \prod_{i\in S} f_i \,.
  \]
  If such a proof exists, we say that \(\cA\) \emph{(sos-)proves} \(q\ge 0\) within degree \(\ell\), denoted by $\mathcal{A}\sststile{\ell}{} q\geq 0$.
  In order to clarify the variables quantified by the proof, we often write \(\cA(x)\sststile{\ell}{x} q(x)\geq 0\).
  We say that the system \(\cA\) is \emph{sos-refuted} within degree \(\ell\) if $\mathcal{A}\sststile{\ell}{} -1 \geq 0$.
  Otherwise, we say that the system is \emph{sos-consistent} up to degree \(\ell\), which also means that there exists a level-$\ell$ pseudo-distribution satisfying the system.
\end{definition}

\subsection{Stochastic block model}

\begin{definition}[Symmetric balanced \(k\)-community stochastic block model]
\label{def:SBM_k}
Assume that \( k \) is a power of 2 and \( n \) is a multiple of \( k \).
Let \( (\bm{G}, \bm{Z}^{\circ}) \sim \SBM_n(d, \varepsilon, k) \) denote the symmetric balanced \( k \)-community stochastic block model with degree parameter \( d > 0 \) and bias parameter \( \varepsilon \in [0,1] \).
The model is defined as follows:
\begin{itemize}
  \item \textbf{Communities.}
  The communities $\bm{Z}^{\circ}$ are selected uniformly at random from even partitions of the nodes, i.e. uniformly at random partition the graph into $k$ sets of \( \frac{n}{k} \) vertices each.

  \item \textbf{Edges.}
  For every pair of distinct vertices \( i, j \in [n] \), they are connected by an edge independently at random with probability
  \[
    \Pr\Set{(i, j) \in E(\bm{G})} =
    \begin{cases}
      p_1 = \left(1 + \left(1 - \frac{1}{k}\right)\varepsilon\right)\dfrac{d}{n}, & \text{if } \bm{Z}^{\circ}(i) = \bm{Z}^{\circ}(j), \\[6pt]
      p_2 = \left(1 - \dfrac{\varepsilon}{k}\right)\dfrac{d}{n}, & \text{otherwise}.
    \end{cases}
  \]
\end{itemize}
\end{definition}

\begin{definition}[Community membership matrix]\label{def:community-matrix}
  For a given community membership matrix $Z \in \{0,1\}^{n\times k}$ we define the associated centered community membership matrix
  \[
    X \coloneqq ZZ^\top - \frac{1}{k}J.
  \]
  In particular, the true centered block matrix is
  \[
    X^{\circ} \coloneqq Z^{\circ}(Z^{\circ})^\top - \frac{1}{k}J \,.
  \]
  Let $G^{\circ}$ be the uncorrupted adjacency matrix and $\bar{G}^{\circ} \coloneqq G^{\circ} - \frac{d}{n}J$ be the centered adjacency matrix, it follows that
  \[
    \E\bigl[\bar{G}^{\circ} \,\big|\, Z^{\circ}\bigr] = \frac{\varepsilon d}{n}\, X^{\circ}.
  \]
\end{definition}

\begin{definition}[Signal-to-noise ratio (SNR)]
\label{def:snr}
The \emph{signal-to-noise ratio (SNR)} for the symmetric balanced \( k \)-community stochastic block model (\cref{def:SBM_k}) is defined as
\[
    n I,\quad\text{where}\quad I = D_{1/2}\left(\Ber(p_1)\,\middle\|\,\Ber(p_2)\right)
\]
is the Rényi divergence of order \(1/2\) between two Bernoulli distributions with parameters \(p_1\) and \(p_2\), as specified in \cref{def:SBM_k}.
In the sparse regime where \( d = o(n) \), this expression simplifies to
\[
  nI = \left(\sqrt{p_1n}-\sqrt{p_2n}\right)^2 \cdot (1+o(1)).
\]
Throughout this paper, we denote this simplified expression by \(C_{d,\varepsilon}\).
\end{definition}
\section{Statistical properties of the SBM}\label{sec:stat-prop}

In this section, we focus on the statistical properties of majority voting in SBM.
At a high level, majority voting for a single vertex $u$ asks whether its centered signed degree $(\bar G y)_u \cdot y_u$ is positive. 
Aggregating this test over a set of candidate vertices $S$ amounts to lower bounding the linear form $\iprod{\bar G\,y\odot y,\,z}$ for indicators $z\in\{0,1\}^n$ supported on the level-$i$ bisection (or on a pair of communities in the $k$-clustering step). 
The voting bounds in \cref{thm:inner-product-lb,thm:inner-product-lb-k-clustering} show that, with high probability, this inner product is uniformly large whenever $\|z\|_1$ is not too small and remains not too negative even for very small $\|z\|_1$, 
while the masked variants in \cref{cor:inner-product-lb,cor:inner-product-lb-k-clustering} guarantee the same after restricting to a large trusted subset. 
We will use these properties to certify that a small set of possibly adversarial vertices cannot flip many majority votes at once and to enable robust boosting.
The proofs of the results in this section are deferred to \cref{sec:stat-prop-appendix}.

\paragraph{Concentration bound of single-vertex majority voting}
The following result provides a concentration bound for a mixture of binomial distributions.
The argument follows a generic Chernoff bound approach similar to lemma 5.6 in \cite{Liu2022minimax}.
The main difference is that we need to find the right bound (i.e. SNR) for induced bisections of $k$-SBM.
This boils down to bounding the moments of the following summation of binomial distributions.

\begin{theorem}
\label{thm:three_binom_conc}
Fix parameters $\beta \in (0, 1]$ and $\alpha \in (0, \beta]$.
Consider the distribution
\[
    \mathcal{D} = \mathrm{Binom}(\alpha n, a/n)
    + \mathrm{Binom}\bigl((\beta-\alpha)n,\, b/n\bigr)
    - \mathrm{Binom}(\beta n,\, b/n) \enspace,    
\]
where the three binomials are independent.
Let $\gamma\coloneqq \alpha/\beta$ and define
\[
\widetilde a \coloneqq a^{\gamma} b^{\,1-\gamma} \enspace,
\qquad \widetilde b \coloneqq b \enspace,
\qquad \widetilde C \coloneqq \bigl(\sqrt{\widetilde a}-\sqrt{\widetilde b}\,\bigr)^{2} \enspace,
\qquad R(p,q) \coloneqq \frac{p(1-q)}{q(1-p)} \enspace.
\]
Then for every $\theta\in\mathbb R$,
\[
\Pr_{\bm X \sim \mathcal{D}}[\,\bm X \le \theta\,]
\;\le\;
\exp \Bigl(-\,\beta\,\widetilde C \;+\; \frac{\theta}{2}\,\log R(\widetilde a/n,\, \widetilde b/n)\Bigr) \enspace.
\]
\end{theorem}

As an immediate corollary, we get:

\begin{theorem}\label{thm:mixture-binomial-concentration-bound}
Fix a level $i\in\{1,2,\ldots,\log_2 k\}$ and let $\beta_i \coloneqq 2^{-i}$ (with $\alpha_i = 1/k$). 
Let $\widetilde a,\widetilde b,\widetilde C$ and $R(\cdot,\cdot)$ be as in Theorem~\ref{thm:three_binom_conc}, evaluated at this $(\alpha, \beta)$.
Then for every $\theta\in\mathbb{R}$,
\[
\Pr[\bm X\le \theta] \le
\exp \left(
-\frac{(\log 2)^2}{4}\cdot \frac{d\,\e^{2}}{\beta_i\,k^{2}}
+\frac{\theta}{2}\log R(\widetilde a/n,\widetilde b/n)
\right).
\]
\end{theorem}

\paragraph{Majority voting for bisections}
We first give a theorem for the error of majority voting in induced bisections of the graph.
\begin{theorem} \torestate{\label{thm:inner-product-lb}
  Fix a level $i$ and set $\beta_i=2^{-i}$, $n_i=2 \beta_i n$, and $k_i=\beta_i k$. 
  Let $\gamma\in [0, 0.99]$. 
  Choose $\rho_\gamma\coloneqq \exp(-\gamma \beta_i \tilde{C_i}/2)$, and set $t=0.001(1-\gamma)\tilde{C_i}$.
  Then with probability at least $1-\exp(-100k)-\frac{1}{n^3}$, for every $z\in \Set{0,1}^n$, and for every valid bisections $y\in \Set{0,\pm 1}^n$ at level $i$. we have
  \begin{equation*}
    \iprod{\bar{G}y\odot y,z}\geq \frac{(1-\gamma)\e d}{8k} \Paren{\norm{z}_1-\frac{96 \rho_\gamma k n_i}{1-\gamma}}\,,
  \end{equation*}   
  where $\bar{G}$ is the centered adjacency matrix.
  }
\end{theorem}

For establishing the feasibility of the program constraints in \cref{sec:robust-bisection-boosting} and \cref{sec:Optimal-robust-clustering},
we need to bound the error of majority voting when a small fraction of the vertices in the graph are removed, i.e. the following corollary.
\begin{corollary}\torestate{
  \label{cor:inner-product-lb}
  Fix a level $i$ and set $\beta_i=2^{-i}$, $n_i= 2 \beta_i n$, and $k_i=\beta_i k$. 
  Let $\gamma\in [0, 0.99]$.
  Choose $\rho_\gamma\coloneqq \exp(-\gamma \beta_i \tilde{C_i}/2)$, and set $t=0.001(1-\gamma)\tilde{C_i}$.
  Then with probability at least $1-\exp(-100k)-\frac{2}{n^3}$, for every $z\in \Set{0,1}^n$, for every $s \in \Set{0,1}^n$ such that $\norm{s}_1 \geq (1-\exp(-2C_{d, \e}))n$, and for every valid bisections $y\in \Set{0,\pm 1}^n$ at level $i$, we have
  \begin{equation*}
    \iprod{\bar{G} \odot (s s^{\top}) y\odot y,z} \geq \frac{(1-\gamma)\e d}{16k} \Paren{\norm{z}_1 -\frac{640 \rho_\gamma k n_i}{1-\gamma}} \,,
  \end{equation*}
  where $\bar{G}$ is the centered adjacency matrix.
}
\end{corollary}

\paragraph{Majority voting for pairwise communities}
We give a similar theorem for bounding the error of majority voting for each pair of communities.
\begin{theorem} \torestate{\label{thm:inner-product-lb-k-clustering}
  Let $C_{d,\epsilon}=(\sqrt{a}-\sqrt{b})^2$. 
  Let $\gamma\in [0,1-\frac{1000\chi k}{\epsilon \sqrt{d}}]$. 
  Choose $\rho_\gamma\coloneqq \exp(-\gamma C_{d,\epsilon}/k)$.
  Then with probability at least $1-\exp(-100k)-\frac{1}{n^3}$, for every $z\in \Set{0,1}^n$, and for every pair of communities $y\in \Set{0,\pm 1}^n$. 
  We have
  \begin{equation*}
    \iprod{\bar{G}y\odot y,z}\geq \frac{(1-\gamma)\e d}{8k} \Paren{\norm{z}_1-\frac{96 \rho_\gamma k n}{1-\gamma}}\,,
  \end{equation*}   
  where $\bar{G}$ is the centered adjacency matrix.
  }
\end{theorem}
Similarly, for establishing the feasibility of the program constraints in \cref{sec:Optimal-robust-clustering}, We
 need the following corollary, which bounds the error of majority voting when a small fraction of the vertices are removed from the graph.
\begin{corollary}\torestate{
  \label{cor:inner-product-lb-k-clustering}
  Let $\gamma\in [0,1-\frac{1000\chi k}{\epsilon \sqrt{d}}]$. 
  Choose $\rho_\gamma\coloneqq \exp(-\gamma C)$.
  Then with probability at least $1-\exp(-100k)-\frac{2}{n^3}$, for every $z\in \Set{0,1}^n$, for every $s \in \Set{0,1}^n$ such that $\norm{s}_1 \geq (1-\exp(-2C_{d, \e}))n$, and for every pair of communities $y\in \Set{0,\pm 1}^n$.
  We have
  \begin{equation*}
    \iprod{\bar{G} \odot (s s^{\top}) y\odot y,z} \geq \frac{(1-\gamma)\e d}{16k} \Paren{\norm{z}_1 -\frac{640 \rho_\gamma k n}{1-\gamma}} \,,
  \end{equation*}
  where $\bar{G}$ is the centered adjacency matrix.
}
\end{corollary}

\section{Algorithmic framework and results}\label{sec:Main-theorem-algorithm}

In this section, we give a sketch of our robust optimal recovery algorithm (see \cref{algo:final-boost}) with its key building blocks.
\cref{algo:final-boost} contains two main steps: \emph{initialization} and \emph{boosting based on pairwise majority voting}.
The goal of initialization is to get a $1/\poly(k)$ approximation such that, in each community, we have error rate at most $0.001$.
Given the $1/\poly(k)$-initial estimation, we can apply pairwise majority voting to boost the accuracy to the minimax rate.

\subsection{Robust bisection algorithm}

The key component of our initialization algorithm is the following robust bisectioning algorithm that finds bisections in the graph with bisection error $\exp\left(-\bigl(1 - o(1)\bigr)\,\frac{C_{d,\varepsilon}}{k^2}\right) + \poly(k) \eta$.
\begin{algorithmbox}[Robust bisection algorithm]\label{algo:full-robust-bisectioning}
	\mbox{}\\
	\textbf{Input:} A graph $G$ sampled from the $k$-stochastic block model $\SBM_n(d,\e,k)$ with $\eta n$ corrupted nodes.
	\begin{enumerate}
    \item \textbf{Graph splitting:} We let $G_1$ be the graph obtained by subsampling each edge in $G$ independently with probability $0.99$ and let $G_2\coloneqq G\setminus G_1$.
    \item \textbf{Rough initialization:} Run a rough $k$-clustering algorithm on graph $G_1$ to obtain rough initialization $Z_{\text{rough}}$ with error rate $0.001 + 10^4 \eta$.
    \item \textbf{Identifying well recovered blocks:} Use graph $G_2$ and $Z_{\text{rough}}$ to identify $k/2$ clusters in which $0.99$-fraction of the nodes belongs to the same community. Construct rough bisection $x_{\text{rough}}$ by putting the $k/2$ identified clusters on one side of the bisection.
    \item \textbf{Bisection boosting:} Use majority voting on bisections with $x_{\text{rough}}$ as the initialization to obtain a bisection $\hat{x}$ with error rate $\exp\left(-\bigl(1 - o(1)\bigr)\,\frac{C_{d,\varepsilon}}{k^2}\right) + \poly(k) \eta$.
	\end{enumerate}
    \textbf{Output:} $\hat{x}$.
\end{algorithmbox}

We will now briefly explain and state the results for the subroutines \emph{rough initialization}, \emph{identifying well recovered blocks} and \emph{bisection boosting}, then state the initialization guarantee that is achieved by recursively applying the bisection algorithm.

\paragraph{Rough initialization}
Previous work based on semidefinite programming and spectral algorithms can get a rough initialization with error $0.001 + 10^4\,k\,\eta$ (Lemma~7.3 of~\cite{Liu2022minimax}). In our work, we show that a rough initialization procedure based on Sum-of-Squares (SoS) can get an improved guarantee (see \cref{sec:Initialization} for details).

\begin{theorem}[Robust rough initialization]\torestate{\label{thm:Initialization}
Let \((G^{\circ}, Z^{\circ}) \sim \SBM_n(d,\varepsilon,k)\) be generated from the \(k\)-stochastic block model and $G$ be generated by adversarially corrupting $\eta$-fraction of the nodes in $G^{\circ}$.
Assume $k\leq n^{0.001}$, $\eta \leq \frac{1}{\poly(k)}$, $d = o(n)$, and $\e^2 d \geq K k^2$ for some sufficiently large constant \(K\).
There exists a polynomial-time algorithm that, given observation of $G$, outputs an estimator \(\hat{Z} \in \set{0, 1}^{n\times k}\) such that, with probability $1 - \exp(-\Omega(k)) - \tfrac{1}{\poly(n)}$,
\[
\error_k(\hat{Z},Z^{\circ}) \le 0.001 + 10^4 \,\eta.
\]
}
\end{theorem}

\paragraph{Identifying well recovered clusters}
Once we have a rough initialization with misclassification rate at most $0.001$, at least $k/2$ clusters must be ``well recovered'', in the sense that at least $0.99$ of the nodes in these well recovered clusters belong to the same community.
We show that there exists an algorithm that can robustly identify these well recovered clusters (see \cref{sec:IdentifyingRecoveredBlocks} for details).

\begin{theorem}[Robust identification of well recovered clusters]\torestate{\label{thm:identify-recovered-blocks}
Let \((G^{\circ}, Z^{\circ}) \sim \SBM_n(d,\varepsilon,k)\) be generated from the \(k\)-stochastic block model and $G$ be generated by adversarially corrupting $\eta$-fraction of the nodes in $G^{\circ}$.
Assume $k\leq n^{0.001}$, $\eta \leq \frac{1}{\poly(k)}$, $d = o(n)$, and $\e^2 d \geq K k^2$ for some sufficiently large constant \(K\).
For any set \(S \subset V\) of \(n/k\) vertices, there exists a polynomial-time algorithm that, with probability $1 - \exp(-\Omega(k)) - \tfrac{1}{\poly(n)}$, outputs
\begin{itemize}
    \item YES if at least \(0.99\) of the vertices in \(S\) belong to the same community;
    \item NO if no more than \(0.98\) of the vertices in \(S\) belong to any single community.
\end{itemize}
}
\end{theorem}

\paragraph{Bisection boosting}
Given $k/2$ well recovered clusters, we can construct a rough bisection $x_{\text{rough}}$ by putting them on one side of the bisection.
We show that, given this rough bisection with error rate $0.001$, we can use majority voting to boost the bisection accuracy to obtain a bisection with optimal bisection error (see \cref{sec:robust-bisection-boosting} for details).

\begin{theorem}[Robust bisection boosting]\torestate{\label{thm:robust-bisectioning}
Let \((G^{\circ}, Z^{\circ}) \sim \SBM_n(d,\varepsilon,k)\) be generated from the \(k\)-stochastic block model and $G$ be generated by adversarially corrupting $\eta$-fraction of the nodes in $G^{\circ}$.
Assume $k\leq n^{0.001}$, $\eta \leq \frac{1}{\poly(k)}$, $d = o(n)$, and $\e^2 d \geq K k^2$ for some sufficiently large constant \(K\).
Let \(S_1,S_2,\ldots,S_{k/2} \subset V\) be disjoint subsets of size \(n/k\) such that in every \(S_i\) at least \(0.99\)-fraction of the nodes belong to the same community.
Let \(x^{\circ} \in \{\pm1\}^n\) be the ground-truth community bisection with the underlying communities of \(S_1,S_2,\ldots,S_{k/2}\) on the same side.
There exists a polynomial-time algorithm that, given observation of $G$ and \(S_1,S_2,\ldots,S_{k/2}\), outputs \(\hat{x} \in \{\pm1\}^n\) such that, with probability $1 - \exp(-\Omega(k)) - \tfrac{1}{\poly(n)}$,
\[
\frac{1}{n}\|\hat{x} - x^{\circ}\|^2 
\le 
\exp\left(-\bigl(1 - o(1)\bigr)\frac{\tilde{C}}{8}\right) 
+ \poly(k) \eta \,,
\]
where $\tilde{C} = \Paren{\sqrt{a^{\frac{2}{k}} b^{1-\frac{2}{k}}}-\sqrt{b}}^2$ is the bisection SNR.
}
\end{theorem}

\paragraph{Robust bisection algorithm}
As a corollary, by combining \cref{thm:Initialization}, \cref{thm:identify-recovered-blocks} and \cref{thm:robust-bisectioning}, we have a robust bisection algorithm for the $k$-stochastic block model with optimal bisection error rate (see \cref{sec:robust-bisection-algo} for details).

\begin{theorem}[Robust bisection algorithm]\torestate{\label{thm:full-robust-bisectioning}
Let \((G^{\circ}, Z^{\circ}) \sim \SBM_n(d,\varepsilon,k)\) be generated from the \(k\)-stochastic block model and $G$ be generated by adversarially corrupting $\eta$-fraction of the nodes in $G^{\circ}$.
Assume $k\leq n^{0.001}$, $\eta \leq \frac{1}{\poly(k)}$, $d = o(n)$, and $\e^2 d \geq K k^2$ for some sufficiently large constant \(K\).
There exists a polynomial-time algorithm that, given observation of $G$, outputs \(\hat{x} \in \{\pm1\}^n\) such that, with probability $1 - \exp(-\Omega(k)) - \tfrac{1}{\poly(n)}$,
\[
\frac{1}{n}\|\hat{x} - x^{\circ}\|^2 
\,\le\, 
\exp\left(-\bigl(1 - o(1)\bigr) \frac{\tilde{C}}{8}\right) + O(k \eta).
\]
where $\tilde{C} = \Paren{\sqrt{a^{\frac{2}{k}} b^{1-\frac{2}{k}}}-\sqrt{b}}^2$ is the bisection SNR and \(x^{\circ} \in \{\pm1\}^n\) is a true community bisection of \(G^{\circ}\).
}
\end{theorem}

\subsection{Robust initialization for symmetry breaking}

By recursively applying \cref{thm:full-robust-bisectioning}, we can cluster the vertices into $k$ communities with error rate $\exp\left(-\bigl(1 - o(1)\bigr)\,\frac{C_{d,\varepsilon}}{k^2}\right) + \poly(k) \eta$  (see \cref{sec:sym-break-robust-clustering} for details).
When $C_{d,\varepsilon} \geq \Omega(k^2 \log k)$, we get an error rate of $1/\poly(k) + \poly(k) \eta$.

\begin{theorem}[Initialization for symmetry breaking]\torestate{\label{thm:boost_to_sym_break}
Let \((G^{\circ}, Z^{\circ}) \sim \SBM_n(d,\varepsilon,k)\) be generated from the \(k\)-stochastic block model and $G$ be generated by adversarially corrupting $\eta$-fraction of the nodes in $G^{\circ}$.
Assume $k\leq n^{0.001}$, $\eta \leq \frac{1}{\poly(k)}$, $d = o(n)$, and $\e^2 d \geq K k^2$ for some sufficiently large constant \(K\).
There exists a polynomial-time algorithm that, given observation of $G$, outputs an estimator \(\hat{Z} \in \set{0, 1}^{n\times k}\) such that, with probability $1 - \exp(-\Omega(k)) - \tfrac{1}{\poly(n)}$,
\[
\error_k(\hat{Z}, Z^{\circ})
\,\le\, 
\exp\left(-\bigl(1 - o(1)\bigr)\,\frac{C_{d,\varepsilon}}{k^2}\right) + \poly(k) \eta.
\]
}
\end{theorem}

\subsection{Boosting via pairwise majority voting}

Given the initialization with error $1/\poly(k) + \poly(k) \eta$, we can apply pairwise majority voting to boost the accuracy of the estimator to the minimax rate.
The formal main theorem of our robust clustering algorithm for node-corrupted k-stochastic block model is as follows.
The algorithm and proofs of \cref{thm:formal_main_result} are presented in \cref{sec:Optimal-robust-clustering}.

\begin{theorem}[Robust recovery in k-SBM with optimal rate]\torestate{\label{thm:formal_main_result}
Let \((G^{\circ}, Z^{\circ}) \sim \SBM_n(d,\varepsilon,k)\) be generated from the \(k\)-stochastic block model and $G$ be generated by adversarially corrupting $\eta$-fraction of the nodes in $G^{\circ}$.
Assume $k\leq n^{0.001}$, $\eta \leq \frac{1}{\poly(k)}$, $d = o(n)$, and $\e^2 d \geq K k^2 \log k$ for some sufficiently large constant \(K\).
There exists a polynomial-time algorithm (see \cref{algo:final-boost}) that, given observation of $G$, outputs an estimator \(\hat{Z} \in \set{0, 1}^{n\times k}\) such that, with probability $1 - \exp(-\Omega(k)) - \tfrac{1}{\poly(n)}$,
\[
\error_k(\hat{Z}, Z^{\circ})
\le \exp\Bigl(-\bigl(1-o(1)\bigr)\frac{C_{d,\varepsilon}}{k}\Bigr) + \poly(k) \cdot \eta \,.
\]
}
\end{theorem}

\section{Robust rough initialization}\label{sec:Initialization}

In this section, we prove \cref{thm:Initialization} via a natural SoS-based initialization procedure.
Our analysis is of independent interest as it both simplifies and strictly improves upon the initialization guarantees of~\cite{Liu2022minimax}. 
In particular, we distinguish two regimes based on the adversarial corruption fraction \(\eta\):
\begin{itemize}
    \item If \(\eta \sim \exp\bigl(-C_{d,\e}/k\bigr)\) (the critical, minimax regime), 
    our goal is to obtain an initial recovery error bounded by a small constant (e.g. \(0.001\)). 
    \item If \(\eta\) is significantly larger than \(\exp\bigl(-C_{d,\e}/k\bigr)\), 
    the recovery error is dominated by the corruption level, 
    meaning that one can only hope to recover the community structure up to the inherent corruption threshold.
\end{itemize}
In the former case, both our method and that of~\cite{Liu2022minimax} achieve the desired $0.001$ recovery error,
which is sufficient to bootstrap our subsequent boosting step.
In the latter regime, however, our SoS formulation offers a significant improvement by tolerating a constant corruption fraction \(\eta\) independently of \(k\). 
By contrast, the guarantee in \cite{Liu2022minimax} requires that \(\eta \lesssim 1/k\) in order to maintain the misclassification error 
(on the order of \(k\eta\)) bounded, rendering their bound vacuous as \(k\) diverges.

\restatetheorem{thm:Initialization}

We consider a constant-degree sum-of-squares relaxation for the following program:
\begin{equation}
\cA_{\text{init}}(\xi, G, Z; G_{\text{input}}, \eta) \seteq 
\left\{
\begin{aligned}
& \xi \odot \xi = \xi, \quad \iprod{\xi, \Ind} \ge \Bigl(1 - \eta - \exp\Bigl(-\frac{2\e^2 d}{k}\Bigr)\Bigr) n, \\[1ex]
& Z \odot Z = Z, \quad \sum_{a \in [k]} Z(i,a) = 1, \quad \sum_{i \in [n]} Z(i,a) = \frac{n}{k}, \\[1ex]
& \Normop{\,G - \frac{d}{n}\,11^\top - \frac{\e d}{n}\,(ZZ^\top-\frac1k\,J)\,} \le \chi \sqrt{d}, \\[1ex]
& G \odot \xi\xi^\top = G_{\text{input}} \odot \xi\xi^\top.
\end{aligned}
\right\}
\label{poly_sys:initialization}
\end{equation}
For simplicity of notation, we let
\begin{align*}
    \cA_{\text{bool}}(\xi,Z) &\seteq \Set{\xi\odot \xi=\xi, Z\odot Z=Z, \sum_{a\in [k]}Z(i,a)=1, \sum_{i\in [n]} Z(i,a)=n/k}\\
    \cA_{\text{spec}}(G,Z) &\seteq \Set{ \Normop{\,G - \frac{d}{n}\,11^\top - \frac{\e d}{n}\,(ZZ^\top-\frac1k\,J)\,}\leq \chi\sqrt{d}}\\
    \cA_{\text{corr}}(\xi,G; G_{\text{input}},\eta) &\seteq \Set{G \odot \xi\xi^\top = G_{\text{input}}\odot \xi\xi^\top, \iprod{\xi, \Ind} \geq \Paren{1-\eta-\exp(-2\e^2 d/k)}n}
\end{align*}
Then we have
\begin{equation*}
    \cA_{\text{init}}(\xi,G,Z; G_{\text{input}},\eta)=\cA_{\text{bool}}(\xi,Z) \cup \cA_{\text{spec}}(G,Z) \cup \cA_{\text{corr}}(\xi,G; G_{\text{input}},\eta) \,.
\end{equation*}

We first show that the program is feasible with high probability.

\begin{lemma}\label{lem:initialization-feasible}
Let $G_{\mathrm{input}}$ be an $\eta$-corrupted stochastic block model (SBM) graph on $n$ nodes partitioned into $k$ communities. Suppose that
\[
\e^2 d \ge K\, k^2,
\]
for a sufficiently large constant $K>0$. 
Then, with high probability, the feasibility conditions in \cref{poly_sys:initialization}
are satisfied by the choice \(\xi=\xi_0\odot \one_S\) (where \(S\) is as in Cor.~\ref{cor:spectral-remove-cor}),
\(Z=\Znull\), and a suitable \(G\) agreeing with \(G_{\mathrm{input}}\) on the subgraph induced by the set of vertices indicated by $\xi$.
\end{lemma}

\begin{proof}
    By Corollary~\ref{cor:spectral-remove-cor}, there exists a subset $S \subseteq [n]$ with
    $|S| \ge \bigl(1-\exp(-2C_{d,\e})\bigr)n$ such that
    \[
      \Normop{\Bigl(\Gnull - \frac{d}{n}J - \frac{\e d}{n}\,\bigl({\Znull}{\Znull}^\top-\tfrac{1}{k}J\bigr)\Bigr)\odot \one_S \one_S^\top} \le \chi\sqrt{d} \enspace,
    \]
    where $\chi$ is a universal constant. 
    Let $\xi$ be the indicator of the uncorrupted nodes intersected with $S$, i.e.,
    $\xi := \xi_0 \odot \one_S$, and take $\Znull\in\{0,1\}^{n\times k}$ to be the ground-truth membership matrix.
    Define $G$ to agree with $G_{\mathrm{input}}$ on the block indexed by $\supp(\xi)$ and to equal the model mean
    $\frac{d}{n}J + \frac{\e d}{n}\bigl(\Znull(\Znull)^\top - \frac{1}{k}J\bigr)$ outside this block.
    Then the centered matrix
    $G - \frac{d}{n}J - \frac{\e d}{n}\bigl(\Znull(\Znull)^\top - \frac{1}{k}J\bigr)$ vanishes outside $\supp(\xi)$ and,
    on $\supp(\xi)$, its operator norm is bounded by $\chi\sqrt{d}$ by
    Corollary~\ref{cor:spectral-remove-cor}. 
    Hence the spectral constraint in
    \cref{poly_sys:initialization} holds. 
    Moreover,
    $\iprod{\xi,\Ind} \ge \bigl(1-\eta - \exp(-2C_{d,\e})\bigr)n$, which satisfies
    $\cA_{\text{corr}}$. Therefore, with these choices of $\xi$, $Z=\Znull$, and $G$, all feasibility
    conditions in \cref{poly_sys:initialization} are satisfied with high probability.
\end{proof}

Now, we give the SoS proof for the recovery guarantees.
\begin{lemma}\label{lem:intialization-identifiability}
Let $G_{\text{input}}$ be an $\eta$-corrupted SBM graph with $n$ nodes and $k$ communities, and assume that
\[
\e^2 d \geq K k^2,
\]
for some sufficiently large constant $K$. Then, we have
\begin{align*}
  \cA_{\text{init}}(\xi^{(1)},G^{(1)},Z^{(1)}; G_{\text{input}},\eta),\quad \cA_{\text{init}}(\xi^{(2)},G^{(2)},Z^{(2)}; G_{\text{input}},\eta)  & \\
  \sststile{8}{\xi^{(1)},\,\xi^{(2)},\,G^{(1)},\,G^{(2)},\,Z^{(1)},\,Z^{(2)}}\,\Normf{Z^{(1)}(Z^{(1)})^\top-Z^{(2)}(Z^{(2)})^\top}^2 \leq (0.001+10\eta)n^2/k &\,. 
\end{align*}
\end{lemma}
\begin{proof}
    Since the constraints $G^{(1)} \odot \xi^{(1)}(\xi^{(1)})^\top = G_{\text{input}}\odot \xi^{(1)}(\xi^{(1)})^\top$ and $G^{(2)} \odot \xi^{(2)}(\xi^{(2)})^\top = G_{\text{input}}\odot \xi^{(2)}(\xi^{(2)})^\top$ hold, let $\xi=\xi^{(1)}\odot \xi^{(2)}$. We then have a degree-$8$ sum-of-squares proof that
    \begin{equation*}
        G^{(1)} \odot \xi\xi^\top = G_{\text{input}}\odot \xi\xi^\top \quad \text{and} \quad G^{(2)} \odot \xi\xi^\top = G_{\text{input}}\odot \xi\xi^\top\,.
    \end{equation*}
    Therefore, we have $G^{(1)}\odot \xi\xi^\top = G^{(2)}\odot \xi\xi^\top$.
    By the min–max characterization of the spectral norm,
    \begin{align*}
        \cA_{\text{bool}}(\xi^{(1)},Z^{(1)}) \cup \cA_{\text{spec}}(G^{(1)},Z^{(1)}) & \sststile{8}{G^{(1)},Z^{(1)}} \Normop{\Paren{G^{(1)} - \frac{d}{n} J - \frac{\e d}{n}(Z^{(1)}(Z^{(1)})^\top-\frac1k\,J )\,}\odot \xi\xi^\top}&\leq \chi\sqrt{d} \\
        \cA_{\text{bool}}(\xi^{(2)},Z^{(2)}) \cup \cA_{\text{spec}}(G^{(2)},Z^{(2)}) & \sststile{8}{G^{(2)},Z^{(2)}} \Normop{\Paren{G^{(2)} -\frac{d}{n} J - \frac{\e d}{n} ( Z^{(2)}(Z^{(2)})^\top-\frac1k\,J )\,}\odot \xi\xi^\top}&\leq \chi\sqrt{d}
    \end{align*}
    By the triangle inequality, this yields a degree-8 sum-of-squares certificate that
    \begin{equation*}
        \Normop{\frac{\e d}{n}\Paren{Z^{(1)}(Z^{(1)})^\top-Z^{(2)}(Z^{(2)})^\top}\odot \xi\xi^\top} \leq 2\chi\sqrt{d}\,.
    \end{equation*}
    Since $Z^{(1)}$ and $Z^{(2)}$ are community membership matrices, each Gram matrix
    $Z^{(t)}(Z^{(t)})^\top$ has rank at most $k$, so
    $\Paren{Z^{(1)}(Z^{(1)})^\top-Z^{(2)}(Z^{(2)})^\top}\odot \xi\xi^\top$ has rank at most $2k$\footnote{Indeed, letting $D=\mathrm{diag}(\xi)$ and
$\Delta = Z^{(1)}(Z^{(1)})^\top - Z^{(2)}(Z^{(2)})^\top$, we have
$\Delta\odot \xi\xi^\top = D\Delta D$, hence
$\mathrm{rank}(\Delta\odot \xi\xi^\top)\le \mathrm{rank}(\Delta)\le 2k$.}.
    Thus, we obtain a degree-$8$ sum-of-squares proof that
    \begin{equation*}
        \Normf{\Paren{Z^{(1)}(Z^{(1)})^\top-Z^{(2)}(Z^{(2)})^\top}\odot \xi\xi^\top}^2 \leq \frac{8\chi^2 k n^2}{\e^2 d}\,.
    \end{equation*}
    Finally, since $\norm{\xi}_1\geq (1-2\eta-2\exp(-2\e^2d/k)) n$, we have 
    \begin{align*}
        \cA_{\text{init}}(\xi^{(1)},G^{(1)},Z^{(1)}; G_{\text{input}},\eta), \cA_{\text{init}}(\xi^{(2)},G^{(2)},Z^{(2)}; G_{\text{input}},\eta) \\
        \sststile{8}{\xi^{(1)},\xi^{(2)},G^{(1)},G^{(2)},Z^{(1)},Z^{(2)}} \Normf{Z^{(1)}(Z^{(1)})^\top-Z^{(2)}(Z^{(2)})^\top}^2 \leq \frac{8\chi^2 k n^2}{\e^2 d}+ 10\eta n^2/k\,.
    \end{align*} 
\end{proof}
Finally, we conclude with the rounding step.
\begin{proof}[Proof of \cref{thm:Initialization}]
Let us write $M^{\circ} \coloneqq \Znull(\Znull)^\top$ for the ground-truth matrix. 
By \cref{lem:initialization-feasible}, the ground-truth community matrix \(\Znull\) is feasible for \cref{poly_sys:initialization}. 
Hence, any pseudo-distribution satisfying these constraints obeys
\[
\|ZZ^\top - M^{\circ}\|_F^2 \le 10\eta\,n^2/k + \frac{8\chi^2 k\, n^2}{\varepsilon^2\, d}\,.
\]
Standard rounding then yields a community matrix \(M = \tE ZZ^T \in [0,1]^{n\times n}\) such that
\begin{equation}\label{eq:frob-up-bound}
\|M - M^{\circ}\|_F^2 \le \bigl(\frac{10\,\eta}{k} + \frac{8\,\chi^2 \,k}{\varepsilon^2\, d}\bigr) \, n^2.
\end{equation}
Having established the Frobenius-norm bound, the rest of the proof mirrors Lemma~7.3 of \cite{Liu2022minimax}; we present it below for completeness. Since the rows of $M$ lie in $\R^n$, we can apply a Euclidean $k$-means algorithm with a $\gamma=7$ approximation guarantee~\cite{grandoni2021euclideankmeans}.
Evaluating the $k$-means objective at the ground-truth partition $(S_1,\ldots,S_k)$ shows that
\[
k\text{-means}_{\text{opt}}(M)\ \le\ \|M - M^{\circ}\|_F^2 \ \le\ 
\Bigl(\frac{10\,\eta}{k} + \frac{8\,\chi^2 \,k}{\varepsilon^2\, d}\Bigr) \, n^2,
\]
where the last inequality is \cref{eq:frob-up-bound}.

Now fix any alternative clustering $S'_1,\ldots,S'_k$ of the rows, 
and consider the same clustering evaluated on $M^{\circ}$. For one cluster, say $S'_1$, let
\[
s_b = |S'_1 \cap S_b| \quad \text{ for } b \in [k].
\]
As in Lemma~7.3 of~\cite{Liu2022minimax}, the $k$-means objective contributed by $S'_1$ on the rows of $M^{\circ}$ is at least
\begin{equation}\label{eq:k_means_identity}
\Bigl(s_1 + \cdots + s_k - \max\{s_1,\ldots,s_k\}\Bigr) \cdot \frac{n}{4k}.
\end{equation}
Summing \eqref{eq:k_means_identity} over all clusters $S'_a$ lower-bounds the $k$-means objective on $M^{\circ}$ for the partition $S'_1,\ldots,S'_k$. 
Let the computed clustering be $S'_1, \ldots ,S'_k$. 
Define
\[
\delta = \frac{1}{n} \min_{\substack{\pi:[k]\to[k]\\ \pi\ \text{invertible}}}
\sum_{a=1}^k |S'_a \setminus S_{\pi(a)}|,
\qquad \delta' = \frac{1}{n} \min_{f:[k]\to[k]}
\sum_{a=1}^k |S'_a \setminus S_{f(a)}|.
\]
Here $\delta$ is exactly the misclassification error of $S'_1,\ldots,S'_k$ relative to the ground truth.  
As in~\cite{Liu2022minimax}, one has $\delta \le 2\delta'$. 
Moreover, using \eqref{eq:k_means_identity} and summing over clusters shows that the $k$-means objective of $S'_1,\ldots,S'_k$ on $M^{\circ}$ is at least $\tfrac{n^2}{4k}\,\delta'$. 
Combining this lower bound with the $\gamma$-approximation guarantee on $M$ and \cref{eq:frob-up-bound} 
(exactly as in Lemma~7.3 of~\cite{Liu2022minimax}) yields
\[\delta \le 10^4\,k\left(\frac{\eta}{k} + \frac{\chi^2\,k}{\varepsilon^2\,d}\right) \enspace,\]
which is the desired bound.
\end{proof}
\section{Robustly identifying well-recovered clusters}
\label{sec:IdentifyingRecoveredBlocks}

In this section, we show that given a set of $n/k$ vertices in the graph sampled from $\SBM_n(d,\e,k)$, we can decide between
\begin{itemize}
    \item Case I: at least 0.99 of the vertices in the set belong to the same community;
    \item Case II: no more than 0.98 of the vertices in the set belong to any single community.
\end{itemize}

\paragraph{Restriction to the given set.}
Let $S\subseteq[n]$ denote the given set of $\abs{S}=n/k$ vertices, and let $\Ind=\mathbf{1}_S\in\{0,1\}^n$ be its indicator.
We define the matrices restricted to $S$ as
\[
G_s \coloneqq G \odot (\Ind_S \Ind_S^\top)
\qquad\text{and}\qquad
J_s \coloneqq \Ind_S \Ind_S^\top.
\]
Throughout this section, every appearance of $G$ and $J$ inside operator norms or inner products is to be understood as $G_s$ and $J_s$ respectively. We also view all vectors (e.g., $z,\tau,\zeta$) as elements of $\R^n$ supported on $S$.

\restatetheorem{thm:identify-recovered-blocks}

We will show that the feasibility of the degree-$2$ SoS relaxation of the following program satisfies \cref{thm:identify-recovered-blocks}.
\begin{gather}
    \cA_{\text{verify}}(z; G,\eta) \seteq 
    \left. 
        \begin{cases}
        z \odot z = z,\, \iprod{z, \Ind} \geq \Paren{\frac{0.99}{k}-\eta-\rho}n \\
        \Normop{\Paren{G_s-\frac{\alpha}{n} J_s} \odot z z^{\top}} \leq \chi \sqrt{\alpha /k} \\
        \Iprod{\Paren{G_s-\frac{d}{n} J_s} \odot z z^{\top}, \Ind \Ind^{\top}} \geq \frac{0.97 (k-1) \e d n}{k^3}
        \end{cases}
    \right\} \,, \label{poly_sys:verification}
\end{gather}
where $\alpha = d +(1-\frac{1}{k})\e d$, $\eta$ is the corruption rate, $\chi$ is the universal constant defined in the spectral result in \cref{thm:spectral-remove}, and $\rho = \exp\Paren{-2C_{d,\e}}$ is the fraction of high‑degree vertices that need to be trimmed to obtain the spectral norm bound (see \cref{cor:spectral-remove-cor}).

\begin{lemma}
\label{lem:verfication_feasibility}
    Given a set of $n/k$ vertices in the graph sampled from $\SBM_n(d,\e,k)$ where at least 0.99 fraction of the vertices in the set belongs to the same community (Case I), when $\eta \leq C_{\eta}$ for some constant $C_{\eta}$ that is small enough, Program \cref{poly_sys:verification} is feasible with probability at least $1-n^{-\Omega(1)}$.
\end{lemma}

\begin{proof}
    Let $\tau \in \set{0, 1}^{n}$ (supported on $S$) be the indicator vector for the largest set of uncorrupted and degree-bounded nodes (according to \cref{thm:spectral-remove}) that are in the same community. We will show that, with probability $1-n^{-\Omega(1)}$, $\tau$ is a feasible solution to the program \cref{poly_sys:verification}.
    
    By assumption, $\Norm{\tau}_1 \geq (\frac{0.99}{k} - \eta - \rho)n$. Therefore, $\tau$ satisfies the integrality and norm constraints of program \cref{poly_sys:verification}.
    
    By \cref{thm:spectral-remove}, with probability at least $1-\frac{k^2}{n^2}$, we have $\Normop{\Paren{G_s-\frac{\alpha}{n} J_s} \odot \tau \tau^{\top}} \leq \chi \sqrt{\frac{\alpha}{k}}$. Therefore, $\tau$ satisfies the spectral constraint of program \cref{poly_sys:verification}.

    Now, consider $\Iprod{G_s-\frac{d}{n} J_s, \tau \tau^{\top}}$. Notice that $\Iprod{G_s, \tau \tau^{\top}}$ is equal to $2$ times the number of edges in the induced subgraph $G_s(\tau)$ where each edge is sampled independently with probability $\frac{\alpha}{n}$. Therefore, by Chernoff bound, it follows that, with probability at least $1-n^{-100}$,
    \begin{equation*}
        \Iprod{G_s, \tau \tau^{\top}}
        \geq \frac{\alpha}{n} \Norm{\tau}_1^2
        - O \Paren{\sqrt{\frac{\alpha \log(n)}{n} \Norm{\tau}_1^2}} \,.
    \end{equation*}
    Hence,
    \begin{equation*}
        \Iprod{G_s-\frac{d}{n} J_s, \tau \tau^{\top}}
        \geq \frac{\alpha-d}{n} \Norm{\tau}_1^2
        - O \Paren{\sqrt{\frac{\alpha \log(n)}{n} \Norm{\tau}_1^2}} \,.
    \end{equation*}
    Since $\alpha = d +(1-\frac{1}{k})\e d$, $\alpha \leq 2d$ and $\Norm{\tau}_1 \geq (\frac{0.99}{k} - \eta - \rho)n$, it follows that
    \begin{align*}
        \Iprod{G_s-\frac{d}{n} J_s, \tau \tau^{\top}}
        & \geq \frac{0.98 (k-1) \e d n}{k^3} - O \Paren{\sqrt{d \log(n)}} \\
        & \geq \frac{0.975 (k-1) \e d n}{k^3} \,.
    \end{align*}
    Thus, $\tau$ satisfies the sum constraint of program \cref{poly_sys:verification}, which finishes the proof.
\end{proof}

\begin{lemma}
\label{lem:verfication_refucation}
    Given a set of $n/k$ vertices in the graph sampled from $\SBM_n(d,\e,k)$ where at most a 0.98 fraction of the vertices in the set belongs to the same community (Case II), when $\eta \leq C_{\eta}$ for some constant $C_{\eta}$ that is small enough, Program \cref{poly_sys:verification} is not feasible with probability at least $1-n^{-\Omega(1)}$.
\end{lemma}

\begin{proof}
    We will show that there is an SoS proof of
    \begin{equation*}
        \cA_{\text{verify}}(z; G,\eta)
        \sststile{2}{z} \Iprod{\Paren{G_s-\frac{d}{n} J_s} \odot z z^{\top}, \Ind \Ind^{\top}} \leq \frac{0.963 (k-1) \e d n}{k^3} \,,
    \end{equation*}
    which is a refutation to constraint $\Iprod{\Paren{G_s-\frac{d}{n} J_s} \odot z z^{\top}, \Ind \Ind^{\top}} \geq \frac{0.97 (k-1) \e d n}{k^3}$ in $\cA_{\text{verify}}(z; G,\eta)$.

    Let us denote the set of uncorrupted and degree-bounded (according to \cref{cor:spectral-remove-cor}) nodes by $\zeta \in \set{0, 1}^{n}$ (supported on $S$).
    By \cref{cor:spectral-remove-cor}, it follows that $\Norm{\zeta}_1 \geq (1-\eta-\rho)\frac{n}{k}$ and $\Normop{\Paren{(G_0)_s - \frac{d}{n}J_s} \odot \zeta \zeta^{\top}} \leq \chi \sqrt{d/k}$, where $(G_0)_s \coloneqq G_0 \odot (\Ind\Ind^\top)$.
    Consider the largest set of vertices in $\zeta$ that are from the same community, denoted by $\tau \in \set{0, 1}^{n}$ (supported on $S$).
    By our assumption, it follows that $\Norm{\tau}_1 \leq 0.98 \frac{n}{k}$.

    We will decompose $\Iprod{\Paren{G_s-\frac{d}{n} J_s} \odot z z^{\top}, \Ind \Ind^{\top}}$ into the following three terms and bound them separately
    \begin{align}
    \label{eq:verfication_refutation_1}
    \begin{split}
        \Iprod{\Paren{G_s-\frac{d}{n} J_s} \odot z z^{\top}, \Ind \Ind^{\top}}
        = & \Iprod{\Paren{G_s-\frac{d}{n} J_s} \odot z z^{\top}, \tau \tau^{\top} + (\zeta - \tau) (\zeta - \tau)^{\top}}\\
        + & \Iprod{\Paren{G_s-\frac{d}{n} J_s} \odot z z^{\top}, 2 (\zeta - \tau) \tau^{\top}} + \Iprod{\Paren{G_s-\frac{d}{n} J_s} \odot z z^{\top}, \Ind \Ind^{\top} - \zeta \zeta^{\top}} \,.
    \end{split}
    \end{align}

    \paragraph{Term 1} For the first term $\Iprod{\Paren{G_s-\frac{d}{n} J_s} \odot z z^{\top}, \tau \tau^{\top} + (\zeta - \tau) (\zeta - \tau)^{\top}}$, we can further decompose it into the following two terms
    \begin{align*}
        \Iprod{\Paren{G_s-\frac{d}{n} J_s} \odot z z^{\top}, \tau \tau^{\top} + (\zeta - \tau) (\zeta - \tau)^{\top}}
        = & \Iprod{\Paren{G_s-\frac{d}{n} J_s} \odot z z^{\top}, \tau \tau^{\top}} \\
        & + \Iprod{\Paren{G_s-\frac{d}{n} J_s} \odot z z^{\top}, (\zeta - \tau) (\zeta - \tau)^{\top}} \,.
    \end{align*}
    The two terms can be bounded using the same method. Consider the first term $\Iprod{\Paren{G_s-\frac{d}{n} J_s} \odot z z^{\top}, \tau \tau^{\top}}$, it follows that
    \begin{align*}
        \Iprod{\Paren{G_s-\frac{d}{n} J_s} \odot z z^{\top}, \tau \tau^{\top}}
        & = \Iprod{\Paren{G_s-\frac{\alpha}{n} J_s} \odot z z^{\top}, \tau \tau^{\top}} + \Iprod{\frac{\alpha-d}{n} J_s \odot z z^{\top}, \tau \tau^{\top}} \\
        & = \Iprod{\Paren{G_s-\frac{\alpha}{n} J_s} \odot z z^{\top}, \tau \tau^{\top}} + \frac{\alpha-d}{n} \iprod{\tau, z}^2 \,.
    \end{align*}
    Since $\cA_{\text{verify}}$ certifies $\Normop{\Paren{G_s-\frac{\alpha}{n} J_s} \odot z z^{\top}} \leq \chi \sqrt{\alpha /k}$, it follows that
    \begin{equation*}
        \cA_{\text{verify}}(z; G,\eta) \sststile{2}{z}
        \Iprod{\Paren{G_s-\frac{d}{n} J_s} \odot z z^{\top}, \tau \tau^{\top}}
        \leq \chi \sqrt{\alpha /k} \Norm{\tau}^2 + \frac{\alpha-d}{n} \iprod{\tau, z}^2 \,.
    \end{equation*}
    Using the same analysis by replacing $\tau$ by $\zeta - \tau$, we can also obtain
    \begin{equation*}
        \cA_{\text{verify}}(z; G,\eta) \sststile{2}{z}
        \Iprod{\Paren{G_s-\frac{d}{n} J_s} \odot z z^{\top}, (\zeta - \tau) (\zeta - \tau)^{\top}} 
        \leq \chi \sqrt{\alpha /k} \Norm{\zeta - \tau}^2 + \frac{\alpha-d}{n} \iprod{\zeta - \tau, z}^2 \,.
    \end{equation*}
    Combining the two inequalities, we get
    \begin{align}
        \label{eq:verfication_refutation_term1_1}
        \begin{split}
        \cA_{\text{verify}}(z; G,\eta) \sststile{2}{z}
        & \Iprod{\Paren{G_s-\frac{d}{n} J_s} \odot z z^{\top}, \tau \tau^{\top} + (\zeta - \tau) (\zeta - \tau)^{\top}} \\
        & \leq \frac{\alpha-d}{n} \Paren{\iprod{\tau, z}^2 + \iprod{\zeta - \tau, z}^2} + \chi \sqrt{\alpha /k} \Paren{\Norm{\tau}^2 + \Norm{\zeta - \tau}^2} \,.
        \end{split}
    \end{align}
    Since $\cA_{\text{verify}}(z; G,\eta) \sststile{2}{z} 0 \leq z \leq 1$, it follows that
    $\iprod{\tau, z} \leq \Norm{\tau}_1$ and $\iprod{\zeta - \tau, z} \leq \Norm{\zeta-\tau}_1$.
    Moreover, since $\Norm{\tau}_1 \leq 0.98 \frac{n}{k}$ and $\Norm{\zeta}_1 \leq \frac{n}{k}$, the value of $\Norm{\tau}_1^2 + \Norm{\zeta-\tau}_1^2$ is upper bounded by $(0.98 \frac{n}{k})^2 + (0.02 \frac{n}{k})^2 \leq 0.961\frac{n^2}{k^2}$.
    Therefore, we can obtain
    \begin{equation}
    \label{eq:verfication_refutation_term1_2}
        \cA_{\text{verify}}(z; G,\eta) \sststile{2}{z}
        \iprod{\tau, z}^2 + \iprod{\zeta - \tau, z}^2
        \leq \Norm{\tau}_1^2 + \Norm{\zeta-\tau}_1^2
        \leq 0.961\frac{n^2}{k^2} \,.
    \end{equation}
    Since $\tau$ and $\zeta$ are Boolean vectors and $\tau \subseteq \zeta$, it follows that
    \begin{equation}
    \label{eq:verfication_refutation_term1_3}
        \Norm{\tau}^2 + \Norm{\zeta - \tau}^2
        = \Norm{\zeta}^2
        \leq \frac{n}{k} \,.
    \end{equation}
    Plugging \cref{eq:verfication_refutation_term1_2} and \cref{eq:verfication_refutation_term1_3} into \cref{eq:verfication_refutation_term1_1}, we can obtain
    \begin{align*}
        \cA_{\text{verify}}(z; G,\eta) \sststile{2}{z}
        & \Iprod{\Paren{G_s-\frac{d}{n} J_s} \odot z z^{\top}, \tau \tau^{\top} + (\zeta - \tau) (\zeta - \tau)^{\top}} \\
        & \leq \frac{0.961(\alpha-d)n}{k^2}  + \frac{\chi \sqrt{\alpha} n}{\sqrt{k^3}} \,.
    \end{align*}
    Plugging in $\alpha = d +(1-\frac{1}{k})\e d$ and $\alpha \leq 2d$, we get
    \begin{equation}
    \label{eq:verfication_refutation_2}
        \cA_{\text{verify}}(z; G,\eta) \sststile{2}{z}
        \Iprod{\Bigl(G_s-\frac{d}{n} J_s\Bigr) \odot z z^{\top}, \tau \tau^{\top} + (\zeta - \tau) (\zeta - \tau)^{\top}}
        \leq \frac{0.961 (k-1) \e d n}{k^3} + \frac{2 \chi \sqrt{d} n}{\sqrt{k^3}} \,.
    \end{equation}

    \paragraph{Term 2} For the the second term $\Iprod{\Paren{G_s-\frac{d}{n} J_s} \odot z z^{\top}, 2 (\zeta - \tau) \tau^{\top}}$, since we are summing over entries in the set of uncorrupted vertices $\zeta$, it follows that
    \begin{equation*}
        \Iprod{\Paren{G_s-\frac{d}{n} J_s} \odot z z^{\top}, 2 (\zeta - \tau) \tau^{\top}}
       = 2 \Iprod{\Paren{(G_0)_s-\frac{d}{n} J_s} \odot \zeta \zeta^{\top}, (\zeta - \tau) \tau^{\top} \odot z z^{\top}} \,.
   \end{equation*}
   To simplify notation, let us denote $M_d = \Paren{(G_0)_s-\frac{d}{n} J_s} \odot \zeta \zeta^{\top}$. It follows by SoS Cauchy-Schwarz (\cref{fact:simple-sos-version-of-cauchy-schwarz}) that
   \begin{align*}
        2 \Iprod{M_d, (\zeta - \tau) \tau^{\top} \odot z z^{\top}}
        & = 2 \Iprod{M_d \cdot (\tau \odot z), (\zeta - \tau) \odot z } \\
        & \leq \frac{\Norm{M_d \cdot (\tau \odot z)}^2}{\chi \sqrt{d/k}} + \chi \sqrt{d/k} \Norm{(\zeta - \tau) \odot z}^2 \,.
   \end{align*}
   Since we chose $\zeta$ such that $\Normop{M_d} \leq \chi \sqrt{d/k}$, it follows that
    \begin{equation*}
        2 \Iprod{M_d, (\zeta - \tau) \tau^{\top} \odot z z^{\top}}
        \leq \chi \sqrt{d/k} \Paren{\Norm{(\zeta - \tau) \odot z}^2 + \Norm{\tau \odot z}^2} \,.
    \end{equation*}
    Since $\cA_{\text{verify}}(z; G,\eta) \sststile{2}{z} 0 \leq z \leq 1$, it follows that $\Norm{(\zeta - \tau) \odot z}^2 \leq \Norm{\zeta - \tau}^2$ and $\Norm{\tau \odot z}^2 \leq \Norm{\tau}^2$. Moreover, since $\zeta$ and $\tau$ are Boolean, we have $\Norm{\tau}^2 + \Norm{\zeta - \tau}^2 = \Norm{\zeta}^2 \leq \frac{n}{k}$. Hence, we can obtain
    \begin{equation}
        \label{eq:verfication_refutation_3}
        \cA_{\text{verify}}(z; G,\eta) \sststile{2}{z}
        \Iprod{G_s-\frac{d}{n} J_s, \Paren{2 (\zeta - \tau) \tau^{\top}} \odot z z^{\top}}
        \leq \frac{\chi \sqrt{d} n}{\sqrt{k^3}} \,.
    \end{equation}

    \paragraph{Term 3} For the third term $\Iprod{\Paren{G_s-\frac{d}{n} J_s} \odot z z^{\top}, \Ind \Ind^{\top} - \zeta \zeta^{\top}}$, we can further decompose it into
    \begin{equation}
        \label{eq:verfication_refutation_4}
        \Iprod{\Paren{G_s-\frac{d}{n} J_s} \odot z z^{\top}, \Ind \Ind^{\top} - \zeta \zeta^{\top}}
        = \Iprod{\Paren{G_s-\frac{\alpha}{n} J_s} \odot z z^{\top}, \Ind \Ind^{\top} - \zeta \zeta^{\top}}
        + \Iprod{\Paren{\frac{\alpha-d}{n} J_s} \odot z z^{\top}, \Ind \Ind^{\top} - \zeta \zeta^{\top}} \,.
    \end{equation}
    Since $\cA_{\text{verify}}$ certifies $\Normop{\Paren{G_s-\frac{\alpha}{n} J_s} \odot z z^{\top}} \leq \chi \sqrt{\alpha /k}$ and $\Norm{\zeta}_1 \geq (1-\eta-\rho) \frac{n}{k}$, we can apply Frobenius Cauchy–Schwarz and get
    \begin{equation*}
        \cA_{\text{verify}}(z; G,\eta) \sststile{2}{z}
        \Iprod{\Paren{G_s-\frac{\alpha}{n} J_s} \odot z z^{\top}, \Ind \Ind^{\top} - \zeta \zeta^{\top}}
        \leq \frac{2 \chi \sqrt{\eta + \rho}\, \sqrt{\alpha} n}{\sqrt{k^3}} \,.
    \end{equation*}
    Since $\alpha \leq 2d$, $\rho \ll 1$ (due to \cref{cor:spectral-remove-cor}) and the assumption that $\eta \ll 1$, it follows that
    \begin{equation}
        \label{eq:verfication_refutation_5}
        \cA_{\text{verify}}(z; G,\eta) \sststile{2}{z}
        \Iprod{\Paren{G_s-\frac{\alpha}{n} J_s} \odot z z^{\top}, \Ind \Ind^{\top} - \zeta \zeta^{\top}}
        \leq \frac{\chi \sqrt{d} n}{\sqrt{k^3}} \,.
    \end{equation}
    Since $\cA_{\text{verify}}(z; G,\eta) \sststile{2}{z} 0 \leq z \leq 1$ and  $\alpha = d +(1-\frac{1}{k})\e d$, it follows that
    \begin{equation*}
        \cA_{\text{verify}}(z; G,\eta) \sststile{2}{z}
        \Iprod{\Paren{\frac{\alpha-d}{n} J_s} \odot z z^{\top}, \Ind \Ind^{\top} - \zeta \zeta^{\top}}
        \leq \frac{2(k-1) (\eta + \rho) \e d n}{k^3} \,.
    \end{equation*}
    Since $\rho \ll 1$ (due to \cref{cor:spectral-remove-cor}) and the assumption that $\eta \ll 1$, it follows that
    \begin{equation}
        \label{eq:verfication_refutation_6}
        \cA_{\text{verify}}(z; G,\eta) \sststile{2}{z}
        \Iprod{\Paren{\frac{\alpha-d}{n} J_s} \odot z z^{\top}, \Ind \Ind^{\top} - \zeta \zeta^{\top}}
        \leq \frac{0.001 (k-1) \e d n}{k^3} \,.
    \end{equation}
    Plugging \cref{eq:verfication_refutation_5} and \cref{eq:verfication_refutation_6} into \cref{eq:verfication_refutation_4}, it follows that
    \begin{equation}
        \label{eq:verfication_refutation_7}
        \cA_{\text{verify}}(z; G,\eta) \sststile{2}{z}
        \Iprod{\Paren{G_s-\frac{d}{n} J_s} \odot z z^{\top}, \Ind \Ind^{\top} - \zeta \zeta^{\top}}
        \leq \frac{0.001 (k-1) \e d n}{k^3} + \frac{\chi \sqrt{d} n}{\sqrt{k^3}} \,.
    \end{equation}

    \paragraph{Conclusion} Plugging \cref{eq:verfication_refutation_2}, \cref{eq:verfication_refutation_3} and \cref{eq:verfication_refutation_7} into \cref{eq:verfication_refutation_1}, we can obtain
    \begin{equation*}
        \cA_{\text{verify}}(z; G,\eta) \sststile{2}{z}
        \Iprod{\Paren{G_s-\frac{d}{n} J_s} \odot z z^{\top}, \Ind \Ind^{\top}}
        \leq \frac{0.962 (k-1) \e d n}{k^3} + \frac{4 \chi \sqrt{d} n}{\sqrt{k^3}} \,.
    \end{equation*}
    Since $k^2 \ll \e^2 d$, we can obtain
    \begin{equation*}
        \frac{4 \chi \sqrt{d} n}{\sqrt{k^3}}
        \leq \frac{0.00001}{\sqrt{k}} \cdot \frac{\e d n}{k^2}
        \leq \frac{0.001 (k-1) \e d n}{k^3} \,.
    \end{equation*}
    Thus, it follows that
    \begin{equation*}
        \cA_{\text{verify}}(z; G,\eta) \sststile{2}{z}
        \Iprod{\Paren{G_s-\frac{d}{n} J_s} \odot z z^{\top}, \Ind \Ind^{\top}}
        \leq \frac{0.963 (k-1) \e d n}{k^3} \,.
    \end{equation*}
    This is a refutation to constraint $\Iprod{\Paren{G_s-\frac{d}{n} J_s} \odot z z^{\top}, \Ind \Ind^{\top}} \geq \frac{0.97 (k-1) \e d n}{k^3}$ in $\cA_{\text{verify}}(z; G,\eta)$, which finishes the proof.
\end{proof}

\begin{proof}[Proof of \cref{thm:identify-recovered-blocks}]
    By \cref{lem:verfication_feasibility} and \cref{lem:verfication_refucation}, it follows that, given a set of $n/k$ vertices in the graph sampled from $\SBM_n(d,\e,k)$, we can check feasibility of $\cA_{\text{verify}}(z; G,\eta)$ (interpreted with $G_s,J_s$ as above) in time $n^{O(1)}$ to decide whether the set of vertices is in Case I or Case II:
    \begin{itemize}
        \item Case I (at least $0.99$ of the vertices in the set belong to the same community): feasibility follows by \cref{lem:verfication_feasibility}. Since the SoS SDP has small bit complexity and there exists a Boolean solution, by arguments analogous to those in \cite{raghavendra2017bit}, the ellipsoid method will be able to find a feasible solution in time $n^{O(1)}$.
        \item Case II (no more than $0.98$ of the vertices in the set belong to any single community): 
        By \cref{lem:verfication_refucation}, the SoS program is infeasible; hence no feasible solution exists (and, in particular, none can be found in time $n^{O(1)}$).
    \end{itemize}
    The success probability is $1-n^{-\Omega(1)}$.
\end{proof}
\section{Robust boosting algorithm for bisection}
\label{sec:robust-bisection-boosting}
In this section, we prove \cref{thm:robust-bisectioning} for the robust bisection algorithm in the $k$-stochastic block model. The procedure essentially constitutes recovery in the $2$-SBM model by aggregating communities into two global groups, each comprising $k/2$ local communities.

\subsection{Algorithm and constraint systems}

We first introduce the polynomial constraints that we use in the algorithm.

\smallskip
\noindent\textbf{(i) Labeling constraints.} The feasible label variable \(Z \in \{0,1\}^{n\times k}\) satisfies
\begin{equation}
    \cAlabel(Z)
    \seteq \Set{ Z(i,a)^{2}=Z(i,a), ~ \sum_{a}Z(i,a)=1, ~ \sum_{i} Z(i,a)=\tfrac n k}
    \,.
\end{equation}

\smallskip
\noindent\textbf{(ii) Initialization.} Let \(Z_{\mathrm{init}}\in \R^{n\times k}\) be an initializer $0.001$-approximation error and recovers $0.99$ fraction of the vertices in communities $a \in [k/2]$, we add constraints to enforce the SoS variable $Z$ is close to $Z_{\mathrm{init}}$,
\begin{equation}
    \label{eq:initialization-bisection}
    \cAinit(Z;Z_{\mathrm{init}}) \seteq
    \left. 
        \begin{cases}
        \Norm{Z(\cdot,a) - Z_{\mathrm{init}}(\cdot,a)}^2 \leq 0.001n/k \quad \forall a \in [k/2] \\
        \Norm{\sum_{a \in [k/2]} Z(\cdot,a) - \sum_{a \in [k/2]} Z_{\mathrm{init}}(\cdot,a)}^2 \leq 0.001n
        \end{cases}
    \right\}\,.
\end{equation}

\smallskip
\noindent\textbf{(iii) Corruption mask / node distance.} 
We model node corruptions through a binary mask \(\xi\in\{0,1\}^n\) and require the program matrix to agree with the observed graph outside the corrupted rows/columns:
\begin{equation}
    \cAclose(Y,\xi;\eta,\bar{G}) \seteq \Set{(Y-\bar{G})\odot (\one-\xi)(\one-\xi)^\top=0, \xi\odot \xi=\xi,\sum_i \xi_i \leq \delta_{\eta} n}\,,
\end{equation}
where $\delta_{\eta} = \exp(-2 C_{d,\varepsilon}) + \eta$.

\smallskip
\noindent\textbf{(iv) Community structure and spectral condition.} 
We encode the decomposition of the centered adjacency into pairwise components and bound the spectral norm of the noise blocks:
\begin{equation}
    \label{eq:mixing-constraint}
    \cAmix(Y,Z)
    \seteq \Set{X=ZZ^\top-\frac{1}{k} J, \Normop{Y - \frac{\e d}{n} X} \leq \Paren{\chi + \frac{1}{k}}\sqrt{d}} \,.
\end{equation}

\smallskip
\noindent\textbf{(v) Majority-vote consistency.}
The final constraint system characterizes the constraints for the consistency of bisection majority voting, i.e. enforcing that the induced bisection of \(X\) (bisection of vertices in community $[\frac{k}{2}]$ and $(\frac{k}{2}, k]$) is consistent with the concentration result for majority voting:
\begin{equation}
  \label{eq:bisection-majority-vote-constraints}
  \cAmaj(Y,Z,R; \gamma, \tilde{C}) \seteq
  \Set{R \colon \cAset(z)\sststile{2}{u} \Iprod{Y x(Z), x(Z) \odot z} \ge \alpha_{\gamma} \Paren{\sum_i z_i - \beta_{\gamma, \tilde{C}} n}}
  \,,
\end{equation}
where $\alpha_\gamma = \frac{(1-\gamma) \e d}{16 k}$, $\beta_{\gamma, \tilde{C}} = \frac{640 k }{1-\gamma} \exp(-\gamma \tilde{C}/2)$, \(\cAset(z) \seteq \set{z \odot z=z}\), and \(x(Z) \coloneqq \displaystyle 2 \sum_{a\in [k/2]} Z(\cdot, a)-\one\) \footnote{We write $x$ for simplicity when $Z$ is clear from context}.

For our algorithm, we define polynomial system \(\cA(Y,Z,\xi,R;\bar{G}, Z_{\mathrm{init}}, \eta,\gamma, \tilde{C})\) based on the set of constraints defined above \footnote{For simplicity, we use $\cA(Y,Z,\xi)$ to denote $\cA(Y,Z,\xi,R;\bar{G}, Z_{\mathrm{init}}, \eta,\gamma, \tilde{C})$ when the input is clear from the context.}
\begin{align}
    \label{eq:poly-system}
    \begin{split}
    & \cA(Y,Z,\xi,R;\bar{G}, Z_{\mathrm{init}}, \eta,\gamma, \tilde{C}) \\
    & \seteq \cAlabel(Z) \cup \cAinit(Z;Z_{\mathrm{init}}) \cup \cAclose(Y,\xi;\eta, \bar{G}) \cup \cAmix(Y,Z,\xi) \cup \cAmaj(Y,Z,R; \gamma, \tilde{C}) \,.
    \end{split}
\end{align}

\paragraph{Algorithm.} Now we describe our algorithm in \cref{algo:robust-bisection-boosting}.

\begin{algorithmbox}[Robust bisection boosting algorithm]\label{algo:robust-bisection-boosting}
    \mbox{}\\
    \textbf{Input:} The graph $G$ sampled from the $k$-stochastic block model $\SBM_n(d,\e,k)$ with $\eta n$ corrupted nodes, and initial clusters $S_1,S_2,\ldots, S_{k/2}\subseteq [n]$ with its corresponding label matrix $Z_{\mathrm{init}}$.
    \begin{enumerate}[1.]
        \item Finding degree-$O(1)$ pseudo-distribution satisfying the polynomial constraints \(\cA(Y,Z,\xi,R;\bar{G}, Z_{\mathrm{init}}, \eta,\gamma, \tilde{C})\)
        where $\bar{G} = G - \frac{d}{n} J$ is the centered adjacency matrix, $\gamma = \frac{1}{2}$ and $\tilde{C} = \Paren{\sqrt{a^{\frac{2}{k}} b^{1-\frac{2}{k}}}-\sqrt{b}}^2$.
        \item Rounding the pseudo-distribution by sign to obtain the bisection $\hat{x} \in \Set{\pm 1}^n$ from $x(Z) = 2 \sum_{a\in [k/2]} Z(\cdot, a) - \one$.
    \end{enumerate}
\end{algorithmbox}

\subsection{Feasibility and time complexity}

We first show that the program is feasible with high probability and runs in polynomial time.

\begin{lemma}\label{lem:bisectioning_feasible}
    Let $\Gbarnull$ be the centered adjacency matrix of the uncorrupted graph and $S$ be the set of nodes with degree larger than $20d$.
    Let $T$ be the set of corrupted nodes.
    Let $\Znull$ be the ground truth label matrix.
    Under the setting of \cref{algo:robust-bisection-boosting}, program $\cA(Y,Z,\xi)$ is satisfied by $(\Gbarnull \odot (\one_{\bar{S}} \one_{\bar{S}}^{\top}), \Znull, \one_{S \cup T})$ with probability at least $1-n^{-\Omega(1)}-\exp(-\Omega(k))$.
\end{lemma}

\begin{proof}
    Note that $\cAlabel(\Znull)$ and $\cAinit(Z;Z_{\mathrm{init}})$ is satisfied by definition. Let $\Xnull= \Znull (\Znull)^\top-\frac{1}{k} J$. 
    By \cref{cor:spectral-remove-cor}, the size of set $S$ is bounded by $\exp(-2 C_{d,\varepsilon}) \,n$ and the spectral norm of $\Gbarnull - \frac{\e d}{n} \Xnull$ is bounded by $\chi\sqrt{d}$ with high probability. Notice that
    \begin{align*}
        \Normop{\Gbarnull \odot (\one_{\bar{S}} \one_{\bar{S}}^{\top}) - \frac{\e d}{n} \Xnull}
        & \leq \Normop{\Paren{\Gbarnull - \frac{\e d}{n} \Xnull} \odot (\one_{\bar{S}} \one_{\bar{S}}^{\top})} + \frac{\e d}{n} \Normop{\Xnull \odot (J - \one_{\bar{S}} \one_{\bar{S}}^{\top})} \\
        & \leq \chi\sqrt{d} + \frac{\e d}{n} \Normf{\Xnull \odot (J - \one_{\bar{S}} \one_{\bar{S}}^{\top})} \\
        & \leq \chi\sqrt{d} + \frac{2 \e d}{n} \Normf{\Xnull \odot (\one \one_S^{\top})} \\
        & \leq \chi\sqrt{d} + \frac{4 \e \sqrt{d}}{\sqrt{k}} \exp(-C_{d,\varepsilon}) \sqrt{d} \\
        & = \chi\sqrt{d} + 4 C_{d,\varepsilon} \exp(-C_{d,\varepsilon}) \sqrt{d} \\
        & \leq \Paren{\chi + \frac{1}{k}}\sqrt{d} \,.
    \end{align*}
    Therefore, $\cAmix(\Gbarnull \odot (\one_{\bar{S}} \one_{\bar{S}}^{\top}),\Znull)$ is feasible with probability $1-n^{-O(1)}$. Moreover, $\norm{\one_{S \cup T}}_1 \leq |S|+|T| \leq \delta_{\eta} n$ with probability $1-n^{-O(1)}$. Hence, $\cAclose(\Gbarnull \odot (\one_{\bar{S}} \one_{\bar{S}}^{\top}), \one_{S \cup T};\eta,\bar{G})$ is feasible with probability $1-n^{-O(1)}$.

    By \cref{cor:inner-product-lb}, the voting lower bound in $\cAmaj(\Gbarnull \odot (\one_{\bar{S}} \one_{\bar{S}}^{\top}),\Znull,R; \gamma, \tilde{C})$ is feasible for $\gamma = \frac{1}{2}$ and $\tilde{C} = \Paren{\sqrt{a^{\frac{2}{k}} b^{1-\frac{2}{k}}}-\sqrt{b}}^2$ with probability $1-\exp(-\Omega(k))-n^{-\Omega(1)}$. The existence of the SoS proof of \cref{thm:inner-product-lb} follows by \cref{thm:sos-as-constraints} and \cref{lem:sos-certificate-subset-sum}.
    
    The feasibility proof is completed by a union bound over all failure probabilities.
\end{proof}

\begin{lemma}\label{lem:algorithm-efficiency}
    \cref{algo:robust-bisection-boosting} runs in polynomial time.
\end{lemma}

\begin{proof}
Since the sum-of-squares relaxation is at a constant degree, the resulting semidefinite program involves only polynomially many constraints and can be solved in polynomial time.
\end{proof}

\subsection{SoS guarantees for robust bisection boosting}

Now, we show that $\cA(Y,Z,\xi)$ boosts the bisection error rate from $0.001$ to $1/\poly(k)$.
To do this, we consider two feasible solutions $(Y^{(1)},Z^{(1)},\xi^{(1)})$ and $(Y^{(2)},Z^{(2)},\xi^{(2)})$ of $\cA(Y,Z,\xi)$.
We will show that the constraints in $\cA(Y,Z,\xi)$ allow us to prove that any two good solutions will be close to each other in the bisection sense.
Since $Z^{\circ}$ will also be a good solution, this implies that any feasible solution of $\cA(Y,Z,\xi)$ will be close to $Z^{\circ}$ in the bisection sense.

Throughout this section, we define the following terms:
\begin{itemize}
    \item Let $x^{(1)} = 2 \sum_{a\in [k/2]} Z^{(1)}(\cdot,a)-\one$ and $x^{(2)} = 2 \sum_{a\in [k/2]} Z^{(2)}(\cdot,a) - \one$ be the induced bisections of $Z^{(1)}$ and $Z^{(2)}$.
    \item Let $v=\frac{\one-x^{(1)} \odot x^{(2)}}{2}$ be the set of nodes where the bisections $x^{(1)}$ and $x^{(2)}$ differ.
    \item Let $s=\one - (\one-\xi^{(1)}) \odot (\one-\xi^{(2)})$ be the set of nodes where $Y^{(1)}$ and $Y^{(2)}$ differ.
    \item Let $w=v \odot (\one-s)$ be the set of vertices in the shared part of $Y^{(1)}$ and $Y^{(2)}$ such that the bisections $x^{(1)}$ and $x^{(2)}$ differ.
    \item Let $g=\one - (\one-v) \odot (\one-s)$ be union of $v$ and $s$, notice that, equivalently $g=w+s$.
\end{itemize}

To prove the algorithmic guarantee of $\cA(Y,Z,\xi)$, we need the following observations on properties of the variables defined above.

\begin{lemma}
\label{clm:bisection_equivalences}
    $x^{(1)}$, $x^{(2)}$ and $v$ satisfies
    \begin{equation*}
        x^{(1)} \odot v = - x^{(2)} \odot v
       \quad \text{and} \quad
       x^{(1)} \odot (\one-v) = x^{(2)} \odot (\one-v) \,,
    \end{equation*}
    and,
    \begin{equation*}
       v \odot \frac{\one-x^{(1)}}{2} = v \odot \frac{\one+x^{(2)}}{2}
       \quad \text{and} \quad
       v \odot \frac{\one-x^{(2)}}{2} = v \odot \frac{\one+x^{(1)}}{2}\,.
    \end{equation*}
\end{lemma}

\begin{proof}
    By the definition of $v=\frac{\one-x^{(1)} \odot x^{(2)}}{2}$, it follows that
    \begin{equation*}
        x^{(1)} \odot v = \frac{x^{(1)}-x^{(2)}}{2} = - x^{(2)} \odot v \,,
    \end{equation*}
    and,
    \begin{equation*}
        x^{(1)} \odot (\one - v) = \frac{x^{(1)}+x^{(2)}}{2} = x^{(2)} \odot (\one - v) \,.
    \end{equation*}
    Now, we show $v \odot \frac{\one-x^{(1)}}{2} = v \odot \frac{\one+x^{(2)}}{2}$, and the other side follows by symmetry. By plugging in $v=\frac{\one-x^{(1)} \odot x^{(2)}}{2}$, the left hand side is
    \begin{equation*}
        v \odot \frac{\one-x^{(1)}}{2} = \frac{\one-x^{(1)}-x^{(1)} \odot x^{(2)} + x^{(2)}}{4} \,,
    \end{equation*}
    and the right hand side is
    \begin{equation*}
        v \odot \frac{\one+x^{(2)}}{2} = \frac{\one+x^{(2)}-x^{(1)} \odot x^{(2)} - x^{(1)}}{4} \,.
    \end{equation*}
    Therefore, we have $v \odot \frac{\one-x^{(1)}}{2} = v \odot \frac{\one+x^{(2)}}{2}$.
\end{proof}

\begin{lemma}
    \label{fact:key-equality}
    \begin{align*}
        & \cA\Paren{Y^{(1)},Z^{(1)},\xi^{(1)}},\cA\Paren{Y^{(2)},Z^{(2)},\xi^{(2)}}  \\
        & \sststile{O(1)}{Y^{(1)},Z^{(1)},\xi^{(1)},Y^{(2)},Z^{(2)},\xi^{(2)}} \Iprod{Y^{(1)} (x^{(1)} \odot (\one-g)), x^{(1)} \odot w} = -\Iprod{Y^{(2)} (x^{(2)} \odot (\one-g)) , x^{(2)}\odot w}\,.
    \end{align*}
\end{lemma}

\begin{proof}
Plugging in the definition of $g$ and $w$, the left hand side can be written as
\begin{align*}
    \Iprod{Y^{(1)} (x^{(1)} \odot (\one-g)), x^{(1)} \odot w}
    & = \Iprod{Y^{(1)} (x^{(1)} \odot (\one-v) \odot (\one-s)), x^{(1)} \odot v \odot (\one-s)} \\
    & = \Iprod{\Paren{Y^{(1)} \odot (\one-s) (\one-s)^{\top}} (x^{(1)} \odot (\one-v)), x^{(1)} \odot v} \,.
\end{align*}
By $\cAclose(Y,\xi;\eta,\bar{G})$, we have
\begin{equation*}
    Y^{(1)}\odot (\one-s)(\one-s)^\top = \bar{G} \odot (\one - s)(\one -s)^\top = Y^{(2)}\odot (\one-s)(\one-s)^\top \,,
\end{equation*}
and, by \cref{clm:bisection_equivalences}, we have
    \begin{equation*}
        x^{(1)} \odot v = - x^{(2)} \odot v
       \quad \text{and} \quad
       x^{(1)} \odot (\one-v) = x^{(2)} \odot (\one-v) \,,
    \end{equation*}
Thus,
\begin{align*}
    \Iprod{Y^{(1)} (x^{(1)} \odot (\one-g)), x^{(1)} \odot w}
    & = - \Iprod{\Paren{Y^{(2)} \odot (\one-s) (\one-s)^{\top}} (x^{(2)} \odot (\one-v)), x^{(2)} \odot v} \\
    & = - \Iprod{Y^{(2)} (x^{(2)} \odot (\one-g)), x^{(2)} \odot w} \,.
\end{align*}
\end{proof}

\begin{lemma}
\label{lem:agreement_lemma}
    The norm of $v$, $s$ and $g$ satisfies the following SoS inequalities
    \begin{equation*}
        \cA\Paren{Y^{(1)},Z^{(1)},\xi^{(1)}},\cA\Paren{Y^{(2)},Z^{(2)},\xi^{(2)}}
        \sststile{O(1)}{Z^{(1)},Z^{(2)}} \Norm{v}^2 \leq 0.004 n \,,
    \end{equation*}
    and,
    \begin{equation*}
        \cA\Paren{Y^{(1)},Z^{(1)},\xi^{(1)}},\cA\Paren{Y^{(2)},Z^{(2)},\xi^{(2)}}
        \sststile{O(1)}{\xi^{(1)},\xi^{(2)}} \Norm{s}^2 \leq 2 \delta_{\eta} n \,,
    \end{equation*}
    and,
    \begin{equation*}
        \cA\Paren{Y^{(1)},Z^{(1)},\xi^{(1)}},\cA\Paren{Y^{(2)},Z^{(2)},\xi^{(2)}}
        \sststile{O(1)}{Z^{(1)},Z^{(2)},\xi^{(1)},\xi^{(2)}} \Norm{w}^2 \leq 0.004 n \,,
    \end{equation*}
    and, for any $b \in [\frac{k}{2}]$ and $t \in \set{1, 2}$,
    \begin{equation*}
        \cA\Paren{Y^{(1)},Z^{(1)},\xi^{(1)}},\cA\Paren{Y^{(2)},Z^{(2)},\xi^{(2)}}
        \sststile{O(1)}{Z^{(1)},Z^{(2)},\xi^{(1)},\xi^{(2)}} \Iprod{Z^{(t)}(\cdot,b), w} \leq \frac{0.004n}{k} \,.
    \end{equation*}
\end{lemma}

\begin{proof}
    For the bound on $v$, by constraints in $\cAinit(Z;Z_{\mathrm{init}})$, it follows that
    \begin{align*}
        \Norm{v}^2
        & = \Norm{\frac{\one-x^{(1)} \odot x^{(2)}}{2}}^2
        = \Norm{\frac{x^{(1)} - x^{(2)}}{2}}^2 \\
        & = \Norm{\sum_{a\in [k/2]} Z^{(1)}(\cdot,a) - \sum_{a\in [k/2]} Z^{(2)}(\cdot,a)}^2 \\
        & \leq 2 \Norm{\sum_{a\in [k/2]} Z^{(1)}(\cdot,a) - \sum_{a \in [k/2]} Z_{\mathrm{init}}(\cdot,a)}^2 + 2 \Norm{\sum_{a\in [k/2]} Z^{(2)}(\cdot,a) - \sum_{a \in [k/2]} Z_{\mathrm{init}}(\cdot,a)}^2 \\
        & \leq 0.004 n \,.
    \end{align*}
    For the bound on $s$, by constraints in $\cAclose(Y,\xi;\eta,\bar{G})$, it follows that
    \begin{align*}
        \Norm{s}^2 &= \sum_i s_i
        = \sum_i 1 - (1-\xi^{(1)}_i) \odot (1-\xi^{(2)}_i)
        = \sum_i \xi^{(1)}_i + \sum_i \xi^{(2)}_i (1-\xi^{(1)}_i) \\
        & \leq \sum_i \xi^{(1)}_i + \sum_i \xi^{(2)}_i
        \leq 2 \delta_{\eta} n \,.
    \end{align*}
    Now, consider $w = v \odot (\one - s)$. Therefore,
    \begin{equation*}
        \Norm{w}^2 = \Norm{v \odot (\one - s)}^2 \leq \Norm{v}^2 \leq 0.004 n \,.
    \end{equation*}
    For the last inequality, for $b \in [\frac{k}{2}]$, it follows that
    \begin{equation*}
        \Iprod{Z^{(1)}(\cdot,b), w}
        = \sum_i Z^{(1)}(i,b) v_i (1-s_i)
        \leq \sum_i Z^{(1)}(i,b) v_i \,.
    \end{equation*}
    Since $v = x^{(1)} \odot \Big(\sum_{a\in [k/2]} Z^{(1)}(\cdot,a) - \sum_{a\in [k/2]} Z^{(2)}(\cdot,a)\Big)$ and $x^{(1)} \odot Z^{(1)}(\cdot,b) = Z^{(1)}(\cdot,b)$ for $b \in [\frac{k}{2}]$, we have
    \begin{align*}
        \sum_i Z^{(1)}(i,b) v_i
        &= \sum_i Z^{(1)}(i,b) x^{(1)}_i \Big(\sum_{a\in [k/2]} Z^{(1)}(i,a) - \sum_{a\in [k/2]} Z^{(2)}(i,a)\Big) \\
        &= \sum_i Z^{(1)}(i,b) - \sum_{a\in [k/2]} Z^{(2)}(i,a)Z^{(1)}(i,b) \\
        &\le \sum_i Z^{(1)}(i,b) - Z^{(2)}(i,b)Z^{(1)}(i,b) \\
        &= \sum_i Z^{(1)}(i,b)\big(1 - Z^{(2)}(i,b)\big) \\
        &\le \sum_i \big(Z^{(1)}(i,b) - Z^{(2)}(i,b)\big)^2 \\
        &= \Norm{Z^{(1)}(\cdot,b) - Z^{(2)}(\cdot,b)}^2 \\
        &\le 2 \Norm{Z^{(1)}(\cdot,b) - Z_{\mathrm{init}}(\cdot,b)}^2 + 2 \Norm{Z^{(2)}(\cdot,b) - Z_{\mathrm{init}}(\cdot,b)}^2 \\
        &\le \frac{0.004\,n}{k} \,.
    \end{align*}
    Thus,
    \begin{equation*}
        \Iprod{Z^{(1)}(\cdot,b), w}
        \leq \sum_i Z^{(1)}(i,b) v_i
        \leq \frac{0.004\,n}{k} \,.
    \end{equation*}
\end{proof}

\begin{lemma}
\label{lem:sos-signal-upper-bound-single}
    For any $a\in [\frac{k}{2}]$ and $t \in \set{1, 2}$,
    \begin{align*}
        \cA\Paren{Y^{(1)},Z^{(1)},\xi^{(1)}}, & \cA\Paren{Y^{(2)},Z^{(2)},\xi^{(2)}} \\
        \sststile{O(1)}{Y^{(1)},Z^{(1)},\xi^{(1)},Y^{(2)},Z^{(2)},\xi^{(2)}} & \Iprod{X^{(t)} \Paren{x^{(t)} \odot w}, x^{(t)} \odot w \odot Z^{(t)}(\cdot,a)} \leq \frac{0.008 n}{k} \Iprod{w, Z^{(t)}(\cdot,a)} \,,
    \end{align*}
    and,
    \begin{align*}
        \cA\Paren{Y^{(1)},Z^{(1)},\xi^{(1)}}, & \cA\Paren{Y^{(2)},Z^{(2)},\xi^{(2)}} \\
        \sststile{O(1)}{Y^{(1)},Z^{(1)},\xi^{(1)},Y^{(2)},Z^{(2)},\xi^{(2)}} & \Iprod{X^{(t)} \Paren{x^{(t)} \odot w}, x^{(t)} \odot w \odot Z^{(t)}(\cdot,a)}
        \geq - \frac{0.004 n}{k} \Iprod{w, Z^{(t)}(\cdot,a)} \,.
    \end{align*}
\end{lemma}

\begin{proof}
    Recall that $X^{(t)}=Z^{(t)}(Z^{(t)})^\top-\frac{1}{k} J$. We consider $Z^{(t)}(Z^{(t)})^\top$ and $\frac{1}{k} J$ separately. For the $Z^{(t)}(Z^{(t)})^\top$ part, notice that, since $x_i = 2 \sum_{a \in [k/2]} Z(i, a)-1$, we have $x_i^{(t)} \cdot Z^{(t)}(i, a) = Z^{(t)}(i, a)$ for $a \in [k/2]$. Therefore, it follows that
    \begin{align*}
        \Iprod{Z^{(t)} (Z^{(t)})^{\top} \Paren{x^{(t)} \odot w}, x^{(t)} \odot w \odot Z(\cdot,a)}
        & = \sum_{b\in [k]} \Iprod{Z^{(t)}(\cdot,b) Z^{(t)}(\cdot,b) ^\top \Paren{x^{(t)} \odot w}, x^{(t)} \odot w \odot Z^{(t)}(\cdot,a)} \\
        &=\sum_{b\in [k]} Z^{(t)}(\cdot,b)^\top \Paren{x^{(t)} \odot w} Z^{(t)}(\cdot,b)^\top (x^{(t)} \odot w \odot Z^{(t)}(\cdot,a))\\
        &= Z^{(t)}(\cdot,a) ^\top (x^{(t)} \odot w) Z^{(t)}(\cdot,a)^\top (x^{(t)} \odot w \odot Z^{(t)}(\cdot,a))\\
        &= \Iprod{Z^{(t)}(\cdot,a), w} Z^{(t)}(\cdot,a)^\top (w \odot Z^{(t)}(\cdot,a))\\
        & \leq \frac{0.004 n}{k} \Iprod{w, Z^{(t)}(\cdot,a)} \,,
    \end{align*}
    where the last inequality is by \cref{lem:agreement_lemma}. Moreover, from the last equality, we can also obtain
    \begin{equation*}
        \Iprod{Z^{(t)} (Z^{(t)})^{\top} \Paren{x^{(t)} \odot w}, x^{(t)} \odot w \odot Z(\cdot,a)} \geq 0
    \end{equation*}
    For the $\frac{1}{k} J$ part, we have
    \begin{align*}
        \Iprod{\frac{1}{k} J \Paren{x^{(t)} \odot w}, x^{(t)} \odot w \odot Z(\cdot,a)}
        & = \frac{1}{k} \Iprod{\one, x^{(t)} \odot w} \Iprod{\one, x^{(t)} \odot w \odot Z^{(t)}(\cdot,a)} \\
        & \leq \frac{1}{k} \Iprod{\one, w} \Iprod{w, Z^{(t)}(\cdot,a)} \\
        & \leq \frac{0.004 n}{k} \Iprod{w, Z^{(t)}(\cdot,a)} \,.
    \end{align*}
    Equivalently, we can obtain
    \begin{equation*}
        \Iprod{\frac{1}{k} J \Paren{x^{(t)} \odot w}, x^{(t)} \odot w \odot Z(\cdot,a)}
        \geq - \frac{0.004n}{k} \Iprod{w, Z^{(t)}(\cdot,a)} \,.
    \end{equation*}
    Therefore, we have
    \begin{align*}
        \Iprod{X^{(t)} \Paren{x^{(t)} \odot w}, x^{(t)} \odot w \odot Z^{(t)}(\cdot,a)}
        & = \Iprod{\Paren{Z^{(t)} (Z^{(t)})^\top-\frac{1}{k} J} \Paren{x^{(t)} \odot w}, x^{(t)} \odot w \odot Z^{(t)}(\cdot,a)} \\
        & \leq \frac{0.008 n}{k} \Iprod{w, Z^{(t)}(\cdot,a)} \,,
    \end{align*}
    and,
    \begin{equation*}
        \Iprod{X^{(t)} \Paren{x^{(t)} \odot w}, x^{(t)} \odot w \odot Z^{(t)}(\cdot,a)}
        \geq - \frac{0.004 n}{k} \Iprod{w, Z^{(t)}(\cdot,a)} \,.
    \end{equation*}
\end{proof}

\begin{corollary}\label{cor:sos-signal-upper-bound}
    For any $t \in \set{1, 2}$,
    \begin{align*}
        &\cA\Paren{Y^{(1)},Z^{(1)},\xi^{(1)}}, \cA\Paren{Y^{(2)},Z^{(2)},\xi^{(2)}} \\ 
        &\sststile{O(1)}{Y^{(1)},Z^{(1)},\xi^{(1)},Y^{(2)},Z^{(2)},\xi^{(2)}}
        \Iprod{X^{(t)}\Paren{x^{(t)}\odot w}, x^{(t)}\odot w\odot \frac{x^{(t)}+\one}{2}}
        \leq \frac{0.008 n}{k} \Norm{w}^2 \,,
    \end{align*}
    and,
    \begin{align*}
        &\cA\Paren{Y^{(1)},Z^{(1)},\xi^{(1)}}, \cA\Paren{Y^{(2)},Z^{(2)},\xi^{(2)}} \\ 
        &\sststile{O(1)}{Y^{(1)},Z^{(1)},\xi^{(1)},Y^{(2)},Z^{(2)},\xi^{(2)}}
        \Iprod{X^{(t)}\Paren{x^{(t)}\odot w}, x^{(t)}\odot w\odot \frac{x^{(t)}+\one}{2}}
        \geq - \frac{0.004 n}{k} \Norm{w}^2 \,.
    \end{align*}
\end{corollary}

\begin{proof}
    Recall that $x^{(t)} = 2 \sum_{a\in [k/2]} Z^{(t)}(\cdot,a)-\one$. Therefore,
    \begin{align*}
        \Iprod{X^{(t)}\Paren{x^{(t)}\odot w}, x^{(t)} \odot w \odot \frac{(x^{(t)}+\one)}{2}}
        &=\Iprod{X^{(t)}\Paren{x^{(t)}\odot w}, x^{(t)} \odot w \odot \sum_{a\in [k/2]} Z^{(t)}(\cdot,a)} \\
        &=\sum_{a\in [k/2]} \Iprod{X^{(t)}\Paren{x^{(t)}\odot w}, x^{(t)} \odot w \odot  Z^{(t)}(\cdot,a)} \,.
    \end{align*}
    Applying \cref{lem:sos-signal-upper-bound-single}, we can get
    \begin{align*}
        \Iprod{X^{(t)}\Paren{x^{(t)}\odot w}, x^{(t)} \odot w \odot \frac{(x^{(t)}+\one)}{2}}
        & \leq \sum_{a\in [k/2]} \frac{0.008 n}{k} \Iprod{w, Z^{(t)}(\cdot,a)} \\
        & = \frac{0.008 n}{k} \Iprod{w, \frac{x^{(t)}+\one}{2}} \\
        & \leq \frac{0.008 n}{k} \Norm{w}^2 \,,
    \end{align*}
    and,
    \begin{align*}
        \Iprod{X^{(t)}\Paren{x^{(t)}\odot w}, x^{(t)} \odot w \odot \frac{(x^{(t)}+\one)}{2}}
        & \geq - \sum_{a\in [k/2]} \frac{0.004 n}{k} \Iprod{w, Z^{(t)}(\cdot,a)} \\
        & = - \frac{0.004 n}{k} \Iprod{w, \frac{x^{(t)}+\one}{2}}\\
        & \geq - \frac{0.004 n}{k} \Norm{w}^2 \,.
    \end{align*}
\end{proof}

The following bound on $\Iprod{Y^{(t)}\Paren{x^{(t)}\odot g}, x^{(t)}\odot w}$ will be the key thing that we need in the proof of the algorithmic guarantees.

\begin{lemma}
    \label{lem:sos-validity-condition}
    For $t\in \{1,2\}$,
    \begin{align*}
        &\cA\Paren{Y^{(1)},Z^{(1)},\xi^{(1)}}, \cA\Paren{Y^{(2)},Z^{(2)},\xi^{(2)}} \\ 
        &\sststile{O(1)}{Y^{(1)},Z^{(1)},\xi^{(1)},Y^{(2)},Z^{(2)},\xi^{(2)}} \Iprod{Y^{(t)}\Paren{x^{(t)}\odot g}, x^{(t)}\odot w} \leq \frac{0.016 \eps d}{k} \Norm{w}^2 + 4 \eps d \delta_{\eta} n \,.
    \end{align*}  
\end{lemma}

\begin{proof}
    Without loss of generality, we prove the statement for $t=1$. Let $Y^{(1)}_S \seteq Y^{(1)} \odot (\one-s)(\one-s)^\top$, we use decomposition
    \begin{equation}
        \label{eq:sos-validity-condition-1}
        \Iprod{Y^{(1)}\Paren{x^{(1)}\odot g}, x^{(1)}\odot w}
        = \Iprod{\Paren{Y^{(1)} - Y^{(1)}_S} \Paren{x^{(1)}\odot g}, x^{(1)}\odot w} + \Iprod{Y^{(1)}_S\Paren{x^{(1)}\odot g}, x^{(1)}\odot w} \,.
    \end{equation}
    For the term $\Iprod{\Paren{Y^{(1)} - Y^{(1)}_S} \Paren{x^{(1)}\odot g}, x^{(1)}\odot w}$, since $Y^{(1)} - Y^{(1)}_S = Y^{(1)} \odot \Paren{J - (\one-s)(\one-s)^\top}$, $g = w+s$, and $w = v \odot (\one - s)$, we have $\Iprod{\Paren{Y^{(1)} - Y^{(1)}_S} \Paren{x^{(1)}\odot w}, x^{(1)}\odot w} = 0$ and $\Paren{J - (\one-s)(\one-s)^\top} \odot \Paren{x^{(1)}\odot s} \Paren{x^{(1)}\odot w}^{\top} = \Paren{x^{(1)}\odot s} \Paren{x^{(1)}\odot w}^{\top}$. Therefore,
    \begin{align*}
        \Iprod{\Paren{Y^{(1)} - Y^{(1)}_S} \Paren{x^{(1)}\odot g}, x^{(1)}\odot w}
        = & \Iprod{\Paren{Y^{(1)} - Y^{(1)}_S} \Paren{x^{(1)}\odot s}, x^{(1)}\odot w} \\
        = & \Iprod{Y^{(1)} \odot \Paren{J - (\one-s)(\one-s)^\top}, \Paren{x^{(1)}\odot s} \Paren{x^{(1)}\odot w}^{\top}} \\
        = & \Iprod{Y^{(1)}, \Paren{x^{(1)}\odot s} \Paren{x^{(1)}\odot w}^{\top}} \,.
    \end{align*}
    let $Y^{(1)} = \frac{\eps d}{n} X^{(1)} + E^{(1)}$ where $X^{(1)}=Z^{(1)}(Z^{(1)})^\top-\frac{1}{k} J$ and $E^{(1)} = Y^{(1)} - \frac{\eps d}{n} X^{(1)}$, it follows that
    \begin{align*}
        \Iprod{\Paren{Y^{(1)} - Y^{(1)}_S} \Paren{x^{(1)}\odot g}, x^{(1)}\odot w}
        = & \frac{\eps d}{n} \Iprod{Z^{(1)}(Z^{(1)})^\top, \Paren{x^{(1)}\odot s} \Paren{x^{(1)}\odot w}^{\top}} \\
        & - \frac{\eps d}{k n} \Iprod{J , \Paren{x^{(1)}\odot s} \Paren{x^{(1)}\odot w}^{\top}} \\
        & + \Iprod{E^{(1)}, \Paren{x^{(1)}\odot s} \Paren{x^{(1)}\odot w}^{\top}} \,.
    \end{align*}
    For the first term, since $\zeros \leq Z^{(1)}(Z^{(1)})^\top \leq J$ and $x^{(1)} \leq \one$, it follows that
    \begin{equation*}
        \frac{\eps d}{n} \Iprod{Z^{(1)}(Z^{(1)})^\top, \Paren{x^{(1)}\odot s} \Paren{x^{(1)}\odot w}^{\top}}
        \leq \frac{\eps d}{n} \Iprod{J, s w^{\top}}
        = \frac{\eps d}{n} \Paren{\sum_{i} s_i} \Paren{\sum_{i} w_i}
        \leq 2 \eps d \delta_{\eta} \norm{w}^2 \,.
    \end{equation*}
    For the second term, since $x^{(1)} \geq -\one$, it follows that
    \begin{equation*}
        - \frac{\eps d}{k n} \Iprod{J, \Paren{x^{(1)}\odot s} \Paren{x^{(1)}\odot w}^{\top}}
        \leq \frac{\eps d}{k n} \Iprod{J, s w^{\top}}
        =  \frac{\eps d}{k n} \Paren{\sum_{i} s_i} \Paren{\sum_{i} w_i}
        \leq \frac{2 \eps d \delta_{\eta}}{k} \norm{w}^2 \,.
    \end{equation*}
    For the third term, $\cAmix(Y,Z)$ and SoS Cauchy-Schwarz inequality, it follows that
    \begin{align*}
        \Iprod{E^{(1)}, \Paren{x^{(1)}\odot s} \Paren{x^{(1)}\odot w}^{\top}}
        & \leq \Paren{\chi + \frac{1}{k}}\sqrt{d} \Norm{x^{(1)}\odot s}^2 + \frac{1}{\Paren{\chi + \frac{1}{k}}\sqrt{d}}\Normop{E^{(1)}}^2 \Norm{x^{(1)}\odot w}^2 \\
        & \leq \Paren{\chi + \frac{1}{k}}\sqrt{d} \delta_{\eta} n  + \Paren{\chi + \frac{1}{k}}\sqrt{d} \Norm{w}^2
    \end{align*}
    Since $\eps^2 d \gg k^2$, we have $\Paren{\chi + \frac{1}{k}}\sqrt{d} \ll \frac{\eps d}{k}$, we can obtain
    \begin{equation*}
        \Iprod{E^{(1)}, \Paren{x^{(1)}\odot s} \Paren{x^{(1)}\odot w}^{\top}}
        \leq \frac{0.0001 \eps d \delta_{\eta}}{k} n  + \frac{0.001 \eps d}{k} \Norm{w}^2 \,.
    \end{equation*}
    Thus, it follows that
    \begin{equation}
    \label{eq:sos-validity-condition-2}
        \Iprod{\Paren{Y^{(1)} - Y^{(1)}_S} \Paren{x^{(1)}\odot g}, x^{(1)}\odot w}
        \leq \frac{0.001 \eps d}{k} \Norm{w}^2 + 4 \eps d \delta_{\eta} n \,.
    \end{equation}
    Now, we consider the term $\Iprod{Y^{(1)}_S\Paren{x^{(1)}\odot g}, x^{(1)}\odot w}$. By $\cAclose(Y,\xi;\eta,\bar{G})$, we have
    \begin{equation*}
        Y^{(1)} \odot (\one-s)(\one-s)^\top = \bar{G} \odot (\one-s)(\one-s)^\top = Y^{(2)}\odot (\one-s)(\one-s)^\top\,.
    \end{equation*}
    By \cref{clm:bisection_equivalences} and $w = v \odot (\one - s)$,
    \begin{equation*}
        \frac{\one-x^{(1)}}{2}\odot w=\frac{\one+x^{(2)}}{2}\odot w
        \quad \text{and} \quad
        x^{(1)}\odot w = -x^{(2)}\odot w\,.
    \end{equation*}
    Therefore, we have
    \begin{align*}
        & \Iprod{Y^{(1)}_S \Paren{x^{(1)}\odot g}, x^{(1)}\odot w} \\
        & =\Iprod{Y^{(1)}_S \Paren{x^{(1)}\odot g}, x^{(1)}\odot w\odot \frac{1+x^{(1)}}{2}}+\Iprod{Y^{(1)}_S \Paren{x^{(1)}\odot g}, x^{(1)}\odot w\odot \frac{1-x^{(1)}}{2}}\\
        & =\Iprod{Y^{(1)}_S \Paren{x^{(1)}\odot g}, x^{(1)}\odot w \odot \frac{1+x^{(1)}}{2}}-\Iprod{Y^{(2)}_S \Paren{x^{(2)}\odot g}, x^{(2)}\odot w \odot \frac{1+x^{(2)}}{2}} \\
        & =\Iprod{Y^{(1)} \Paren{x^{(1)}\odot w}, x^{(1)}\odot w \odot \frac{1+x^{(1)}}{2}}-\Iprod{Y^{(2)} \Paren{x^{(2)}\odot w}, x^{(2)}\odot w \odot \frac{1+x^{(2)}}{2}} \,.
    \end{align*}
    Consider the first term, it follows that
    \begin{align*}
        \Iprod{Y^{(1)} \Paren{x^{(1)}\odot w}, x^{(1)}\odot w \odot \frac{1+x^{(1)}}{2}}
        = & \frac{\eps d}{n}\Iprod{X^{(1)} \Paren{x^{(1)}\odot w}, x^{(1)}\odot w \odot \frac{1+x^{(1)}}{2}} \\
        & + \Iprod{E^{(1)} \Paren{x^{(1)}\odot w}, x^{(1)}\odot w \odot \frac{1+x^{(1)}}{2}} \,.
    \end{align*}
    By \cref{cor:sos-signal-upper-bound}, we have
    \begin{equation*}
        \frac{\eps d}{n}\Iprod{X^{(1)} \Paren{x^{(1)}\odot w}, x^{(1)}\odot w \odot \frac{1+x^{(1)}}{2}}
        \leq \frac{0.008 \eps d}{k} \Norm{w}^2 \,.
    \end{equation*}
    By $\cAmix(Y,Z)$ and SoS Cauchy-Schwarz inequality, we have
    \begin{equation*}
       \Iprod{E^{(1)} \Paren{x^{(1)}\odot w}, x^{(1)}\odot w \odot \frac{1+x^{(1)}}{2}}
       \leq \Paren{\chi + \frac{1}{k}}\sqrt{d} \norm{w}^2\,.
    \end{equation*}
    Since $\eps^2 d \gg k^2$, we have $\Paren{\chi + \frac{1}{k}}\sqrt{d} \ll \frac{\eps d}{k}$, and, therefore,
    \begin{equation*}
        \Iprod{Y^{(1)} \Paren{x^{(1)}\odot w}, x^{(1)}\odot w \odot \frac{1+x^{(1)}}{2}}
        \leq \frac{0.01 \eps d}{k} \Norm{w}^2\,.
    \end{equation*}
    Using similar analysis on the lower bound side, we can obtain 
    \begin{equation*}
       \Iprod{Y^{(2)} \Paren{x^{(2)}\odot w}, x^{(2)}\odot w \odot \frac{1+x^{(2)}}{2}}
        \geq - \frac{0.005 \eps d}{k} \Norm{w}^2\,.
    \end{equation*}
    Thus,
    \begin{equation}
    \label{eq:sos-validity-condition-3}
        \Iprod{Y^{(1)}_S \Paren{x^{(1)}\odot g}, x^{(1)}\odot w}
        \leq \frac{0.015 \eps d}{k} \Norm{w}^2\,.
    \end{equation}
    Plugging \cref{eq:sos-validity-condition-2} and \cref{eq:sos-validity-condition-3} into \cref{eq:sos-validity-condition-1}, it follows that
    \begin{equation*}
        \Iprod{Y^{(1)}\Paren{x^{(1)}\odot g}, x^{(1)}\odot w} \leq \frac{0.016 \eps d}{k} \Norm{w}^2 + 4 \eps d \delta_{\eta} n \,.
    \end{equation*}
\end{proof}

Now combining \cref{fact:key-equality} and \cref{lem:sos-validity-condition}, we can prove \cref{thm:Identifiability-proof}.

\begin{theorem}
\label{thm:Identifiability-proof}
    Consider the setting of \cref{algo:robust-bisection-boosting}. We have
\begin{align*}
    \cA\Paren{Y^{(1)},Z^{(1)},\xi^{(1)}},\cA\Paren{Y^{(2)},Z^{(2)},\xi^{(2)}} &  \\
    \sststile{O(1)}{Y^{(1)},Z^{(1)},\xi^{(1)},Y^{(2)},Z^{(2)},\xi^{(2)}} \Norm{\frac{\one-x^{(1)}\odot x^{(2)}}{2}}^2 &\leq O \Paren{\frac{k}{0.96-\gamma} \Paren{\exp(-\frac{\gamma \tilde{C}}{2}) + \delta_{\eta}} n} \,.
\end{align*} 
\end{theorem}

\begin{proof}
    By \cref{fact:key-equality}, we can obtain
    \begin{equation*}
        \Iprod{Y^{(1)} (x^{(1)} \odot (\one-g)), x^{(1)} \odot w} = -\Iprod{Y^{(2)} (x^{(2)} \odot (\one-g)) , x^{(2)}\odot w} \,.
    \end{equation*}
   Therefore, we have
   \begin{equation*}
        \Iprod{Y^{(1)} (x^{(1)}), x^{(1)} \odot w} + \Iprod{Y^{(2)} (x^{(2)}) , x^{(2)}\odot w}
        =\Iprod{Y^{(1)} (x^{(1)} \odot g), x^{(1)} \odot w} + \Iprod{Y^{(2)} (x^{(2)} \odot g) , x^{(2)}\odot w}\,.
   \end{equation*}
   By $\cAmaj(Y,Z,R; \gamma, \tilde{C})$, it follows that
   \begin{equation*}
        \Iprod{Y^{(1)} (x^{(1)}), x^{(1)} \odot w} + \Iprod{Y^{(2)} (x^{(2)}) , x^{(2)}\odot w}
        \geq 2 \alpha_{\gamma} \Paren{\norm{w}^2 - \beta_{\gamma, \tilde{C}} n} \,.
   \end{equation*}
   By \cref{lem:sos-validity-condition}, we have for each $t \in \set{1, 2}$,
   \begin{equation*}
         \Iprod{Y^{(t)}\Paren{x^{(t)}\odot g}, x^{(t)}\odot w}
         \leq \frac{0.016 \eps d}{k} \Norm{w}^2 + 4 \eps d \delta_{\eta} n \,.
   \end{equation*}
    Therefore, we have
    \begin{equation*}
        \alpha_{\gamma} \Paren{\norm{w}^2 - \beta_{\gamma, \tilde{C}} n}
        \leq \frac{0.016 \eps d}{k} \Norm{w}^2 + 4 \eps d \delta_{\eta} n \,.
    \end{equation*}
    By rearranging terms and plugging in $\alpha_\gamma = \frac{(1-\gamma) \e d}{16 k}$ and $\beta_{\gamma, \tilde{C}} = \frac{640 k }{1-\gamma} \exp(-\gamma \tilde{C}/2)$, we can obtain
    \begin{equation*}
        \Norm{w}^2 \leq O \Paren{\frac{k}{0.96-\gamma} \Paren{\exp(-\frac{\gamma \tilde{C}}{2}) + \delta_{\eta}} n}\,.
    \end{equation*}
    Since $w = v \odot (\one-s)$, it follows that
    \begin{equation*}
        \Norm{v}^2
        = \Norm{w + v \odot s}^2
        \leq 2 \Norm{w}^2 + 2 \Norm{v \odot s}^2
        \leq O \Paren{\frac{k}{0.96-\gamma} \Paren{\exp(-\frac{\gamma \tilde{C}}{2}) + \delta_{\eta}} n} \,.
    \end{equation*}
\end{proof}

\subsection{Algorithmic guarantees for robust bisection boosting}

Now, we are ready to prove \cref{thm:robust-bisectioning}.
\restatetheorem{thm:robust-bisectioning}

\begin{proof}
    Since the program is feasible for the spectrally truncated graph $\Gbarnull \odot (\one_{\bar{S}} \one_{\bar{S}}^{\top})$ with ground-truth label $\Znull$ with the ground-truth bisection be \(x^{\circ} \in \{\pm 1\}^n\), by \cref{thm:Identifiability-proof}, it follows that
    \begin{equation*}
        \frac{1}{n}\Norm{\tilde{\E} x - x^{\circ}}^2
        \leq O \Paren{\frac{k}{0.96-\gamma} \Paren{\exp\Bigl(-\frac{\gamma \tilde{C}}{4}\Bigr) + \delta_{\eta}}} \,,
    \end{equation*}
    where $\tilde{C} = \Paren{\sqrt{a^{\frac{2}{k}} b^{1-\frac{2}{k}}}-\sqrt{b}}^2$. Plug in $\gamma = \frac{1}{2}$ and $\delta_{\eta} = \exp(-2 C_{d,\varepsilon}) + \eta$, it follows that
    \begin{equation*}
        \frac{1}{n}\Norm{\tilde{\E} x - x^{\circ}}^2
        \leq O \Paren{k \Paren{\exp\Bigl(-\frac{\tilde{C}}{8}\Bigr) + \eta}}\,.
    \end{equation*}
    Let $\hat{x} = \sign(\tilde{\E} x)$. Since the entrywise error of sign rounding increases by at most a multiplicative factor of $O(1)$, it follows that
    \begin{equation*}
        \frac{1}{n}\Norm{\hat{x} - x^{\circ}}^2
        \leq O \Paren{k \Paren{\exp\Bigl(-\frac{\tilde{C}}{8}\Bigr) + \eta}}
        = \exp\left(-\bigl(1 - o(1)\bigr)\frac{\tilde{C}}{8}\right) + O(k \eta) \,.
    \end{equation*}
\end{proof}
\section{Robust bisection algorithm}
\label{sec:robust-bisection-algo}

In this section, we combine results from \cref{sec:Initialization}, \cref{sec:IdentifyingRecoveredBlocks} and \cref{sec:robust-bisection-boosting} to design an algorithm for robust bisectioning and prove \cref{thm:full-robust-bisectioning}.

\begin{algorithmbox}[Robust bisection algorithm]\label{algo:robust-bisection}
	\mbox{}\\
	\textbf{Input:} A graph $G$ sampled from the $k$-stochastic block model $\SBM_n(d,\e,k)$ with $\eta n$ corrupted nodes.
	\begin{enumerate}
    \item \textbf{Graph splitting:} We let $G_1$ be the graph obtained by subsampling each edge in $G$ independently with probability $0.99$ and let $G_2\coloneqq G\setminus G_1$.
    \item \textbf{Rough initialization:} Apply the algorithm from \cref{thm:Initialization} on graph $G_1$ to obtain a rough $k$-clustering $Z_{\text{rough}}$ with error rate $0.001 + 10^4 \eta$.
    \item \textbf{Identifying well recovered blocks:} Apply the algorithm from \cref{thm:identify-recovered-blocks} on graph $G_2$ and $Z_{\text{rough}}$ to identify $k/2$ clusters \( S_1, S_2, \dots, S_{k/2} \) in which $0.99$-fraction of the nodes belongs to the same community.
    \item \textbf{Bisection boosting:} Apply the algorithm from \cref{thm:robust-bisectioning} on graph $G$ with \( S_1, S_2, \dots, S_{k/2} \) and $Z_{\text{rough}}$ as input, find a bisection $\hat{x}$ that separates the identified blocks from the remaining blocks with an error rate $\exp\left(-\bigl(1 - o(1)\bigr)\,\frac{C_{d,\varepsilon}}{k^2}\right) + \poly(k) \eta$.
	\end{enumerate}
    \textbf{Output:} $\hat{x}$.
\end{algorithmbox}

\restatetheorem{thm:full-robust-bisectioning}

\begin{proof}
For \cref{thm:robust-bisectioning}, it suffices to establish the guarantees required by the initialization phase. 
Specifically, we must obtain \( k/2 \) disjoint subsets \( S_1, S_2, \dots, S_{k/2} \), each of size exactly \( n/k \),
\footnote{We assume that \( k \) is a power of 2 and \( n \) is divisible by \( k \), 
ensuring that \( n/k \) is an integer and eliminating rounding issues in the partitioning process.} such that, with high probability, 
at least a \( 0.99 \) fraction of the vertices within each subset belong to the same community.

By \cref{thm:Initialization}, we obtain with high probability disjoint subsets \( V_1, V_2, \dots, V_k \subseteq [n] \), such that for at least \( 0.99 \) fraction of these subsets, 
each contains at least \( 0.999 \) fraction of its vertices from one ground-truth community, 
with fewer than \( 0.001 n/k \) vertices from other communities. 
Without loss of generality, let these well-aligned subsets be \( V_1, V_2, \dots, V_{0.99k} \).

To guarantee that each subset has exactly \( n/k \) vertices, we rebalance the subsets by redistributing vertices appropriately. 
Specifically, surplus vertices are transferred from subsets exceeding size \( n/k \) to subsets with fewer than \( n/k \) vertices.
Let the resulting balanced subsets be denoted by \( S_1, S_2, \dots, S_k \).

The key observation here is that, initially, each of the subsets \( V_1, \dots, V_{0.99k} \) contains between \( 0.999 n/k \) and \( 1.001 n/k \) vertices. 
Consequently, the number of vertices added to or removed from each of these subsets during rebalancing is at most \( 0.001 n/k \). 
Thus, each subset retains at least \( 0.998 n/k \) vertices from its dominant community. 
Therefore, for at least \( 0.99 \) fraction of the subsets, a \( 0.998 \) fraction of their vertices still belongs to a single ground-truth community.

Now, we have \( k/2 \) disjoint subsets \( S_1, S_2, \dots, S_{k/2} \), each of size \( n/k \), satisfying the condition that at least \( 0.998 \) fraction of vertices in each subset originate from the same community. 
By \cref{thm:identify-recovered-blocks}, the algorithm will accept these subsets as valid clusters and reject subsets with recovery rates below \( 0.99 \). 
Combining this with \cref{thm:robust-bisectioning} concludes the proof of the theorem.
\end{proof}
\section{Robust initialization for symmetry breaking}\label{sec:sym-break-robust-clustering}

Building on the robust optimal bisection algorithm from \cref{sec:robust-bisection-algo}, 
we present a robust polynomial-time algorithm for clustering vertices into \(k\) communities.
Our approach attains a misclassification error strictly below the critical \(1/k\) symmetry-breaking threshold.

\begin{algorithmbox}[Robust \(k\)-clustering via recursive bisections]
    \label{algo:k-clustering}
    \mbox{}\\
    \textbf{Input:} Graph \(G\) drawn from the \(k\)-stochastic block model \(\SBM_{n_i}(d_i,\epsilon,k_i)\) with corruption fraction at most \(\eta_i\).
    \begin{enumerate}
        \item Apply the robust bisection algorithm (\cref{algo:robust-bisection}) on \(G\) to obtain bisection \(x \in \{\pm 1\}^{n_i}\).
        
        \item Let $x^{+} \seteq \set{i \in [n]: x_i = 1}$ and $x^{-} \seteq \set{i \in [n]: x_i = -1}$. Partition \(G\) into two induced subgraphs $G_1 = G[x^{+}]$ and $G_2 = G[x^{-}]$ such that
        each distrbutes as \(\SBM_{n_{i+1}}\left(d_{i+1}, \epsilon, k_{i+1}\right)\).
        
        \item Recursively apply Algorithm~\ref{algo:k-clustering} to \(G_1\) and \(G_2\) with parameters \(n_{i+1}, d_{i+1}, \epsilon, k_{i+1}\).
        
        \item After \(\lfloor \log k \rfloor\) levels of recursion, merge the resulting clusters into a final community assignment vector \(\hat{Z} \in \set{0, 1}^{n\times k}\).
    \end{enumerate}
    \textbf{Output:} $\hat{Z}$.
\end{algorithmbox}

Now, we restate and prove the algorithmic guarantees of \cref{algo:k-clustering}.

\restatetheorem{thm:boost_to_sym_break}

\begin{proof}
We follow the level-wise parametrization of Appendix~\ref{sec:stat-prop-appendix-1}:
for level $i\in\{1,\dots,\lfloor\log_2 k\rfloor\}$ set
$\beta_i \coloneqq 2^{-i}$, $n_i \coloneqq 2\beta_i n$, and $k_i \coloneqq \beta_i k$.
There are $2^i$ subproblems at level $i$, each on $n_i$ vertices and $k_i$ communities.
Let $\widetilde C_i$ be the level-$i$ bisection SNR and $\mathrm{err}_i$ the (per-vertex)
bisection error of \cref{algo:full-robust-bisectioning} on a level-$i$ subproblem.

The original instance has at most $\eta n$ corrupted vertices.
Restricting to any level-$i$ vertex set $V_i$ of size $n_i$ gives at most $\eta n$ corruptions inside $V_i$, so the
effective corruption fraction satisfies
\[
\eta_i \le \frac{\eta n}{n_i} = \frac{\eta}{2\beta_i} \le \frac{\eta}{\beta_i} \le k\eta,
\]
since $\beta_i \ge 1/k$ for $i\le \log_2 k$.
Applied at level $i$, \cref{thm:full-robust-bisectioning} (with $k_i$ communities and corruption $\eta_i$) gives, for some absolute $C_0>0$,
\[
\mathrm{err}_i \le
\exp\Bigl(-\bigl(1-o(1)\bigr)\frac{\widetilde C_i}{8}\Bigr) + C_0 k_i\eta_i.
\]
By \cref{thm:mixture-binomial-concentration-bound}, $\beta_i\widetilde C_i$ is nondecreasing in $i$ and
$\beta_1\widetilde C_1 = \Theta(C_{d,\varepsilon}/k^2)$, hence
\[
\exp\Bigl(-\bigl(1-o(1)\bigr)\frac{\widetilde C_i}{8}\Bigr) \le
\exp\Bigl(-\bigl(1-o(1)\bigr)\frac{C_{d,\varepsilon}}{k^2}\Bigr) \enspace. 
\]
Moreover $k_i\eta_i \le (\beta_i k)(\eta/\beta_i)=k\eta$, so
\begin{equation}\label{eq:err-i-uniform}
\mathrm{err}_i \le \exp\Bigl(-\bigl(1-o(1)\bigr)\frac{C_{d,\varepsilon}}{k^2}\Bigr) + C_0 k\eta
\quad\text{for all } i.
\end{equation}
At level $i$ there are $2^i$ disjoint subproblems, each with error at most $\mathrm{err}_i$, so the number of newly
misclustered vertices at that level is at most $2^i \mathrm{err}_i n_i$.
Since $n_i = 2\beta_i n = 2^{1-i}n$, we have $2^i n_i = 2n$, and the contribution of level $i$ to the global error is at most
\[
\frac{2^i \mathrm{err}_i n_i}{n} \le 2\mathrm{err}_i.
\]
Summing over $i=1,\dots,\lfloor\log_2 k\rfloor$ and using \eqref{eq:err-i-uniform},
\[
\error_k(\hat Z, Z^\circ) \le
\sum_{i=1}^{\lfloor\log_2 k\rfloor} 2\mathrm{err}_i \le
O(\log k)\bigl(\exp\Bigl(-\bigl(1-o(1)\bigr)\frac{C_{d,\varepsilon}}{k^2}\Bigr) + C_0 k\eta\bigr).
\]
Absorbing the $O(\log k)$ and $C_0$ factors into $\poly(k)$, we obtain
\[
\error_k(\hat Z, Z^\circ) \le
\exp\Bigl(-\bigl(1-o(1)\bigr)\frac{C_{d,\varepsilon}}{k^2}\Bigr) + \poly(k)\,\eta,
\]
as claimed.
\end{proof}

As a direct consequence of \cref{thm:boost_to_sym_break}, we obtain the following corollary, capturing the symmetry-breaking improvement:

\begin{corollary}\label{cor:sym-break}
If \( d\epsilon^2 \geq K k^2 \log k \) for some sufficiently large constant \(K\), then \cref{algo:k-clustering} achieves, with probability at least $1 - \exp(-\Omega(k)) - \tfrac{1}{\poly(n)}$, a labeling error at most \(O\left(\frac{1}{\poly(k)}\right)\).
\end{corollary}
\section{Robust optimal recovery algorithm}\label{sec:Optimal-robust-clustering}

In this section, we build on the guarantees from \cref{sec:sym-break-robust-clustering} and give a robust polynomial-time algorithm that partitions the vertices into \(k\) communities at the minimax-optimal rate. The algorithm is a constant-degree sum-of-squares (SoS) relaxation of polynomial constraints that encode pairwise majority voting. These constraints closely mirror those used for robust bisectioning. The key difference is the symmetry-breaking guarantee from \cref{cor:sym-break}, which ensures that across all feasible solutions the community labels are aligned. This alignment allows us to analyze the majority voting guarantees for each pair of communities.

\subsection{Algorithm and constraint systems}

We first introduce the polynomial constraints that are used in our robust optimal recovery algorithm.

\smallskip
\noindent\textbf{(i) Labeling constraints.} The feasible label variable \(Z \in \{0,1\}^{n\times k}\) satisfies
\begin{equation}
    \cAlabel(Z)
    \seteq \Set{ Z(i,a)^{2}=Z(i,a), \quad \sum_{a}Z(i,a)=1, \quad \sum_{i} Z(i,a)=\frac n k } \,.
\end{equation}

\smallskip
\noindent\textbf{(ii) Initialization (symmetry breaking).} Let \(Z_{\mathrm{init}}\in \R^{n\times k}\) be an initializer with \(\mu\)-approximation error. We constrain the SoS variable \(Z\) to remain close to \(Z_{\mathrm{init}}\):
\begin{equation}
    \label{eq:label-assignment-constraint}
    \cAsymbreak(Z;Z_{\mathrm{init}}, \mu) \seteq 
    \Set{\Normf{Z - Z_{\mathrm{init}}}^2 \leq \mu} \,.
\end{equation}

\smallskip
\noindent\textbf{(iii) Corruption mask / node distance.} We model node corruptions through a binary mask \(\xi\in\{0,1\}^n\) and require the program matrix to agree with the observed graph outside the corrupted rows/columns:
\begin{equation}
    \label{eq:node-distance-constraint}
    \cAclose(Y,\xi;\bar{G},\eta) \seteq \Set{(Y-\bar{G})\odot (\one-\xi)(\one-\xi)^\top=0, \ \xi\odot \xi=\xi,\ \sum_i \xi_i \leq \delta_{\eta} n} \,,
\end{equation}
where \(\delta_{\eta} = \exp(-2 C_{d,\varepsilon}) + \eta\).

\smallskip
\noindent\textbf{(iv) Community structure and spectral condition.} 
We encode the decomposition of the centered adjacency into pairwise components and bound the spectral norm of the noise blocks:
\begin{equation}
    \label{eq:opt-mixing-constraint}
    \cAmix(Y,Z)
    \seteq \Set{X=ZZ^\top-\frac{1}{k} J, \Normop{Y - \frac{\e d}{n} X} \leq \Paren{\chi + \frac{1}{k}}\sqrt{d}} \,.
\end{equation}

\smallskip
\noindent\textbf{(v) Majority-vote consistency.}
For each pair \(a,b\in[k]\) define $x_{ab} = Z(\cdot,a)-Z(\cdot,b)$.
We require that, the pairwise labels are satisfy majority voting constraint.
\begin{equation}
  \label{eq:full-majority-vote-constraints}
  \cAmaj(Y, Z, R; \gamma, \beta) \seteq
  \Set{
    \begin{aligned}
        & R \colon \cAset(z) \sststile{2}{u} \Iprod{Y x_{ab}, z \odot x_{ab}} 
        \ge \frac{(1 - \gamma) \e d}{16 k} \Big(\Norm{z \odot x_{ab}}^2 - 640\beta n\Big) \\
        &x_{ab} = Z(\cdot,a)-Z(\cdot,b)
    \end{aligned}
  } \,,
\end{equation}
where \(\cAset(z) \seteq \set{z \odot z = z}\).
We will use two parameter settings, corresponding to a two-round majority vote that achieves the minimax rate \footnote{Two rounds are needed to sharpen the error from a constant-factor exponent to the optimal exponent.}
\begin{itemize}
    \item $\mu_1 = \frac{n}{\poly(k)}$, $\gamma_1 = 0.99$, $\beta_1= 1000 k \exp\!\left(-0.99 \frac{C_{d, \e}}{k} \right)$.
    \item $\mu_{\text{opt}} = \exp \Paren{-0.99 (1 - o(1)) \frac{C_{d,\e}}{k}}$, $\gamma_\text{opt} = 1 - \frac{10 \chi k}{\sqrt{C_{d, \e}}}$, $\beta_\text{opt} = \frac{\sqrt{C_{d, \e}}}{10 \chi} \exp\!\left(- \left(1 - \frac{10 \chi k}{\sqrt{C_{d, \e}}}\right) \frac{C_{d, \e}}{k} \right)$.
\end{itemize}
For our algorithm, we define polynomial system \(\cOpt(Y,Z,\xi,R;\bar{G}, Z_{\mathrm{init}}, \eta, \mu, \gamma, \beta)\) based on the set of constraints defined above \footnote{For simplicity, we use $\cOpt(Y,Z,\xi)$ to denote $\cOpt(Y,Z,\xi,R;\bar{G}, Z_{\mathrm{init}}, \eta,\gamma, \beta)$ when the input is clear from the context.}
\begin{align}
    \label{eq:opt-recovery-poly-system}
    \begin{split}
    & \cOpt(Y,Z,\xi,R;\bar{G}, Z_{\mathrm{init}}, \eta, \mu, \gamma, \beta) \\
    & \seteq \cAlabel(Z) \cup \cAsymbreak(Z;Z_{\mathrm{init}},\mu) \cup \cAclose(Y,\xi;\bar{G},\eta) \cup \cAmix(Y, Z) \cup \cAmaj(Y, Z, R; \gamma, \beta) \,.
    \end{split}
\end{align}

\paragraph{Algorithm.} The complete robust optimal recovery algorithm for $k$-SBM is described in \cref{algo:final-boost}.

\begin{algorithmbox}[Robust recovery algorithm]\label{algo:final-boost}
    \mbox{}\\
    \textbf{Input:} The centered adjacency matrix $\bar{G}$ drawn from \(\SBM_n(d,\e,k)\) with at most \(\eta n\) corrupted nodes.
    \begin{enumerate}[1.]
        \item Run \cref{algo:k-clustering} to obtain an initializer \(\hat{Z}_\text{init}\).
        \item Find degree-$O(1)$ pseudo-expectation $\tildeE$ satisfying $\cOpt(Y,Z,\xi,R;\bar{G}, Z_{\mathrm{init}}, \eta,\gamma, \beta)$ with \(\hat{Z}_\text{init}\), $\gamma = 0.99$, $\mu = \frac{n}{\poly(k)}$, $\beta= 1000 k \exp\!\left(-0.99 \frac{C_{d, \e}}{k} \right)$.
        \item Round $\tildeE$ by labelling vertex $i$ with label $\argmax_a \tildeE[Z(i,a)]$ to \(\hat{Z}\).
        \item Find degree-$O(1)$ pseudo-expectation $\tildeE_{\text{opt}}$ satisfying $\cOpt(Y,Z,\xi,R;\bar{G}, Z_{\mathrm{init}}, \eta,\gamma, \beta)$ with \(\hat{Z}\), $\mu = \exp \Paren{-0.99 (1 - o(1)) \frac{C_{d,\e}}{k}}$, $\gamma = 1 - \frac{10 \chi k}{\sqrt{C_{d, \e}}}$, $\beta = \frac{\sqrt{C_{d, \e}}}{10 \chi} \exp\!\left(- \left(1 - \frac{10 \chi k}{\sqrt{C_{d, \e}}}\right) \frac{C_{d, \e}}{k} \right)$.
        \item Round $\tildeE_{\text{opt}}$ by labelling vertex $i$ with label $\argmax_a \tildeE_{\text{opt}}[Z(i,a)]$ to \(Z_\text{opt}\).
    \end{enumerate}
    \textbf{Output:} $Z_\text{opt}$.
\end{algorithmbox}

\subsection{Feasibility and time complexity}

We first establish feasibility of $\cOpt(Y,Z,\xi)$ and time complexity of \cref{algo:final-boost}.

\begin{lemma}\label{lem:k-boosting_feasible}
    Let $\Gbarnull$ be the centered adjacency matrix of the uncorrupted graph and $S$ be the set of nodes with degree larger than $20d$.
    Let $T$ be the set of corrupted nodes.
    Let $\Znull$ be the ground truth label matrix.
    Under the conditions of \cref{thm:formal_main_result}, program $\cOpt(Y,Z,\xi)$ is satisfied by $(\Gbarnull \odot (\one_{\bar{S}} \one_{\bar{S}}^{\top}), \Znull, \one_{S \cup T})$ with probability at least \(1-\exp(-100k)-1/n^3\).
\end{lemma}

\begin{proof}
    The proof parallels \cref{lem:bisectioning_feasible}. For \(\cAmix\), feasibility follows from \cref{cor:inner-product-lb-k-clustering}. 
    The mixing constraints $\Normop{Y - \frac{\e d}{n} X} \leq \Paren{\chi + \frac{1}{k}}\sqrt{d}$ follow from applying \cref{cor:spectral-remove-cor} and similar analysis as \cref{lem:bisectioning_feasible}.
    The feasibility of initialization constraint follows from \cref{thm:boost_to_sym_break}.
    The feasibility of corruption constraint follows from \cref{cor:spectral-remove-cor}.
\end{proof}

\begin{lemma}\label{lem:k-boosting-algorithm-efficiency}
    \cref{algo:final-boost} runs in polynomial time.
\end{lemma}
\begin{proof}
    The SoS relaxation has constant degree and therefore produces an SDP with polynomially many variables and constraints, which can be solved in polynomial time.
\end{proof}

\subsection{SoS guarantees for robust optimal boosting}

We now prove $\cOpt(Y,Z,\xi)$ boosts the accuracy to the desired rate.
To do this, we consider two feasible solutions $(Y^{(1)},Z^{(1)},\xi^{(1)})$ and $(Y^{(2)},Z^{(2)},\xi^{(2)})$ of $\cOpt(Y,Z,\xi)$.
We will show that the constraints in $\cOpt(Y,Z,\xi)$ allow us to prove that any two good solutions will be within minimax error rate to each other.
Since $Z^{\circ}$ is also a good solution, this implies that any feasible solution of $\cOpt(Y,Z,\xi)$ achieves the minimax error rate.

Throughout this section, we define the following terms:
\begin{itemize}
    \item Let $l_{ab}^{(t)} = x_{ab}^{(t)} \odot x_{ab}^{(t)}$ be the support of $x_{ab}^{(t)}$ for $t \in \set{1, 2}$. Notice that, equivalently, $l_{ab}^{(t)} = Z^{(t)}(\cdot, a) + Z^{(t)}(\cdot, b)$.
    \item Let $l_{ab} = l_{ab}^{(1)} \odot l_{ab}^{(2)}$ be the intersection of $l_{ab}^{(1)}$ and $l_{ab}^{(2)}$.
    \item Let $v_{ab}=\frac{l_{ab}^{(1)} \odot l_{ab}^{(2)}-x_{ab}^{(1)}\odot x_{ab}^{(2)}}{2}$ be the coordinates where two feasible pairwise labels \(x_{ab}^{(1)}\) and \(x_{ab}^{(2)}\) differ.
    \item Let $s=\one - (\one-\xi^{(1)}) \odot (\one-\xi^{(2)})$ be the set of nodes where $Y^{(1)}$ and $Y^{(2)}$ differ.
    \item Let $w_{ab} = v_{ab} \odot (\one-s)$ be the set of vertices in the shared part of $Y^{(1)}$ and $Y^{(2)}$ such that the bisections $x^{(1)}$ and $x^{(2)}$ differ.
    \item Let $g_{ab} = \one - (\one-v_{ab}) \odot (\one-s)$ be union of $v_{ab}$ and $s$, notice that, equivalently $g_{ab}=w_{ab}+s$.
\end{itemize}

To prove the algorithmic guarantee of $\cOpt(Y,Z,\xi)$, we need the following observations on properties of the variables defined above.

\begin{lemma}
\label{clm:opt_bisection_equivalences}
    For any \(a,b\in [k]\), $x^{(1)}_{ab}$, $x^{(2)}_{ab}$ and $v_{ab}$ satisfies
    \begin{equation*}
        x^{(1)}_{ab} \odot v_{ab} = - x^{(2)}_{ab} \odot v_{ab}
       \quad \text{and} \quad
       x^{(1)}_{ab} \odot (l_{ab}-v_{ab}) = x^{(2)}_{ab} \odot (l_{ab}-v_{ab}) \,,
    \end{equation*}
    and,
    \begin{equation*}
       v_{ab} \odot \frac{l_{ab}-x^{(1)}_{ab}}{2} = v_{ab} \odot \frac{l_{ab}+x^{(2)}_{ab}}{2}
       \quad \text{and} \quad
       v_{ab} \odot \frac{l_{ab}-x^{(2)}_{ab}}{2} = v_{ab} \odot \frac{l_{ab}+x^{(1)}_{ab}}{2}\,.
    \end{equation*}
\end{lemma}

\begin{proof}
    The proof follows by similar analysis as \cref{clm:bisection_equivalences}.
\end{proof}

\begin{lemma}
\label{lem:opt_agreement_lemma}
    The norm of $v_{ab}$, $s$ and $g_{ab}$ satisfies the following SoS inequalities
    \begin{equation*}
        \cOpt\Paren{Y^{(1)},Z^{(1)},\xi^{(1)}},\cOpt\Paren{Y^{(2)},Z^{(2)},\xi^{(2)}}
        \sststile{O(1)}{\xi^{(1)},\xi^{(2)}} \Norm{s}^2 \leq 2 \delta_{\eta} n \,,
    \end{equation*}
    and,
    \begin{equation*}
        \cOpt\Paren{Y^{(1)},Z^{(1)},\xi^{(1)}},\cOpt\Paren{Y^{(2)},Z^{(2)},\xi^{(2)}}
        \sststile{O(1)}{Z^{(1)},Z^{(2)}} \Norm{v_{ab}}^2 \leq \mu \,,
    \end{equation*}
    and,
    \begin{equation*}
        \cOpt\Paren{Y^{(1)},Z^{(1)},\xi^{(1)}},\cOpt\Paren{Y^{(2)},Z^{(2)},\xi^{(2)}}
        \sststile{O(1)}{Z^{(1)},Z^{(2)},\xi^{(1)},\xi^{(2)}} \Norm{w_{ab}}^2 \leq \mu \,,
    \end{equation*}
    and,
    \begin{equation*}
        \cOpt\Paren{Y^{(1)},Z^{(1)},\xi^{(1)}},\cOpt\Paren{Y^{(2)},Z^{(2)},\xi^{(2)}}
        \sststile{O(1)}{Z^{(1)},Z^{(2)},\xi^{(1)},\xi^{(2)}} \Norm{g_{ab}}^2 \leq \mu + 4 \delta_{\eta} n \,.
    \end{equation*}
\end{lemma}

\begin{proof}
    For the bound on $s$, by constraints in $\cAclose(Y,\xi;\eta,\bar{G})$, it follows that
    \begin{align*}
        \Norm{s}^2
        & =\sum_i s_i
        = \sum_i 1 - (1-\xi^{(1)}_i) \odot (1-\xi^{(2)}_i)
        = \sum_i \xi^{(1)}_i + \sum_i \xi^{(2)}_i (1-\xi^{(1)}_i) \\
        & \leq \sum_i \xi^{(1)}_i + \sum_i \xi^{(2)}_i
        \leq 2 \delta_{\eta} n\,.
    \end{align*}
    For the bound on $v_{ab}$, by constraints in $\cAinit(Z;Z_{\mathrm{init}})$, it follows that
    \begin{align*}
        \Norm{v_{ab}}^2
        = & \Norm{\frac{l_{ab}^{(1)} \odot l_{ab}^{(2)}-x_{ab}^{(1)}\odot x_{ab}^{(2)}}{2}}^2
        = \Norm{\frac{x_{ab}^{(1)} \odot l_{ab}^{(2)} - x_{ab}^{(2)} \odot l_{ab}^{(1)}}{2}}^2 \\
        \leq & \Norm{\frac{x_{ab}^{(1)}  - x_{ab}^{(2)}}{2}}^2
        \leq \frac{1}{2} \Norm{Z^{(1)}(\cdot,a) - Z^{(2)}(\cdot,a)}^2 + \frac{1}{2} \Norm{Z^{(1)}(\cdot,b) - Z^{(2)}(\cdot,b)}^2 \\
        \leq & \Norm{Z^{(1)}(\cdot,a) - Z_{\mathrm{init}}(\cdot,a)}^2 + \Norm{Z^{(2)}(\cdot,a) - Z_{\mathrm{init}}(\cdot,a)}^2  \\
        & + \Norm{Z^{(1)}(\cdot,b) - Z_{\mathrm{init}}(\cdot,b)}^2 + \Norm{Z^{(2)}(\cdot,b) - Z_{\mathrm{init}}(\cdot,b)}^2 \\
        \leq & \Normf{Z - Z_{\mathrm{init}}}^2 \leq \mu \,.
    \end{align*}
    The bound on $\Norm{w_{ab}}^2$ follows by
    \begin{equation*}
        \Norm{w_{ab}}^2 = \Norm{v_{ab} \odot (\one-s)}^2 \leq \Norm{v_{ab}}^2 \leq \mu \,.
    \end{equation*}
    For the bound on $\Norm{g_{ab}}^2$, since $g_{ab} = w_{ab} + s$,
    \begin{equation*}
        \Norm{g_{ab}}^2
        = \Norm{w_{ab} + s}^2
        \leq 2 \Norm{w_{ab}}^2 + 2 \Norm{s}^2
        \leq \mu + 4 \delta_{\eta} n \,.
    \end{equation*}
\end{proof}

\begin{fact}\label{fact:k-key-equality}
    Fix \(a,b\in [k]\), there is a constant-degree SoS proof of
    \[
        \Iprod{Y^{(1)}\, (x_{ab}^{(1)}\odot (\one -g_{ab}))  , x_{ab}^{(1)}\odot w_{ab}} \;=\; -\Iprod{Y^{(2)} \,(x_{ab}^{(2)}\odot (\one -g_{ab})) , x_{ab}^{(2)}\odot w_{ab}}\, .
    \]
\end{fact}

\begin{proof}
Plugging in the definition of $g_{ab}$ and $w_{ab}$, the left hand side can be written as
\begin{align*}
    \Iprod{Y^{(1)} (x^{(1)}_{ab} \odot (\one-g_{ab})), x^{(1)}_{ab} \odot w_{ab}}
    & = \Iprod{Y^{(1)} (x^{(1)}_{ab} \odot (\one-v_{ab}) \odot (\one-s)), x^{(1)}_{ab} \odot v_{ab} \odot (\one-s)} \\
    & = \Iprod{\Paren{Y^{(1)} \odot (\one-s) (\one-s)^{\top}} \Paren{x^{(1)}_{ab} \odot (\one-v_{ab})}, x^{(1)}_{ab} \odot v_{ab}} \,.
\end{align*}
By $\cAclose(Y,\xi;\eta,\bar{G})$, we have
\begin{equation*}
    Y^{(1)}\odot (\one-s)(\one-s)^\top = \bar{G} \odot (\one-s)(\one-s)^\top = Y^{(2)}\odot (\one-s)(\one-s)^\top \,,
\end{equation*}
and, by \cref{clm:opt_bisection_equivalences}, we have
    \begin{equation*}
        x^{(1)}_{ab} \odot v_{ab} = - x^{(2)}_{ab} \odot v_{ab}
       \quad \text{and} \quad
       x^{(1)}_{ab} \odot (l_{ab}-v_{ab}) = x^{(2)}_{ab} \odot (l_{ab}-v_{ab}) \,,
    \end{equation*}
Thus,
\begin{align*}
    \Iprod{Y^{(1)} (x^{(1)}_{ab} \odot (\one-g_{ab})), x^{(1)}_{ab} \odot w_{ab}}
    & = - \Iprod{\Paren{Y^{(2)} \odot (\one-s) (\one-s)^{\top}} (x^{(2)}_{ab} \odot (\one-v_{ab})), x^{(2)}_{ab} \odot v_{ab}} \\
    & = - \Iprod{Y^{(2)} (x^{(2)}_{ab} \odot (\one-g_{ab})), x^{(2)}_{ab} \odot w_{ab}} \,.
\end{align*}
\end{proof}

We next upper bound the contribution of the signal on the set \(g\).

\begin{lemma}\label{lem:k-sos-signal-upper-bound}
    In the setting of \cref{thm:formal_main_result}, for any \(\delta\in [0,0.001]\) and $t \in \set{1, 2}$,
    \begin{align*}
        & \cOpt\Paren{Y^{(1)},Z^{(1)},\xi^{(1)}},\cOpt\Paren{Y^{(2)},Z^{(2)},\xi^{(2)}} \\
        & \sststile{6}{Y^{(1)},Z^{(1)},\xi^{(1)},Y^{(2)},Z^{(2)},\xi^{(2)}}
        \Iprod{Y^{(t)} \Paren{x^{(t)}_{ab} \odot g_{ab}},x^{(t)}_{ab} \odot w_{ab}} \leq \Paren{\frac{\e d}{\poly(k)} + 4 \e d \delta_{\eta}} \Norm{w_{ab}}^2 + \frac{\e d \delta_{\eta}}{k} n \,.
    \end{align*}
\end{lemma}

\begin{proof}
    For notational simplicity, we ignore the superscript $(t)$ throughout the proof. We decompose the left hand side in the following way
    \begin{equation}
        \label{eq:opt-sos-signal-upper1}
        \Iprod{Y \Paren{x_{ab} \odot g},x_{ab} \odot w_{ab}}
        = \frac{\e d}{n}\Iprod{X \Paren{x_{ab} \odot g_{ab}}, x_{ab} \odot w_{ab}}
        + \Iprod{(Y-\frac{\e d}{n} X) \Paren{x_{ab} \odot g_{ab}}, x_{ab} \odot w_{ab}} \,.
    \end{equation}
    For the first term, since $X = Z Z^{\top} - \frac{1}{k} J$, we have $- J \leq X \leq J$ such that
    \begin{equation*}
        \frac{\e d}{n}\Iprod{X \Paren{x_{ab} \odot g_{ab}},x_{ab} \odot w_{ab}}
        \leq \frac{\e d}{n} \Paren{\sum_i (g_{ab})_i} \Paren{\sum_i (w_{ab})_i}
        = \frac{\e d}{n} \Norm{g_{ab}}^2 \Norm{w_{ab}}^2 \,.
    \end{equation*}
    Since $\Norm{g_{ab}}^2 \leq \mu + 4 \delta_{\eta} n$, it follows that
    \begin{equation}
        \label{eq:opt-sos-signal-upper2}
        \frac{\e d}{n}\Iprod{X \Paren{x_{ab} \odot g_{ab}},x_{ab} \odot w_{ab}}
        \leq \Paren{\frac{\e d \mu}{n} + 4 \e d \delta_{\eta}} \Norm{w_{ab}}^2 \,.
    \end{equation}
    For the second term, by $\cAmix(Y,Z)$ and SoS Cauchy-Schwarz, we also have
    \begin{align*}
        \Iprod{(Y-\frac{\e d}{n} X) \Paren{x_{ab} \odot g_{ab}},x_{ab} \odot w_{ab}}
        & \leq \frac{1}{2} \Paren{\chi + \frac{1}{k}}\sqrt{d} \Paren{\Norm{x_{ab} \odot g_{ab}}^2 + \Norm{x_{ab} \odot w_{ab}}^2} \\
        & \leq \frac{1}{2} \Paren{\chi + \frac{1}{k}}\sqrt{d} \Paren{3\Norm{w_{ab}}^2 + 2\Norm{s}^2} \\
        & \leq \frac{1}{2} \Paren{\chi + \frac{1}{k}}\sqrt{d} \Paren{3\Norm{w_{ab}}^2 + 2\delta_{\eta} n} \,.
    \end{align*}
    Since $\sqrt{d} \ll \frac{\e d}{k}$, it follows that
    \begin{equation}
        \label{eq:opt-sos-signal-upper3}
        \Iprod{\Paren{Y-\frac{\e d}{n} X} \Paren{x_{ab} \odot g_{ab}},x_{ab} \odot w_{ab}}
        \leq \frac{\e d}{k} \Norm{w_{ab}}^2 + \frac{\e d \delta_{\eta}}{k} n \,.
    \end{equation}
    Plugging \cref{eq:opt-sos-signal-upper2} and \cref{eq:opt-sos-signal-upper3} into \cref{eq:opt-sos-signal-upper1}, it follows that
    \begin{equation*}
        \Iprod{Y \Paren{x_{ab} \odot g},x_{ab} \odot w_{ab}}
        \leq \Paren{\frac{\e d \mu}{n} + 4 \e d \delta_{\eta}} \Norm{w_{ab}}^2 + \frac{\e d \delta_{\eta}}{k} n \,.
    \end{equation*}
\end{proof}

Now, we are ready to prove the algorithmic guarantee of $\cOpt$.

\begin{theorem}\label{thm:k-Identifiability-proof}
In the setting of \cref{thm:formal_main_result},
\begin{align*}
        &\cOpt\Paren{Y^{(1)},Z^{(1)},\xi^{(1)}},\cOpt\Paren{Y^{(2)},Z^{(2)},\xi^{(2)}} \\
        &{\sststile{6}{Y^{(1)},Z^{(1)},\xi^{(1)},Y^{(2)},Z^{(2)},\xi^{(2)}}}
        \sum_{a,b} \norm{v_{ab}}^2
        \leq O \Paren{\frac{(1-\gamma) k^2 \beta n + k^2 \eta n}{1-\gamma - \mu k / n - \eta k} + k^2 \mu n} \,.
\end{align*}
\end{theorem}

\begin{proof}
    For any pairs of communities \(a,b\in [k]\), by \cref{fact:k-key-equality},
    \[
        \Iprod{Y^{(1)}\, (x_{ab}^{(1)}\odot (\one -g_{ab}))  , x_{ab}^{(1)}\odot w_{ab}} \;=\; -\Iprod{Y^{(2)} \,(x_{ab}^{(2)}\odot (\one -g_{ab})) , x_{ab}^{(2)}\odot w_{ab}}\, .
    \]
    Therefore,
    \begin{equation}\label{eq:key-equality-final}
        \Iprod{Y^{(1)} x_{ab}^{(1)}  , x_{ab}^{(1)} \odot w_{ab}}+ \Iprod{Y^{(2)} x_{ab}^{(2)}  , x_{ab}^{(2)} \odot w_{ab}}
        = \Iprod{Y^{(1)} \!\big(x_{ab}^{(1)}\odot g\big),x_{ab}^{(1)}\odot w_{ab}} + \Iprod{Y^{(2)} \!\big(x_{ab}^{(2)}\odot g\big),x_{ab}^{(2)}\odot w_{ab}}.
     \end{equation}
    By $\cAmaj(Y, Z, R; \gamma, \beta)$, it follows that
    \begin{equation}
        \label{eq:final-boosting-lower-bound}
        \Iprod{Y^{(t)} x^{(t)}_{ab}, x^{(t)}_{ab} \odot w_{ab}} 
        \ge \frac{(1 - \gamma) \e d}{16 k} \Big(\Norm{\odot w_{ab}}^2 - 640\beta n\Big) \,,
    \end{equation}
    for $t \in \set{1, 2}$.
    By \cref{lem:k-sos-signal-upper-bound}, we have
    \begin{equation}\label{eq:final-boosting-equation}
         \Iprod{Y^{(t)} \Paren{x^{(t)}_{ab} \odot g_{ab}},x^{(t)}_{ab} \odot w_{ab}}
         \leq \Paren{\frac{\e d \mu}{n} + 4 \e d \delta_{\eta}} \Norm{w_{ab}}^2 + \frac{\e d \delta_{\eta}}{k} n \,.
    \end{equation}
    for $t \in \set{1,2}$.

    Therefore, plugging \cref{eq:final-boosting-lower-bound} and \cref{eq:final-boosting-equation} into \cref{eq:key-equality-final} and rearranging terms, we can obtain
    \begin{equation*}
        \Norm{w_{ab}}^2 \leq \frac{(1-\gamma) \beta n + \eta n}{1-\gamma - \mu k / n - \eta k} \,.
    \end{equation*}
    Since $\Norm{v_{ab}}^2 \leq 2\Norm{w_{ab}}^2 + 2\Norm{s}^2$, it follows that
    \begin{equation*}
        \sum_{a,b} \norm{v_{ab}}^2
        \leq O \Paren{\frac{(1-\gamma) k^2 \beta n + k^2 \eta n}{1-\gamma - \mu k / n - \eta k} + k^2 \mu n} \,.
    \end{equation*}
\end{proof}

\subsection{Algorithmic guarantees for robust optimal recovery}

In this section, we prove our main theorem and establish the algorithmic guarantees of our robust optimal recovery algorithm for $k$-SBM.

\restatetheorem{thm:formal_main_result}

\begin{proof}
    Let \(x^\circ_{ab}\) be the the ground-truth pairwise labels.
    Notice that \cref{algo:k-clustering} provides a polynomial-time computable initializer \(\hat{Z}_{\text{input}}\) with error at most \(n/\poly(k)\) that we can use in step 2 of \cref{algo:final-boost}.
    Plugging in $Z_{\text{input}} = \hat{Z}_{\text{input}}$, $\mu = \frac{n}{\poly(k)}$, $\gamma = 0.99$ and $\beta= 1000 k \exp\!\left(-0.99 \frac{C_{d, \e}}{k} \right)$ in \cref{thm:k-Identifiability-proof} and rearranging terms, we can obtain $\tildeE$ such that
    \begin{equation*}
        \sum_{a,b} \norm{\tilde{\E} x^\circ_{ab}-x_{ab}}_2^2 \leq \exp\Paren{-(1 - o(1))\frac{0.99 C_{d,\e}}{k}}n + \poly(k) \eta n \,.
    \end{equation*}
    Since the rounding introduces a constant multiplicative error, we can use the rounded solution $\hat{Z}$ in step 3. 
    Repeating \cref{thm:k-Identifiability-proof} with $Z_{\text{input}} = \hat{Z}$, $\mu = \exp \Paren{-0.99 (1 - o(1)) \frac{C_{d,\e}}{k}}$, $\gamma = 1 - \frac{10 \chi k}{\sqrt{C_{d, \e}}}$, $\beta = \frac{\sqrt{C_{d, \e}}}{10 \chi} \exp\!\left(- \left(1 - \frac{10 \chi k}{\sqrt{C_{d, \e}}}\right) \frac{C_{d, \e}}{k} \right)$, we can obtain $\tildeE_{\text{opt}}$ such that
   \begin{equation*}
        \sum_{a,b} \norm{\tildeE_{\text{opt}} x^\circ_{ab}-x_{ab}}_2^2 \leq \exp\Paren{-(1 - o(1))\frac{C_{d,\e}}{k}}n + \poly(k) \eta n \,.
    \end{equation*}
    Combine this with the rounding guarantees, we obtain the desired bound.
\end{proof}


\addcontentsline{toc}{section}{References}
\bibliographystyle{amsalpha}
\bibliography{bib/mathreview,bib/dblp,bib/custom,bib/scholar}

\appendix

\section{Sum-of-Squares proofs}\label{sec:basic-sos-proofs}
In this section, we introduce some Sum-of-Squares (SoS) results that are used in our proofs. 

\subsection{Basic Sum-of-Squares inequalities}
We start with a Cauchy-Schwarz inequality for pseudo-distributions.
\begin{fact}[Cauchy-Schwarz for pseudo-distributions ~\cite{barak2012hypercontractivity}]\label{fact:pd-cauchy-schwarz}
	Let $f,g$ be vector polynomials of degree at most $d$ in indeterminate $x\in \R^n$. Then, for any level $2d$ pseudo-distribution $D$,
	\begin{align*}
		\tilde{\E}_D\Brac{\iprod{f,g}}\leq \sqrt{\tilde{\E}_D\brac{\norm{f}^2} }\cdot \sqrt{\tilde{\E}_D\brac{\norm{g}^2} }\,.
	\end{align*}
\end{fact}
We will also repeatedly use the following two SoS versions of Cauchy-Schwarz inequality:
\begin{fact}[Sum-of-Squares Cauchy-Schwarz I]
	\label{fact:sos-cauchy-schwarz}
	Let $x,y \in \R^d$ be indeterminates. Then,
	\[ 
	\sststile{4}{x,y} \Iprod{x, y}^2 \leq \Paren{\sum_i x_i^2} \Paren{\sum_i y_i^2} \,.
	\]
\end{fact}

\begin{fact}[Sum-of-Squares Cauchy-Schwarz II]
	\label{fact:simple-sos-version-of-cauchy-schwarz}
	Let $x,y \in \R^d$ be indeterminates. Then, for any $C > 0$,
	\[ 
		\sststile{4}{x,y} \Iprod{x, y} \leq \frac{C}{2} \Norm{x}^2 + \frac{1}{2C} \Norm{y}^2 \,,
	\]
	and,
	\[ 
		\sststile{4}{x,y} \Iprod{x, y} \geq - \frac{C}{2} \Norm{x}^2 - \frac{1}{2C} \Norm{y}^2 \,,
	\]
\end{fact}

We will use the following fact that shows how spectral certificates are captured within the SoS proof system.
\begin{fact}[Spectral Certificates] \label{fact:spectral-certificates}
	For any $m \times m$ matrix $A$, 
	\[
	\sststile{2}{u} \iprod{u,Au} \leq \Norm{A} \Norm{u}_2^2 \,.
	\]
\end{fact}

Combining these facts, we have the following lemma as a corollary:
\begin{lemma}\label{lem:sos-spectral-bound}
    For any constant $L\geq 0$, for variables $E\in \R^{n\times n},v\in \R^n,w\in \R^n$, we have constant degree SoS proof that
    \begin{equation*}
        L\Id-EE^\top=CC^\top \sststile{4}{E,C,v,w} \iprod{Ev,w}^2 \leq L \norm{v}^2 \norm{w}^2\,.  
    \end{equation*}
\end{lemma}
\begin{proof}
    We have 
    \begin{align*}
        \sststile{4}{E,C,v,w}\ & \iprod{Ev,w}^2 \leq \norm{v}^2 \norm{E^\top w}^2\\
            & =\iprod{w,EE^\top w} \cdot \norm{v}^2\,.
    \end{align*}
    Since $L\Id-EE^\top=CC^\top$, we have $\iprod{w,EE^\top w}\le L\norm{w}^2$. As a result,
    \begin{equation*}
        L\Id-EE^\top=CC^\top \sststile{4}{E,C,v,w} \iprod{Ev,w}^2 \leq L \norm{v}^2 \norm{w}^2\,.
    \end{equation*}
\end{proof}

\subsection{Sum-of-Squares certificate}
\begin{theorem}\label{thm:sos-as-constraints}
	Let $x, y \in \R^d$ and $p_1(x),p_2(x),\ldots,p_m(x)$ be a set of polynomials $\R^d\to \R$.
	If there exists $y^{\circ}$ such that $p_1(x)\geq 0,p_2(x)\geq 0,\ldots,p_m(x)\geq 0\sststile{\ell}{x} q(x;y^{\circ})\geq 0$,
	then the degree-$(\ell+2)$ SoS relaxation of the following program is feasible
	\begin{equation*}
		\mathcal{A} := \Bigset{p_1(x)\geq 0,p_2(x)\geq 0,\ldots,p_m(x)\geq 0, q(x,y) = \Tr\Paren{RR^\top \Mom_d(p_1,p_2,\ldots,p_m)}} \,,
	\end{equation*}
	where \(\Mom_{d}(p_{1},\ldots,p_{m})\) is a block-diagonal matrix with entries in \(\R[x]\) and (possibly empty) blocks \(M_{S}\) indexed by subsets \(S\subseteq [m]\) such that
	for all pairs of monomials \(x^{\alpha},x^{\beta}\) with \(\deg x^\alpha,\deg x^{\beta}\le \tfrac 1 2 \cdot (\ell-\sum_{i\in S} \deg p_{i})\),
	\begin{equation*}
	  M_{S}(\alpha,\beta)\seteq
	  \Paren{ \prod\nolimits_{i\in S} p_{i}}\cdot x^{\alpha}\cdot x^{\beta}\,.
	\end{equation*}
	We also have the following SoS proof
	\begin{equation*}
		\mathcal{A} \sststile{\ell+2}{x,y,R} q(x,y) \geq 0 \,.
	\end{equation*}
\end{theorem}

\begin{proof}
	Let $R^{\circ} \in \R^{d^{\ell}\times d^{\ell}}$ be the matrix corresponding to a degree-$\ell$ SoS proof of $p_1(x)\geq 0,p_2(x)\geq 0,\ldots,p_m(x)\geq 0\sststile{\ell}{x} q(x;y^{\circ})\geq 0$, i.e.
	\begin{equation*}
		q(x;y^{\circ}) =\Tr\Paren{R^{\circ} (R^{\circ})^\top \Mom_d(p_1,p_2,\ldots,p_m)}\,,
	\end{equation*}
	Now, the degree-$(\ell+2)$ SoS relaxation of $\mathcal{A}$ is feasible with $R = R^{\circ}$ and $y = y^{\circ}$. The SoS proof follows by
	\begin{equation*}
		\mathcal{A} \sststile{\ell+2}{x,y,R} q(x, y) = \Tr\Paren{RR^\top \Mom_d(p_1,p_2,\ldots,p_m)} \geq 0\,.
	\end{equation*}
\end{proof}

\subsection{Sum-of-Squares subset sum}

Next we give an SoS proof for a formulation related to the minimum subset sum.

\begin{lemma}\label{lem:sos-certificate-subset-sum}
	For any vector $v\in \R^n$ such that for any subset $S\subseteq [n]$ we have $\sum_{i\in S} v_i\geq \gamma(|S|-\beta)$, there is a SoS proof that 
	\begin{equation*}
		\Set{z\odot z=z}\sststile{2}{z} \iprod{z,v}\geq \gamma(\norm{z}_1-\beta)\,.
	\end{equation*}
\end{lemma}
\begin{proof}
	We note that lemma is equivalent of showing
	\begin{equation*}
		\Set{z\odot z=z}\sststile{2}{z} \iprod{z,v-\gamma \one} \geq - \gamma \cdot \beta\,.
	\end{equation*}
	If $v_i \leq \gamma$ for all $i$, then
	\begin{equation*}
		\Set{z\odot z=z}\sststile{2}{z} \iprod{z,v-\gamma \one} \geq \iprod{\one,v-\gamma \one} \geq - \gamma \cdot \beta \,.
	\end{equation*}
	Else, w.l.o.g assume that $v$ is in non-decreasing order and $v_k$ is the last element that is at most $\gamma$, i.e. $v_1 \leq v_2 \leq \cdots \leq v_k \leq \gamma < v_{k+1} \leq v_n$. Now, we have
	\begin{align*}
		\Set{z\odot z=z} \sststile{2}{z}
		\iprod{z,v-\gamma \one}
		& = \sum_{i=1}^{k} (v_i-\gamma) z_i + \sum_{i=k+1}^{n} (v_i-\gamma) z_i \\
		& \geq \sum_{i=1}^{k} (v_i-\gamma) \\
		& \geq - \gamma \cdot \beta \,.
	\end{align*}
	where the second last inequality is because $\Set{z\odot z=z} \sststile{2}{z} \zeros \leq z \leq \one$.
\end{proof}

\section{Spectral bounds}

A well-known challenge in analyzing sparse SBMs is that high-degree nodes can inflate spectral bounds by an additional $\log n$ factor. 
To establish identifiability in  \cref{sec:robust-bisection-boosting}, we leverage a classical result from random matrix theory, which demonstrates that pruning nodes with degrees above a certain threshold suffices to achieve the desired spectral bounds.

\begin{theorem}[Originally proved in \cite{FeigeOfek2005}, restatement of theorem 6.7 in \cite{Liu2022minimax}]\label{thm:spectral-remove}
        Suppose $M$ is a random symmetric matrix with zero on the diagonal whose entries above the diagonal are independent with the following distribution
        \[
        M_{ij} = \begin{cases} 1 - p_{ij} &\text{ w.p. } p_{ij} \\ -p_{ij} &\text{ w.p. } 1 - p_{ij}  \end{cases}
        \]
        Let $\sigma$ be a quantity such that $p_{ij} \leq \sigma^2$ and $M_1$ be the matrix obtained from M by zeroing out
        all the rows and columns having more than $20 \sigma^2 n$ positive entries. Then with probability $1 - 1/n^2$, we have $\norm{M_1} \leq \chi \sigma \sqrt{n}$ for some universal constant $\chi$.
\end{theorem}

As a consequence of the above, we admit the following spectral bound for the adjacency matrix of a pruned SBM.

\begin{corollary}[Corollary 6.8 in \cite{Liu2022minimax}]\label{cor:spectral-remove-cor}
    Let \( \Gnull \) be a graph sampled from the \( k \)-stochastic block model \( \SBM_n(d,\e,k)\) and $\Znull$ be the true community membership matrix. Then, with probability at least \( 1 - \frac{2}{n^2} \), there exists a subset \( S \subseteq [n] \) of size at least \( \bigl(1 - \exp{\bigl({-2C_{d, \e}}\bigr)}\bigr) n \) such that
    \[
        \Normop{\Paren{\Gnull - \frac{d}{n}J - \frac{\e d}{n} \Xnull} \odot \one_S \one_S^\top} \leq \chi \sqrt{d},
    \]
    where \( \chi \) is a universal constant and $\Xnull= \Znull (\Znull)^\top-\frac{1}{k} J$ is the signal part of $\Gnull$.
\end{corollary}
\section{Statistical properties of stochastic block model}
\label{sec:stat-prop-appendix}

\subsection{Level-wise SBM statistics.}\label{sec:stat-prop-appendix-1}

Let us consider a level $i\in\{1,2,\ldots,\lfloor \log_2 k\rfloor\}$ in the recursive bisection as described in \cref{sec:sym-break-robust-clustering}. 
We define the parameters
\[
\beta_i \coloneqq 2^{-i},\qquad n_i \coloneqq 2 \beta_i n,\qquad k_i \coloneqq \beta_i k,
\]
so that one side of the level-$i$ bisection has $n_i/2$ vertices and contains $k_i$ communities. 
We let $\alpha_i\coloneqq 1/k$ denote the mass of a single ground-truth community and define the mixing parameter
\[
\gamma_i \coloneqq \alpha_i/\beta_i \in [0,1],
\]
as the fraction of the side occupied by the community of a fixed vertex.  
Using $\gamma_i$, we introduce the \emph{geometrically weighted} connection probabilities at level $i$, so that
\[\widetilde a_i \coloneqq a^{\gamma_i} b^{\,1-\gamma_i} \enspace, \qquad \widetilde b_i \coloneqq b \enspace. \]
This geometric weighting reflects that a random neighbor on the same side belongs to the relevant community with mass $\gamma_i$ 
and to an irrelevant community with mass $1-\gamma_i$. 
A simple uniform bound on the mixed probabilities will be used repeatedly through the section.

\begin{lemma}\label{lem:sqrt-all-levels}
Fix $k\ge 2$ and $0<\e<1$. Then, for every level $i$,
\begin{equation}\label{eq:sqrt-term-1}
\sqrt{\widetilde a_i+\widetilde b_i} \le \sqrt{a+b}
= \sqrt{\,d\bigl(2+\e(1-\frac{2}{k})\bigr)\,} \le \sqrt{3d}.
\end{equation}
\end{lemma}
\begin{proof}
Since $a\ge b$ and $\gamma_i\in[0,1]$, we have $b^{\,1-\gamma_i}\le a^{\,1-\gamma_i}$ and hence
$\widetilde a_i=a^{\gamma_i}b^{\,1-\gamma_i}\le a^{\gamma_i}a^{\,1-\gamma_i}=a$; also $\widetilde b_i=b$.
Therefore $\widetilde a_i+\widetilde b_i\le a+b$, which implies the first inequality in \eqref{eq:sqrt-term-1}.
A direct computation gives
\[
a+b =\Bigl(1+\bigl(1 - \frac1k\bigr)\e + 1 - \frac{\e}{k}\Bigr)d
=\bigl(2+\e(1-\frac{2}{k})\bigr)d
\le (2+\e)d \le 3d.
\]
\end{proof}

\noindent To encode the same-side vs. opposite-side separation in a form amenable to multiplicative bounds used throughout the section 
we consider the per-edge log-odds ratio. For $p,q\in(0,1)$ set
\[
R(p,q)\coloneqq \frac{p(1-q)}{q(1-p)}\,,
\qquad
\widetilde R_i \coloneqq R \Bigl(\widetilde{a_i}/n,\,\widetilde{b_i}/n\Bigr).
\]
The following lemma provides a useful lower bound on $\log \widetilde R_i$. 

\begin{lemma}\label{lem:logR}
Fix $k\ge 2$ and $\e\in(0,1)$, and let $d=o(n)$.
Then, for every level $i$,
\begin{equation}\label{eq:logR-lb}
\frac{\gamma_i\,\e}{2} \le \log \widetilde R_i \le \gamma_i\,\e\log 3 + o(1),
\end{equation}
where $o(1)=o_n(1)$ is uniform in $i$.
\end{lemma}

\begin{proof}
By definition,
\[
\log\widetilde R_i
= \log\frac{\widetilde a_i}{\widetilde b_i}
 + \log\frac{1-\widetilde b_i/n}{1-\widetilde a_i/n}.
\]

\noindent\emph{Lower bound.}
Since $\widetilde a_i/\widetilde b_i=(a/b)^{\gamma_i}$ and $\widetilde a_i\ge\widetilde b_i$,
\[
\log\widetilde R_i \ge \gamma_i\log\frac{a}{b}.
\]
With $a=(1+(1-\frac1k)\e)d$, $b=(1-\frac{\e}{k})d$,
\[
\log\frac{a}{b}
=\log\bigl(1+(1-\tfrac1k)\e\bigr) - \log\bigl(1 - \tfrac{\e}{k}\bigr)
\ge (1-\tfrac1k)\e - \tfrac12(1-\tfrac1k)^2\e^2 + \tfrac{\e}{k}
\ge \e-\tfrac{\e^2}{2}\ge\tfrac{\e}{2},
\]
hence $\log\widetilde R_i\ge \gamma_i\e/2$.

\medskip
\noindent\emph{Upper bound.}
Using $1-\widetilde b_i/n\le1$ and $-\log(1-x)\le x/(1-x)$ for $x\in(0,1)$,
\[
\log\frac{1-\widetilde b_i/n}{1-\widetilde a_i/n}
\le -\log\Bigl(1-\frac{\widetilde a_i}{n}\Bigr)
\le \frac{\widetilde a_i}{n-\widetilde a_i}
\le \frac{a}{n-a},
\]
so
\[
\log\widetilde R_i \le \gamma_i\log\frac{a}{b} + \frac{a}{n-a}.
\]
Moreover, for $k\ge2$,
\[
\log\frac{a}{b} = \log\frac{1+(1-\tfrac1k)\e}{1-\tfrac{\e}{k}}
\le \log\frac{1+\e/2}{1-\e/2}
\le \e\log 3 \enspace.
\]
We, thus, get
\[
\log\widetilde R_i \le \gamma_i\,\e\log 3 + \frac{a}{n-a}.
\]
Since $a=\Theta(d)$ and $d=o(n)$, we have $a/(n-a)=o(1)$ uniformly in $i$, yielding the desired upper bound.
\end{proof}

\subsection{Concentration inequalities}\label{sec:stat-prop-appendix-2}
We now prove \cref{thm:mixture-binomial-concentration-bound}, 
a bound on the success probability of majority voting for the held-out vertex. 
Let us start with a short convexity lemma.

\begin{lemma}\label{lemma:concav_ineq}
    Let $a, b, n > 0$ and $\gamma \in (0,1]$. 
    Assume $a \le n$ and $b \le n$.
    Define $f:[0,1] \to \mathbb{R}$ by
    \[
    f(x)
    = x - \bigl(x+\frac{n}{a}-1\bigr)^{\gamma}\bigl(x+\frac{n}{b}-1\bigr)^{1-\gamma}
    + \bigl(\frac{n}{a}\bigr)^{\gamma}\bigl(\frac{n}{b}\bigr)^{1-\gamma} - 1.
    \]
    Then $f(x)\ge 0$ for all $x\in[0,1]$.
\end{lemma}
    
\begin{proof}[Proof of Lemma~\ref{lemma:concav_ineq}]
    If $\gamma=1$ then $f(x)\equiv 0$, thus we restrict attention to $0<\gamma<1$.
    We certify nonnegativity by showing that $f$ is convex on $[0,1]$ and minimized at $x=1$.\newline
    
    \noindent Let $u(x)=x+\frac{n}{a}-1$ and $v(x)=x+\frac{n}{b}-1$. 
    A direct differentiation gives
    \begin{equation}\label{eq:f_prime}
    f'(x)
    =1-\gamma\,u(x)^{\gamma-1}\,v(x)^{\,1-\gamma}-(1-\gamma)\,u(x)^{\gamma}\,v(x)^{-\gamma}.
    \end{equation}
    Evaluating \eqref{eq:f_prime} at $x=1$ (so $u(1)=\frac{n}{a}$, $v(1)=\frac{n}{b}$) yields
    \[
    f'(1)=1-\gamma\,\left(\frac{a}{b}\right)^{1-\gamma}-(1-\gamma)\,\left(\frac{b}{a}\right)^{\gamma}\le0,
    \]
    where we used the weighted AM-GM inequality
    \[
    \gamma a+(1-\gamma)\,b \ge a^{\gamma}\,b^{\,1-\gamma}
    \quad\Longleftrightarrow\quad
    \gamma\,\left(\frac{a}{b}\right)^{1-\gamma}+(1-\gamma)\,\left(\frac{b}{a}\right)^{\gamma} \ge1.
    \]
    Differentiating \eqref{eq:f_prime} once more and factoring,
    \begin{align*}
    f''(x)
    &=\gamma(1-\gamma)\,u(x)^{\gamma-2}\,v(x)^{\,1-\gamma}
    -2\,\gamma(1-\gamma)\,u(x)^{\gamma-1}\,v(x)^{-\gamma}
    +\gamma(1-\gamma)\,u(x)^{\gamma}\,v(x)^{-\gamma-1}\\
    &=\gamma(1-\gamma)\,u(x)^{\gamma-2}\,v(x)^{-\gamma-1}\,\bigl(v(x)-u(x)\bigr)^2 \ge0.
    \end{align*}
    Hence $f$ is convex on $[0,1]$. 
    In particular, $f'$ is nondecreasing on the interval. 
    Since $f'(1)\le 0$, we have $f'(x)\le 0$ for all $x\in[0,1]$, i.e., $f$ is decreasing on $[0,1]$. 
    Finally,
    \[
    f(1)=1-\Bigl(\frac{n}{a}\Bigr)^{\gamma}\Bigl(\frac{n}{b}\Bigr)^{1-\gamma}
    +\Bigl(\frac{n}{a}\Bigr)^{\gamma}\Bigl(\frac{n}{b}\Bigr)^{1-\gamma}-1=0,
    \]
    so $f(x)\ge f(1)=0$ for every $x\in[0,1]$, as claimed.
\end{proof}

\begin{theorem}
\label{thm:three_binom_conc}
Assume $a \le n$ and $b \le n$. Fix parameters $\beta \in (0, 1]$ and $\alpha \in (0, \beta]$.
Consider the distribution
\[
    \mathcal{D} = \mathrm{Binom}(\alpha n, a/n)
    + \mathrm{Binom}\bigl((\beta-\alpha)n,\, b/n\bigr)
    - \mathrm{Binom}(\beta n,\, b/n) \enspace,    
\]
where the three binomials are independent.
Let $\gamma\coloneqq \alpha/\beta$ and define
\[
\widetilde a \coloneqq a^{\gamma} b^{\,1-\gamma} \enspace,
\qquad \widetilde b \coloneqq b \enspace,
\qquad \widetilde C \coloneqq \bigl(\sqrt{\widetilde a}-\sqrt{\widetilde b}\,\bigr)^{2} \enspace,
\qquad R(p,q) \coloneqq \frac{p(1-q)}{q(1-p)} \enspace.
\]
Then for every $\theta\in\mathbb R$,
\[
\Pr_{\bm X \sim \mathcal{D}}[\,\bm X \le \theta\,] \le
\exp \Bigl(-\,\beta\,\widetilde C + \frac{\theta}{2}\,\log R(\widetilde a/n,\, \widetilde b/n)\Bigr) \enspace.
\]
\end{theorem}

\begin{proof}[Proof of Theorem~\ref{thm:three_binom_conc}]
Let $t>0$ to be fixed later.
By Markov's inequality,
\[
\Pr[\bm X \le \theta]
=\Pr \bigl(e^{-t\bm X} \ge e^{-t \theta} \bigr) \le \E[e^{-t\bm X}]\,e^{t\theta} \enspace.
\]
Since the three binomials are independent and
$\E[e^{s\,\mathrm{Binom}(m,p)}]=(pe^{s}+1-p)^{m}$, we have
\begin{align*}
\E[e^{-t\bm X}]
&=\bigl(a/n \, e^{-t} + 1 - a/n \bigr)^{\alpha n}
  \bigl(b/n \, e^{-t} + 1 - b/n \bigr)^{(\beta-\alpha) n}
  \cdot \bigl(b/n\,e^{\,t}+1-b/n\bigr)^{\beta n}\\
&=\Bigl(\underbrace{\bigl(a/n\,e^{-t} + 1 - a/n\bigr)^{\gamma}
                    \bigl(b/n\,e^{-t} + 1 -b/n\bigr)^{1-\gamma}}_{\text{apply Lemma \ref{lemma:concav_ineq}}}
                    \cdot\bigl(b/n\,e^{\,t} + 1 - b/n\bigr)\Bigr)^{\beta n} \enspace,
\end{align*}
where $\gamma=\alpha/\beta\in[0,1]$.
Applying Lemma~\ref{lemma:concav_ineq} with $x=e^{-t}\in(0,1]$ gives
\[
\bigl(a/n\,e^{-t}+1-a/n\bigr)^{\gamma}
\bigl(b/n\,e^{-t}+1-b/n\bigr)^{1-\gamma} \le
(\widetilde a/n)\,e^{-t} + 1-\widetilde a/n,
\]
with $\widetilde a=a^{\gamma}b^{\,1-\gamma}$ and $\widetilde b = b$. Hence
\[
\E[e^{-t\bm X}] \le
\Bigl(\bigl(\frac{\widetilde a}{n}e^{-t} + 1 - \frac{\widetilde a}{n}\bigr)
      \bigl(\frac{\widetilde b}{n}e^{\,t} + 1 - \frac{\widetilde b}{n}\bigr)\Bigr)^{\beta n} \enspace.
\]
We choose
\[
e^{t} = \sqrt{\frac{(\widetilde a/n)(1-\widetilde b/n)}{(\widetilde b/n)(1-\widetilde a/n)}}
= \sqrt{R(\widetilde a/n, \, \widetilde b/n)} \enspace.
\]
It then follows that
\begin{align*}
\Pr[\bm X \le \theta]
&\le \Bigl(\bigl(\frac{\widetilde a}{n}e^{-t} + 1 - \frac{\widetilde a}{n}\bigr)
            \bigl(\frac{\widetilde b}{n}e^{t} + 1 - \frac{\widetilde b}{n}\bigr)\Bigr)^{\beta n}
      e^{t\theta} \\
&= \Bigl(\sqrt{\frac{\widetilde a\,\widetilde b}{n^{2}}}
          +\sqrt{\bigl(1-\frac{\widetilde a}{n}\bigr)\bigl(1 - \frac{\widetilde b}{n}\bigr)}\Bigr)^{2\beta n}
      e^{t\theta} \\
&\le \Bigl(1-\frac{\bigl(\sqrt{\widetilde a} - \sqrt{\widetilde b}\bigr)^{2}}{2n}\Bigr)^{2\beta n}
   \exp \Bigl(\frac{\theta}{2}\log R(\widetilde a/n,\widetilde b/n)\Bigr) \\
&\le \exp \Bigl(-\beta\,\widetilde C +\frac{\theta}{2}\log R(\widetilde a/n,\widetilde b/n)\Bigr) \enspace.
\end{align*}
\end{proof}

\begin{theorem}\label{thm:mixture-binomial-concentration-bound}
Fix a level $i\in\{1,2,\ldots,\lfloor \log_2 k\rfloor\}$ and let $\beta_i \coloneqq 2^{-i}$ (with $\alpha_i = 1/k$). 
Let $\widetilde a,\widetilde b,\widetilde C$ and $R(\cdot,\cdot)$ be as in Theorem~\ref{thm:three_binom_conc}, evaluated at this $(\alpha, \beta)$.
Then for every $\theta\in\mathbb{R}$,
\[
\Pr[\bm X \le \theta] \le
\exp \left(
-\frac{(\log 2)^2}{4}\cdot \frac{d\,\e^{2}}{\beta_i\,k^{2}}
+\frac{\theta}{2}\log R(\widetilde a/n,\widetilde b/n)
\right).
\]
\end{theorem}

\begin{proof}[Proof of Theorem~\ref{thm:mixture-binomial-concentration-bound}]
By Theorem~\ref{thm:three_binom_conc} (applied with $(\beta,\alpha)=(\beta_i,\alpha_i)$),
\[
\Pr[\bm X \le \theta]
  \le \exp\left(-\,\beta_i\,\widetilde C_i
  + \frac{\theta}{2}\log R(\widetilde a/n,\widetilde b/n)\right).
\]
It remains to lower bound $\beta_i \widetilde C_i$.
Starting from the definition,
\begin{align*}
\widetilde C_i
&= \bigl(\sqrt{\widetilde a} - \sqrt{\widetilde b}\,\bigr)^2 \\
&= b \left(\Bigl(\frac{a}{b}\Bigr)^{\gamma_i/2}-1\right)^{2},
\end{align*}
where $\gamma_i=\alpha_i/\beta_i$. Writing $r\coloneqq a/b>1$, we have
\[
r^{u}-1 = e^{u\log r}-1 \ge u\log r \quad\text{for all } u>0,
\]
so with $u=\gamma_i/2$,
\[
\Bigl(\frac{a}{b}\Bigr)^{\gamma_i/2}-1
\;\ge\; \frac{\gamma_i}{2}\,\log\frac{a}{b}.
\]
Hence
\begin{align*}
\widetilde C_i
&\ge b\left(\frac{\gamma_i}{2}\,\log\frac{a}{b}\right)^2 \\
&= b\,\frac{\gamma_i^{2}}{4}\,\log^{2}\frac{a}{b}.
\end{align*}
Multiplying by $\beta_i$ and using $\gamma_i=\frac{\alpha_i}{\beta_i}$ and $\alpha_i=\frac1k$, we get
\begin{align*}
\beta_i \,\widetilde C_i
&\ge \beta_i\,\frac{b\,\gamma_i^{2}}{4}\,\log^{2}\frac{a}{b} \\
&= \frac{b}{4}\cdot \frac{\alpha_i^{2}}{\beta_i}\,\log^{2}\frac{a}{b} \\
&= \frac{b}{4}\cdot \frac{1}{\beta_i\,k^{2}}\,\log^{2}\frac{a}{b}.
\end{align*}
Recall $a=(1+(1-\frac1k)\e)d$ and $b=(1-\frac{\e}{k})d$, so
\[
\frac{a}{b}
= \frac{1+(1-\frac1k)\e}{1-\frac{\e}{k}},
\qquad
\frac{b}{d}=1-\frac{\e}{k},
\]
and therefore
\begin{align*}
\beta_i \,\widetilde C_i
&\ge \frac{d}{4\,\beta_i\,k^{2}}\,
     \Bigl(1-\frac{\e}{k}\Bigr)\,
     \log^{2}\frac{1+(1-\frac1k)\e}{1-\frac{\e}{k}}.
\end{align*}
Define
\[
H(\e,k)
\coloneqq
\frac{1-\frac{\e}{k}}{\e^{2}}\,
\log^{2}\frac{1+(1-\frac1k)\e}{1-\frac{\e}{k}},
\qquad \e\in(0,1],\ k\ge2.
\]
A direct computation shows that $H(\e,k)$ is decreasing in $k$, so
\[
\inf_{k\ge2} H(\e,k)
= \lim_{k\to\infty} H(\e,k)
= \frac{\log^{2}(1+\e)}{\e^{2}}.
\]
Moreover, $\log(1+\e)/\e$ is decreasing on $(0,1]$, hence
\[
\frac{\log^{2}(1+\e)}{\e^{2}} \ge \log^{2} 2
\quad\text{for all }\e\in(0,1].
\]
Thus $H(\e,k)\ge \log^{2}2$ for all $\e\in(0,1]$, $k\ge2$, and we obtain
\[
\beta_i\,\widetilde C_i
\ge \frac{(\log 2)^2}{4}\,\frac{d\,\e^{2}}{\beta_i\,k^{2}}.
\]
Substituting this into the Theorem~\ref{thm:three_binom_conc} inequality yields
\[
\Pr[\bm X \le \theta] \le
\exp\left(
-\frac{(\log 2)^2}{4}\,\frac{d\,\e^{2}}{\beta_i\,k^{2}}
+\frac{\theta}{2}\log R(\widetilde a/n,\widetilde b/n)
\right),
\]
as claimed.
\end{proof}

\subsection{Voting lower bound}\label{sec:stat-prop-appendix-3}
In this section, we will leverage the results of \cref{sec:stat-prop-appendix-2} to prove the global voting concentration theorems \cref{thm:inner-product-lb} and \cref{thm:inner-product-lb-k-clustering}.

\begin{theorem}[Claim 6.2 in \cite{Liu2022minimax}]\label{subrect_sum}
Let $n,n_{1},n_{2}$ be parameters with $n_{1},n_{2}\le 10^{-6}n$.
Let $M\in\mathbb{R}^{n\times n}$ be a random matrix whose entries are independent,
have mean $0$, variance at most $\sigma^{2}$ for some $\sigma\ge 20/\sqrt{n}$,
and are almost surely bounded in $[-1,1]$.
Then, with probability at least
\[
1-\exp \Bigl(-8\,(n_{1}+n_{2})\,
\log \bigl(\frac{n}{n_{1}}+\frac{n}{n_{2}}\bigr)\Bigr),
\]
the magnitude of the sum of the entries over any $n_{1}\times n_{2}$
combinatorial subrectangle of $M$ is at most $(n_{1}+n_{2})\,\sigma\,\sqrt{n}$.
\end{theorem}

As a corollary, we get the following theorem which applies for rectangles which are long and thin.

\begin{theorem}\label{subrow_sum}
  Let $n,n_{1},n_{2}$ be parameters with $n_{1}\le 10^{-6}n$ and $n_{2}\ge 10^{-6} n$.
  Let $M \in \mathbb{R}^{n \times n}$ be a random matrix whose entries are independent,
  have mean $0$, variance at most $\sigma^{2}$ for some $\sigma \ge 20/\sqrt{n}$,
  and are almost surely bounded in $[-1,1]$. Then, with probability at least
  \[
  1-\exp \Bigl(-8\,(n_{1}+10^{-6}n)\,
  \log \bigl(\tfrac{n}{n_{1}}+10^6\bigr)\Bigr),
  \]
  the magnitude of the sum of the entries over any $n_{1}\times n_{2}$
  combinatorial subrectangle of $M$ is at most $2\,(n_{1}+n_{2})\,\sigma\,\sqrt{n}$.
\end{theorem}

\begin{proof}
  Let us fix a set $S\subseteq[n]$ of rows with $|S|=n_1$ and define
  \[
    v_j := \sum_{i\in S} M_{ij}, \qquad j\in[n].
  \]
  Set $\delta:=10^{-6}$. By \cref{subrect_sum} with parameters $(n_1,\delta n)$, with the stated probability
  we have for all $U\subseteq[n]$ with $|U|=\delta n$ that
  \[
    \Bigl| \sum_{j \in U} v_j \Bigr|
    = \Bigl| \sum_{(i,j)\in S \times U} M_{ij} \Bigr|
    \le (n_1 + \delta n) \, \sigma \sqrt{n}.
  \]
  We now use a simple averaging inequality: for any integers $K\ge L$ and any vector $w$,
  \[
    \sup_{|T|=K}\, \Bigl|\sum_{j\in T} w_j\Bigr|
    \le \frac{K}{L} \sup_{|U|=L}\Bigl|\sum_{j\in U} w_j\Bigr|.
  \]
  Indeed, for any fixed $T$ with $|T|=K$, each $j\in T$ lies in exactly $\binom{K-1}{L-1}$ subsets
  $U\subseteq T$ of size $L$, so
  \[
    \sum_{j\in T} w_j
    = \binom{K-1}{L-1}^{-1}
      \sum_{\substack{U\subseteq T\\|U|=L}}\sum_{j\in U} w_j,
  \]
  and taking absolute values yields the desired result since
  \(\binom{K}{L}/\binom{K-1}{L-1}=K/L\).

  Applying this with $w=v$, $K=n_2$, and $L=\delta n$ yields, uniformly for all $T$ with $|T|=n_2$,
  \begin{align*}
    \Bigl|\sum_{(i,j)\in S\times T} M_{ij}\Bigr|
      &= \Bigl| \sum_{j\in T} v_j \Bigr| \\
      &\le \frac{n_2}{\delta n}\,\sup_{|U|=\delta n}\Bigl|\sum_{j\in U} v_j \Bigr| \\
      &\le \frac{n_2}{\delta n}\,(n_1+\delta n)\,\sigma \sqrt{n} \\
      &\le 2 n_2 \,\sigma \sqrt{n} \\
      &\le 2\,(n_1+n_2)\,\sigma\,\sqrt{n}\,.
  \end{align*}
  Since \cref{subrect_sum} is uniform over all
  choices of $S$ and $U$ of the specified sizes, the bound holds simultaneously for all $n_1\times n_2$
  combinatorial subrectangles, as claimed.
\end{proof}

\begin{theorem}
\label{voting_concentration_lb}
Fix a level $i\in\{1,2,\ldots,\lfloor \log_2 k\rfloor\}$ and set
$\beta_i \coloneqq 2^{-i}$, $n_i\coloneqq 2\beta_i n$, and
$\gamma_i\coloneqq \alpha_i/\beta_i=2^i/k$ with $\alpha_i=1/k$.
Let $y\in\{0,\pm 1\}^n$ be a valid level-$i$ bisection labeling. 
Fix parameters $t\ge 1$ and $0<\rho<10^{-6}$, and set
\[
\widetilde Q_i = \beta_i \widetilde C_i - \log(1/\rho) - 3t.
\]
Assume $a\ge b$, $a,b\le n$, and $a\ge 400$.  Then, with probability at least
$1-e^{-t\rho n_i}- 1/n_i^3$, for any $z$ supported on $\{j:\,y_j\neq 0\}$ with
$\|z\|_1=\rho\,n_i$, we have
\[
\iprod{\bar{G}\,y \odot y,\, z} \ge
2 \rho n_i
\biggl( \frac{\widetilde Q_i}{\log \widetilde R_i} - \sqrt{a+b} \biggr).
\]
\end{theorem}

\begin{proof}
Imagine sampling the matrix $G$ by independently drawing the entries $G_{ij}$ for $i<j$ 
and then filling in the remainder symmetrically, with $G_{ji}=G_{ij}$ and $G_{ii}=0$. 
If $i$ and $j$ are in the same community then $G_{ij}\sim\mathrm{Bernoulli}(a/n)$; 
otherwise $G_{ij}\sim\mathrm{Bernoulli}(b/n)$.
Fix a level-$i$ bisection with sides $S_i^+$ and $S_i^-$ of size $n_i/2$ each, and write $\alpha=1/k$. 
For a vertex $u\in S_i^+$ consider the signed difference 
\[
\bm {X_u} = \sum_{v\in S_i^+\setminus\{u\}}G_{uv} - \sum_{v\in S_i^-}G_{uv} \enspace.
\]
We observe that the neighbors of $u$ split into three disjoint sets: 
its own community $C(u)\setminus\{u\}$ of size $\alpha n-1$ where edges have bias $a/n$; 
the remaining $(\beta_i-\alpha)n$ vertices on the same side where edges have bias $b/n$; 
and the $\beta_i n$ vertices on the opposite side where edges have bias $b/n$.
Independence across disjoint unordered pairs implies that these three contributions are independent binomials,
hence 
\[
\bm {X_u} \stackrel{d}{=} \mathrm{Binom}(\alpha n - 1, a/n)
 + \mathrm{Binom}\bigl((\beta_i-\alpha)n, b/n\bigr)
 - \mathrm{Binom}(\beta_i n, b/n) \enspace.
\]
We note that since $\iprod{G y\odot y, z}$ is linear in $z$ and only coordinates with $y_j\neq 0$ contribute, 
it suffices to take $z=\one_S$ for some $S\subseteq\{j:\,y_j\neq 0\}$ with $|S|=\rho\,n_i$ 
(if $\rho n_i\notin\mathbb N$, take $\lfloor \rho n_i\rfloor$ and $\lceil \rho n_i\rceil$; 
any other $z\in[0,1]^n$ with the same $\ell_1$-norm is a convex combination).

Let us first construct an auxiliary matrix $G^{\mathrm{ind}}$ in which, for each fixed row $u$, the variables $\bm\{G^{\mathrm{ind}}_{uv}\}_{v}$ are independent Bernoullis with the same means as above (and different rows are also independent); define $\bm {X^{\mathrm{ind}}_u}$ analogously. Then
\[
\iprod{G^{\mathrm{ind}} y\odot y,\,\one_S}=\sum_{u\in S} \bm {X^{\mathrm{ind}}_u}.
\]
We now apply \cref{thm:three_binom_conc} to the sum $\displaystyle \sum_{u\in S} \bm {X^{\mathrm{ind}}_u}$ by independence across rows with the parameter
\[
\widetilde{\theta}_{i} 
= 2 \rho n_i\frac{\beta_i \widetilde C_i - \log(1/\rho) - 3t}{\log \widetilde R_i}
= 2 \rho n_i \frac{\widetilde Q_i}{\log \widetilde R_i}.
\]
This yields
\begin{align*}
\Pr\Big[\sum_{u\in S} \bm {X^{\mathrm{ind}}_u} \le \widetilde{\theta}_i\Big]
&\le \exp\Big(-|S|\,\beta_i\widetilde C_i + \tfrac{\widetilde{\theta}_i}{2}\log \widetilde R_i\Big)
= \exp\Big(-|S|(\log(1/\rho)+3t)\Big) \enspace. 
\end{align*}
A union bound over all $S\subseteq\{y\neq 0\}$ with $|S|=\rho n_i$ gives
\[
\Pr\Big[\exists\,S:\, \sum_{u\in S} \bm {X^{\mathrm{ind}}_u} \le \widetilde{\theta}_i\Big]
\le \binom{n_i}{\rho n_i}\,e^{-|S|(\log(1/\rho)+3t)}
\le e^{-t\rho n_i},
\]
using $\binom{n_i}{\rho n_i}\le \exp(\rho n_i(\log(1/\rho)+1))$ and $t\ge 1$.
Therefore, with probability at least $1-e^{-t\rho n_i}$, simultaneously for all such $S$,
\begin{equation}\label{eq:ind-lb-updated}
\sum_{u\in S} \bm {X^{\mathrm{ind}}_u} \ge \widetilde{\theta}_i = 2|S|\,\frac{\widetilde Q_i}{\log \widetilde R_i}.
\end{equation}

It remains to pass from the independent sampling used above to the actual symmetric SBM with zero diagonal. 
We, first, note that zeroing the diagonal affects the sum by at most $|S|=\rho n_i$.
We, further, note that only the block $S\times S$ differs between the two sampling procedures.
Let $\Delta(S)$ be the difference between the contributions of $S\times S$ to
$\sum_{u\in S} \bm {X^{\mathrm{ind}}_u}$ and to $\sum_{u\in S} \bm {X_u}$.
As in \cite{Liu2022minimax}, $\Delta(S)$ is a sum/difference of independent, mean-$0$, bounded variables,
each supported in $[-1,1]$ and nonzero with probability at most $a/n$
(so $\sigma^2 \le (a+b)/n$).  Apply \cref{subrect_sum} with $n_1=n_2=|S|$ and
$\sigma^2=(a+b)/n$ (hence $\sigma\ge 20/\sqrt n$ because $a\ge 400$). Since $i\ge 1$, we have $n_i\le n$ and hence $|S|=\rho n_i\le \rho n\le 10^{-6}n$, so the conditions hold. We obtain
\[ \Pr\Big[|\Delta(S)|>2|S|\sqrt{a+b}\Big] \le \exp\Big(-16|S|\log\big(\tfrac{n}{|S|}\big)\Big)\enspace. \]
A union bound over all $\binom{n_i}{\rho n_i}$ choices of $S$ shows that, with additional failure
probability at most $1/n_i^3$ (for all large $n$), we have $|\Delta(S)|\le 2|S|\sqrt{a+b}$ simultaneously for all such $S$.

Finally, intersecting the two high‑probability events and noting that the level‑$i$ bisection is balanced so we get
\[
\iprod{\bar{G} y\odot y,\,\one_S}
=\sum_{u\in S} y_u \sum_{v} y_v\,(G_{uv}-d_i)
=\sum_{u\in S} \bm {X_u}
\ge \widetilde{\theta}_i - 2|S|\sqrt{a+b}.
\]
Substituting $|S|=\rho n_i$ and the definition of $\widetilde{\theta}_i$ yields the desired conclusion.
\end{proof}

\begin{theorem}\label{thm:level-i-voting}
Fix a level $i\in\{1,2,\ldots,\lfloor \log_2 k\rfloor\}$ 
and set $\beta_i\coloneqq 2^{-i}$, $n_i\coloneqq 2\beta_i n$, 
and $\gamma_i\coloneqq \alpha_i/\beta_i=2^i/k$ with $\alpha_i = 1/k$. 
Assume $0<\e<1$ and $k\ge 2$, and take other parameters as in \cref{thm:three_binom_conc}. 
Define 
\[
\rho_\gamma \coloneqq \min \Bigl(\,\frac{1}{e}, \exp\bigl(-\frac{\e^2 d}{8k}\bigr)\Bigr)\enspace.
\]
If, in addition, $\e^2 d \ge 9k^2$, 
then for any $0< \rho \le \min\{\rho_\gamma,10^{-6}\}$ and $t=\e^2 d/ k$, with probability at least $1-e^{-t\rho n_i}-\frac{1}{n_i^3}$ we have
\[
\iprod{Gy \odot y, z} \ge - 12 \rho_\gamma n_i \e d
\qquad\text{for every }z\text{ with }\|z\|_1=\rho n_i,
\]
where $y\in\{0,\pm1\}^n$ is any valid bisection labeling at level $i$.
\end{theorem}

\begin{proof}
Let $t=\e^2 d/k$ and fix any $0 < \rho \le \min\{\rho_\gamma,10^{-6}\}$. 
By \cref{voting_concentration_lb}, with probability at least $1-e^{-t\rho n_i}-\frac{1}{n_i^3}$, 
we have for all $z$ with $\|z\|_1=\rho n_i$,
\begin{equation}\label{eq:voting_conc_lb}
\iprod{Gy \odot y, z} \ge  2\rho n_i \left(\frac{\beta_i \widetilde C_i-\log(1/\rho) - 3t}{\log \widetilde R_i}
- \sqrt{a + b}\right) \enspace. \end{equation} 
We now bound the three potentially negative contributions on the right-hand side of \eqref{eq:voting_conc_lb}. 
For the $-\log(1/\rho)$ term, \cref{lem:logR} yields
\[
- \frac{2 \rho n_i}{\log\widetilde R_i} \log\frac 1 \rho \ge 
- \frac{4 \rho n_i}{\gamma_i\e} \log\frac1\rho \enspace.
\]
By the monotonicity of $x\log(1/x)$ on $(0,e^{-1}]$ and $\rho \le \rho_\gamma \le e^{-1}$, we have
\[
\rho \log\frac1\rho \le \rho_\gamma \log\frac1{\rho_\gamma},
\]
and therefore
\[
- \frac{2\rho n_i}{\log\widetilde R_i} \log \frac 1 \rho \ge
- \frac{4\rho_\gamma n_i}{\gamma_i\e}\log\frac1{\rho_\gamma}
\ge - \frac{4\rho_\gamma n_i}{\gamma_i\e}\cdot \frac{\e^2 d}{k} =
- 4 \rho_\gamma n_i \e d\cdot \frac{1}{\gamma_i k} \ge
- 4 \rho_\gamma n_i \e d \enspace.
\]
For the $-3t$ term, using $t=\e^2 d/k$ and \cref{lem:logR} we obtain
\[
- \frac{2\rho n_i\cdot 3t}{\log\widetilde R_i} \ge
- \frac{12 \rho n_i\, t}{\gamma_i\,\e}
= - \frac{12 \rho n_i}{\gamma_i\,\e}\cdot \frac{\e^2 d}{k}
= - 12 \rho n_i \e d \cdot \beta_i
\ge - 6 \,\rho_\gamma n_i\,\e d \enspace.
\]
Finally, by \cref{lem:sqrt-all-levels}, we have that
\[
- 2\rho n_i \sqrt{a + b} \ge 
- 2\rho_\gamma n_i \sqrt{3d} \ge 
- 2 \rho_\gamma n_i \e d \enspace.
\]
Combining these three bounds in \eqref{eq:voting_conc_lb} yields
\[
\iprod{Gy\odot y, z} \ge 
- \Bigl(4 + 6 + 2\Bigr)\,\rho_\gamma n_i \e d 
= - 12 \,\rho_\gamma n_i \e d
\ge - 12 \rho_\gamma n_i \e d,
\]
which is the desired inequality. 
The probability statement is exactly that of \cref{voting_concentration_lb} with $t=\e^2 d/k$, completing the proof.
\end{proof}

\begin{lemma} \label{lem:uniform-beta}
Fix a level $i$ and set $\beta_i=2^{-i}$ and $n_i= 2 \beta_i n$.  
Let $y\in\{0,\pm1\}^n$ be a valid bisection labeling at level $i$.  
Choose $0 < \rho_\gamma \le 10^{-6}$ and set $t=\e^2 d/k$. 
Then, with probability at least
\(1 - 2e^{-t} - n_i^{-2}\),
the inequality
\[ \iprod{Gy\odot y, z} \ge -12 \rho_\gamma n_i \e d \]
holds simultaneously for every vector $z\in\{0,\pm1\}^n$ with $\|z\|_1\ \le \rho_\gamma n_i$.
\end{lemma}

\begin{proof}
We discretize over the admissible $\ell_1$ sizes: 
since $z \in \{0, \pm1\}^n$ , 
the quantity $\|z\|_1$ is an integer, 
hence $\rho = \|z\|_1/n_i$ ranges over the set $\{j/n_i: 1\le j\le \lfloor \rho_\gamma n_i\rfloor\}$.  
For a fixed value $\rho$ we invoke the level-$i$ voting lower bound (Theorem~\ref{thm:level-i-voting}) with this $\rho$.  
It states that, with probability at least $1-e^{-t \rho n_i}-n_i^{-3}$, the estimate
\[ \iprod{Gy\odot y, z} \ge -12 \rho_\gamma n_i \e d \]
holds simultaneously for all $z$ with $\|z\|_1=\rho n_i$.  
Taking a union bound over the at most $\lfloor \rho_\gamma n_i \rfloor$ distinct values, 
the failure probability is at most
\[
\sum_{j=1}^{\lfloor \rho_\gamma n_i\rfloor} \Big(e^{-t j} + n_i^{-3}\Big) \le
\frac{e^{-t}}{1-e^{-t}} + \frac{\lfloor \rho_\gamma n_i\rfloor}{n_i^{3}} \le
\frac{e^{-t}}{1-e^{-t}} + \frac{1}{n_i^{2}} \le
2e^{-t}+\frac{1}{n_i^{2}}.
\]
On the complementary event, the desired bound holds for every value in the set, and
hence for every $z\in\{0,\pm1\}^n$ with $\|z\|_1\le \rho_\gamma n_i$, as claimed.
\end{proof}

Next we give a lemma for lower bounding the inner product when the size of the set is large. 

\begin{lemma}\label{lem:inner-product-lb-fixed-large-size}
Fix a level $i$ and set $\beta_i=2^{-i}$, $n_i= 2 \beta_i n$, and $k_i=\beta_i k$.
Let $y\in\{0,\pm1\}^n$ be a valid bisection labeling at level $i$.
Let $\gamma\in\bigl[0,\,1-\frac{1000\,\chi\,k_i}{\e\sqrt{d_i}}\bigr]$.
Define
\[
  \rho_\gamma \coloneqq \exp \Bigl(-\frac{\gamma\,\beta_i\,\widetilde C_i}{2}\Bigr),
  \qquad
  t \coloneqq 0.001\,(1-\gamma)\,\beta_i\,\widetilde C_i.
\]
Then for every $\rho\ge \rho_\gamma$, with probability at least
$1-\exp(-t\rho n_i)-\frac{1}{n_i^3}$, the inequality
\[
  \iprod{\bar G y\odot y,\,z} \ge \frac{\e d}{5k}\,(1-\gamma)\,\rho_\gamma\,n_i
\]
holds for all $z\in\{0,1\}^n$ with $\|z\|_1=\rho n_i$, where $\bar G$ is the centered adjacency matrix.
\end{lemma}

\begin{proof}
Let us fix $\rho\in[\rho_\gamma,1]$ and $z\in\{0,1\}^n$ with $\|z\|_1=\rho n_i$.
By Theorem~\ref{voting_concentration_lb} applied to $\bar G$, with probability at least
$1-e^{-t\rho n_i}-n_i^{-3}$ we have
\begin{equation}\label{eq:large-set-start}
  \iprod{\bar G y\odot y,\,z}
  \ge 2 \rho n_i \left(
      \frac{\widetilde Q_i}{\log \widetilde R_i} - \sqrt{a + b}
    \right),
  \qquad
  \widetilde Q_i \coloneqq \beta_i\widetilde C_i - \log\frac1\rho - 3t.
\end{equation}
Since $\rho\ge\rho_\gamma=\exp(-\gamma\,\beta_i\widetilde C_i/2)$, we have
\[
  -\log(1/\rho)\ge -\gamma\,\beta_i\widetilde C_i/2,
\]
and hence
\begin{align}
  \widetilde Q_i
  &\ge
    \Bigl(1-\frac{\gamma}{2}-0.003(1-\gamma)\Bigr)\,\beta_i\widetilde C_i \notag\\
  &\ge 0.997 \bigl(1-\gamma \bigr)\,\beta_i\widetilde C_i. \label{eq:Qi-lb}
\end{align}
By Lemma~\ref{lem:logR}, $\log\widetilde R_i \le \gamma_i\,\e\log 3 + o(1)$. 
Therefore
\begin{align}
  \frac{\widetilde Q_i}{\log\widetilde R_i}
  &\ge
    0.997 \frac{(1-\gamma)\,\beta_i\,\widetilde C_i}{\gamma_i\,\e\log 3 + o(1)} \notag\\
  &\ge
    0.997 \cdot \frac{(\log 2)^2}{4 \log 3} \, (1-\gamma)\,\frac{\e d}{k},
    \label{eq:signal-term}
\end{align}
using 
\[
  \beta_i\,\widetilde C_i \ge \frac{(\log 2)^2}{4}\,\frac{d\,\e^{2}}{\beta_i\,k^{2}}
\]
by the monotonicity shown in \cref{thm:mixture-binomial-concentration-bound}.

\noindent For the last term, Lemma~\ref{lem:sqrt-all-levels} gives $\sqrt{ a + b\, }\le \sqrt{3d}$.
We thus have
\begin{equation}\label{eq:noise-term}
  \sqrt{a + b}
  \le \sqrt{3d}
  \le 0.001 \,(1-\gamma)\,\frac{\e d}{k}.
\end{equation}
Combining \eqref{eq:large-set-start}, \eqref{eq:signal-term}, and \eqref{eq:noise-term}, we obtain
\begin{align*}
  \iprod{\bar G y\odot y,\,z}
  &\ge
    2 \rho n_i \left(
      0.109\,(1-\gamma)\,\frac{\e d}{k}
      - 0.001\,(1-\gamma)\,\frac{\e d}{k}
    \right) \\
  &\ge
    \frac{\e d}{5k}\,(1-\gamma)\,\rho n_i \\
  &\ge
    \frac{\e d}{5k}\,(1-\gamma)\,\rho_\gamma n_i \enspace,
\end{align*}
as claimed.
\end{proof}

\begin{corollary}\label{cor:inner-product-lb-large-size}
Fix a level $i$ and set $\beta_i=2^{-i}$ and $n_i=2\beta_i n$.
Let $y\in\{0,\pm1\}^n$ be a valid bisection labeling at level $i$.
Let $\gamma\in\bigl[0,\,1-\frac{1000\,\chi\,k_i}{\e\sqrt{d_i}}\bigr]$, and define
\[
\rho_\gamma \coloneqq \exp \Bigl(-\tfrac{\gamma\,\beta_i\,\widetilde C_i}{2}\Bigr),
\qquad
t \coloneqq 0.001\,(1-\gamma)\,\beta_i\,\widetilde C_i .
\]
Then, with probability at least \(1 - 2 e^{-t\,\rho_\gamma n_i} - n_i^{-2}\),
the inequality
\[
\iprod{\bar G y\odot y, z} \ge \frac{\e d}{5k}\,(1-\gamma)\,\rho_\gamma\,n_i
\]
holds simultaneously for every vector $z\in\{0,1\}^n$ with $\|z\|_1\ge \rho_\gamma n_i$, where $\bar G$ is the centered adjacency matrix.
\end{corollary}

\begin{proof}
We discretize over the admissible $\ell_1$ sizes. 
Since $z\in\{0,1\}^n$, the quantity $\|z\|_1$ is an integer, so
$\rho=\|z\|_1/n_i$ ranges over $\{j/n_i:\, m_0\le j\le n_i\}$, 
where $m_0\coloneqq \lceil \rho_\gamma n_i\rceil$.
Fix $j\in\{m_0,\ldots,n\}$ and set $\rho=j/n_i \ge \rho_\gamma$.
By Lemma~\ref{lem:inner-product-lb-fixed-large-size} (applied with our choice of $\rho$), 
with probability at least
$1-e^{-t \rho n_i}-n_i^{-3}=1 - e^{-t j} - n_i^{-3}$, we have
\[
\iprod{\bar G y\odot y, z} \ge \frac{\e d}{5k}\,(1-\gamma)\,\rho_\gamma\,n_i
\quad
\text{for all } z\in\{0,1\}^n \text{ with } \|z\|_1=j.
\]
Taking a union bound over all $j\in\{m_0,\ldots,n_i\}$, the total failure probability is at most
\[
\sum_{j=m_0}^{n_i} \bigl(e^{-t j} + n_i^{-3} \bigr) \le
\frac{e^{-t m_0}}{1-e^{-t}} + \frac{n_i-m_0+1}{n_i^{3}} \le
2 e^{-t \rho_\gamma n_i} + \frac{1}{n_i^{2}}.
\]
On the complementary event, the claimed lower bound holds simultaneously for all sizes
$j\ge m_0$, and hence for every $z\in\{0,1\}^n$ with $\|z\|_1\ge \rho_\gamma n_i$, as desired.
\end{proof}

Now we conclude the proof of the majority voting lower bound of bisections. 
\restatetheorem{thm:inner-product-lb}

\begin{proof}
  We first prove the claim for fixed bisection $y$. 
  When the size of the set is at least $\rho_\gamma n$, the lower bound easily follows from \cref{cor:inner-product-lb-large-size}. 
  When the size of the set is at most $\rho_\gamma n$, the lower bound follows from \cref{lem:uniform-beta} with probability $1 - \exp(-\e^2 d/k)$, as $\norm{z}_1\geq 0$. 
  From the proof of \cref{thm:mixture-binomial-concentration-bound}, 
  we know that $\beta_i \tilde{C_i}\geq \Omega(\epsilon^2 d_i/k_i^2)\geq \Omega(2^i \log(k))$. 
  It follows that $\tilde{C_i}\rho_\gamma n\geq 100k\log(k)$.
  Since there are at most $2^{2k}$ possible choices for community bisections $y$, 
  we can take union bound and the claim follows. 
\end{proof}

Finally we give the bound for larger values of $\gamma$ which will be helpful for obtaining the optimal recovery rate in the final $k$-clustering. 

\restatetheorem{thm:inner-product-lb-k-clustering}

\begin{proof}
  The proof is very similar to \cref{thm:inner-product-lb}. 
  However, now for the probability bound, we only need to use $\rho_\gamma n\geq 1$ and $C\geq 100k$, and we only need to take union bound over $k^2$ pairs of communities.  
\end{proof}

\restatecor{cor:inner-product-lb}
\begin{proof}
  Notice that
  \begin{align*}
    \iprod{\bar{G} \odot (s s^{\top}) y\odot y,z}
    & = \iprod{\bar{G} \odot (y y^{\top}),(z \odot s)s^{\top}} \\
    & = \iprod{\bar{G} \odot (y y^{\top}), (z \odot s) \one^{\top}}
    - \iprod{\bar{G} \odot (y y^{\top}), (z \odot s) (\one - s)^{\top}} \,.
  \end{align*}
  By \cref{thm:inner-product-lb}, it follows that, with probability $1-\exp(-100k)-\frac{1}{n^3}$,
  \begin{align*}
    \iprod{\bar{G} \odot (y y^{\top}), (z \odot s) \one^{\top}}
    & \geq \frac{(1-\gamma)\e d}{8k} \Paren{\norm{z \odot s}_1-\frac{96 \rho_\gamma k n}{1-\gamma}} \\
    & \geq \frac{(1-\gamma)\e d}{8k} \Paren{\norm{z}_1 - \norm{\one - s}_1-\frac{96 \rho_\gamma k n}{1-\gamma}}  \\
    & \geq \frac{(1-\gamma)\e d}{8k} \Paren{\norm{z}_1 - \exp(-2C_{d,\e})n -\frac{96 \rho_\gamma k n}{1-\gamma}} \\
    & \geq \frac{(1-\gamma)\e d}{8k} \Paren{\norm{z}_1 -\frac{200 \rho_\gamma k n}{1-\gamma}} \,.
  \end{align*}
  By \cref{subrow_sum}, it follows that, with probability $1-\frac{1}{n^3}$
  \begin{equation*}
    \iprod{\bar{G} \odot (y y^{\top}), (z \odot s) (\one - s)^{\top}}
    \leq \sqrt{2d} (\norm{z}_1 + \exp(-2C_{d,\e}) n)
    \leq \frac{(1-\gamma)\e d}{16k} \Paren{\norm{z}_1 +\frac{100 \rho_\gamma k n}{1-\gamma}} \,.
  \end{equation*}
  Therefore, with probability $1-\exp(-100k)-\frac{2}{n^3}$,
  \begin{align*}
    \iprod{\bar{G} \odot (s s^{\top}) y\odot y,z}
    & \geq \frac{(1-\gamma)\e d}{8k} \Paren{\norm{z}_1 -\frac{200 \rho_\gamma k n}{1-\gamma}} - \frac{(1-\gamma)\e d}{16k} \Paren{\norm{z}_1 +\frac{100 \rho_\gamma k n}{1-\gamma}} \\
    & \geq \frac{(1-\gamma)\e d}{16k} \Paren{\norm{z}_1 -\frac{640 \rho_\gamma k n}{1-\gamma}} \,.
  \end{align*}
\end{proof}

\restatecor{cor:inner-product-lb-k-clustering}
\begin{proof}
  The proof is very similar to \cref{cor:inner-product-lb}. The only difference is that the voting lower bound is via \cref{thm:inner-product-lb-k-clustering}.
\end{proof}

\end{document}